\documentclass[12pt]{article}

\usepackage[margin=1in]{geometry}   % Reasonable margins for review
\usepackage{setspace}
\usepackage{lmodern}                % Better default font
\usepackage[T1]{fontenc}
\usepackage[utf8]{inputenc}

\usepackage{amsmath, amssymb, amsthm, mathtools}
\usepackage{bm}                     % Bold math symbols
\usepackage{mathrsfs}               % Script math (for e.g. sigma-fields)
\usepackage{bbm}                    % For indicator function \mathbbm{1}

\numberwithin{equation}{section}

\usepackage{float}
\usepackage{graphicx}
\usepackage{subcaption}
\usepackage{algorithm}
\usepackage{makeidx}
\usepackage{algpseudocode}
\usepackage{booktabs}
\usepackage{array}
\usepackage{enumitem}               % Customizable lists
\usepackage{soul}

\usepackage[authoryear]{natbib}     % Author–year citations
\usepackage{hyperref}
\hypersetup{
  colorlinks = true,
  linkcolor  = black,
  citecolor  = blue, % Temporary...
  urlcolor   = black
}
\usepackage[capitalize,noabbrev,nameinlink]{cleveref}
\theoremstyle{plain}
\newtheorem{theorem}{Theorem}[section]
\newtheorem{lemma}[theorem]{Lemma}

\newtheorem{corollary}[theorem]{Corollary}
\newtheorem{assumption}[theorem]{Assumption}

\theoremstyle{definition}

\theoremstyle{remark}
\newtheorem{remark}[theorem]{Remark}

\makeindex

\newcommand{\E}{\mathbb{E}}

\DeclareMathOperator{\tr}{tr}

\DeclareMathOperator*{\argmin}{arg\,min}

\newcounter{mntcomm}

\newcounter{ddcomm}

\newcounter{tmcomm}

\title{\vspace{-1cm}Inversion-Free Natural Gradient Descent on Riemannian Manifolds}
\author{
  Dario Draca\textsuperscript{1},
  Takuo Matsubara\textsuperscript{2} and
  Minh-Ngoc Tran\textsuperscript{2}
  \\[1em]
  \textit{\textsuperscript{1}School of Mathematics and Statistics, University of Sydney, Australia} \\
  \textit{\textsuperscript{2}University of Sydney Business School, Australia}
}
\date{}  % Omit date on title page

\begin{document}
% --------------------------------------------------
\maketitle

% --------------------------------------------------
% Abstract and Keywords
% --------------------------------------------------
\begin{abstract}
\footnotesize
The natural gradient method is a central tool for statistical optimisation, but its broader application is hindered by the assumption of a Euclidean parameter space, the repeated estimation of the Fisher information matrix (FIM), and the computational cost of its subsequent inversion.
% This paper proposes a purely intrinsic, inversion-free natural gradient method for objective functionals defined over density functions parameterised by parameters living on general Riemannian manifolds.
This paper proposes an intrinsic, inversion-free natural gradient method for statistical models whose parameters lie on general Riemannian manifolds.
Formulating statistical optimisation in this non-Euclidean setting allows for the natural enforcement of parameter constraints, the elimination of non-identifiable parameters, and the exploitation of geodesic convexity. 
% % To overcome the computational bottleneck of FIM inversion, our algorithm maintains a moving approximation of the inverse FIM directly on the manifold.
% % By exploiting the outer-product structure of the Fisher information, this approximation is updated using score vectors mapped across distinct tangent spaces via vector transport operations.
% To overcome the computational bottleneck of matrix inversion, our algorithm maintains a moving approximation of the inverse FIM directly on the manifold,
% where new score vectors are incorporated via efficient low rank updates.
Our algorithm is based on a moving approximation of the inverse FIM, which is maintained directly on the manifold.
% This approximation is periodically updated with new score vectors using low-rank matrix identities, at lower cost than dense matrix inversion.
This approximation is efficiently updated with new score vectors using low-rank matrix identities.
% We establish a rigorous analysis, proving almost-sure convergence rates of $O(\log s / s^\alpha)$ for the sequence of iterates, and a similar rate for the approximate FIM.
We prove almost-sure convergence rates of $O(\log s / s^\alpha)$ for the sequence of iterates, and a similar rate for the approximate FIM.
A limited-memory variant with sub-quadratic storage complexity is further proposed for large-scale applications. 
We demonstrate the efficacy of our method on variational Bayes within the Bures-Wasserstein manifold, normalising flows on the Stiefel manifold, and reduced-rank logistic regression.
\end{abstract}
\normalsize


% --------------------------------------------------
% Main text
% --------------------------------------------------

%\newpage
\section{Introduction}

This paper studies the optimisation of an objective function defined over a family of parametric distributions \( \mathcal{Q} = \{q_\theta : \theta \in \mathcal{M} \} \). Problems of this form are central to statistics and machine learning. For instance, maximum likelihood estimation \citep{fisher1922mathematical} minimises the negative log-likelihood over $\mathcal{Q}$, while variational Bayes approximates an intractable posterior distribution by minimising the negative evidence lower bound \citep{blei2017variational}.
% Two broad paradigms exist for addressing this class of optimisation problems. 
% % The first identifies $\mathcal{Q}$ directly with its parameter space $\mathcal{M}$, reducing the optimisation to a standard problem in $\mathcal{M}$.
% % This paradigm, however, treats $\mathcal{Q}$ without regard to the geometry of the underlying statistical model.
% % This paradigm, however, discards the intrinsic differential structure of $\mathcal{Q}$ entirely.
% The first reduces the optimisation to a standard problem on $\mathcal{M}$, treating the objective as a function of the parameters, but without accounting for how parameter perturbations affect the distribution.
% % In contrast, the second geometrically informed paradigm, which this paper adopts, formulates $\mathcal{Q}$ as a statistical manifold in the sense of information geometry \citep{amari2016information}, leveraging tools from differential geometry to accelerate optimisation.
% This paper adopts the second approach, which equips $\mathcal{Q}$ with the structure of a statistical manifold \citep{amari2016information}, and exploits this structure to accelerate optimisation.
The optimisation literature provides a host of general-purpose methods for such problems \citep{duchi2011adaptive, kingma2014adam, byrd2016stochastic}.
However, these approaches are agnostic to how the shape of the distribution changes with respect to its parameterisation.

In information geometry \citep{amari2016information}, the parametric family $\mathcal{Q}$ is formally structured as a Riemannian manifold endowed with the Fisher information metric.
% The gradient of the objective function defined over this manifold is known as the \emph{Fisher-scored gradient} in the statistics literature, and popularised in the machine learning literature by \cite{amari1998natural} under the name \emph{natural gradient}.
The gradient of the objective function with respect to this metric defines the update direction of the classical Fisher scoring algorithm \citep{osborne1992fisher}, and is known as the \emph{natural gradient} in the machine learning literature \citep{amari1998natural}.
% This formulation is fundamental to statistical methodologies, yielding a parameterisation-invariant steepest descent direction on the parametric family $\mathcal{Q}$.
This formulation yields a steepest descent direction on $\mathcal{Q}$ that is intrinsic to the statistical model, and in particular invariant to reparameterisation.
The natural gradient has driven substantial improvements across a broad spectrum of optimisation tasks, from variational inference \citep{hoffman2013stochastic, tran2017variational, pmlr-v97-lin19b} to deep learning \citep{martens2020new}. 
However, 
% despite these well-documented advantages, the broader application of the natural gradient remains fundamentally bottlenecked: 
the vast majority of existing methodology relies on the assumption that the parameter space $\mathcal{M}$ is a flat, Euclidean space.

This paper is concerned with a \emph{doubly} geometric setting, where the underlying parameter space $\mathcal{M}$ of the statistical manifold $\mathcal{Q}$ is also a Riemannian manifold. The extension of natural gradient methods to non-Euclidean parameter spaces is motivated by several practical considerations. First, numerous statistical models involve constrained parameters. For instance, the space of symmetric positive definite (SPD) matrices naturally forms a complete Riemannian manifold \citep{pennec2006riemannian, arsigny2006log}, where updating the matrix along its geodesics enforces positive definiteness without requiring ad-hoc projections \citep{tran2021variational,lin2020handling}.
% Second, formulating optimisation problems on non-Euclidean spaces can eliminate non-identifiable parameters, ensuring the invertibility of the Fisher information matrix (FIM) of the parametric family $\mathcal{Q}$.
Second, formulating optimisation problems on non-Euclidean spaces can eliminate redundancies in the parametrisation, yielding a smaller and more well-conditioned Fisher information matrix (FIM).
% For example, problems involving linear subspaces---such as sufficient dimension reduction \citep{cook2009likelihood, nghiem2024likelihood} and envelope models \citep{cook2015simultaneous}---are naturally formulated on the Grassmann manifold \citep{edelman1998geometry}.
For example, in sufficient dimension reduction \citep{cook2009likelihood, nghiem2024likelihood} and envelope models \citep{cook2015simultaneous}, a key parameter is a linear subspace, typically represented by a semi-orthogonal basis matrix. Since any rotation of this basis spans the same subspace, the likelihood is invariant to such rotations. The redundancy can be eliminated by formulating the problem directly on the Grassmann manifold \citep{edelman1998geometry}.
Similarly, structural non-identifiability in factored low-rank models \citep{yee2003reduced} can be resolved via the manifold of fixed-rank matrices \citep{meyer2011linear, mishra2014fixed}. Finally, adopting a manifold perspective can improve the regularity of the objective function. Certain optimisation problems that are non-convex under standard Euclidean parameterisations exhibit convexity along the geodesics of a suitably chosen manifold \citep{wiesel2012geodesic}. A prominent example is the Bures-Wasserstein manifold of Gaussian distributions \citep{bures1969extension, malago2018wasserstein}, where the Kullback-Leibler divergence to a log-concave target becomes geodesically convex, yielding optimisation algorithms with strong theoretical guarantees \citep{lambert2022variational, diao2023forward}.

Despite the growing interest of Riemannian optimisation \citep{hosseini2020recent, boumal2023introduction}, comparatively little progress has been made on developing natural gradient methods for general Riemannian parameter spaces. 
While the notion of Fisher information matrix extends naturally to Riemannian manifolds \citep{smith2005covariance, xavier2005intrinsic}, the methodological challenge lies in operationalising it into a practical optimisation algorithm.
Recently, \citet{hu2024riemannian} proposed a natural gradient method specialised for matrix manifolds embedded in $\mathbb{R}^{m \times n}$ \citep{absil2009optimization}. 
Their approach relies on the ambient vector space to represent the FIM, estimating it on a per-iteration basis using a Monte Carlo (minibatch) sample of score vectors.
% However, their methodology and theory are inextricably linked to the existence of a convenient ambient embedding space. 
% The development of a purely intrinsic natural gradient method, which operates universally across general Riemannian manifolds while simultaneously overcoming the computational bottleneck, remains a significant challenge.
This strategy has some limitations: first, the ambient space can be much larger than the manifold's intrinsic dimension, incurring additional storage overhead. Second, Monte Carlo estimates of the Fisher information often have large variance, and the estimated FIM must then be inverted at each iteration, incurring up to cubic cost.

% The computational challenges of natural gradient methods arise from the estimation of the FIM and its subsequent inversion at each iterate. 
% This is a bottleneck that persists regardless of the parameter space's geometric structure. 
The computational difficulties of storing, estimating, and inverting the FIM remain the primary impediment to the practical implementation of natural gradient methods. In the Euclidean setting, considerable effort has been devoted to addressing these bottlenecks; see e.g. \citet{martens2020new}.
% In the Euclidean setting, \citet{godichon2024natural} proposed an elegant approach that circumvents the exact estimation of the FIM and its direct inversion. 
Recently, \citet{godichon2024natural} proposed\footnote{A similar approach was studied in the machine learning community \citep{amari2000adaptive, park2000adaptive}.} an elegant approach that streamlines the estimation of the FIM and its inversion. 
% By exploiting the outer-product structure of the FIM, they maintain a stochastic moving approximation of the inverse FIM at each iterate, significantly reducing the computational cost while guaranteeing almost-sure convergence of the optimisation algorithm.
Their method maintains a moving approximation of the inverse FIM, which is repeatedly updated with new score vectors. This is achieved at quadratic cost per iteration using low-rank matrix identities \citep{sherman1950adjustment}. 
% However, generalising this inversion-free paradigm to general Riemannian manifolds introduces major theoretical difficulties. Because the approximate inverse FIM is localized to the tangent space of the current iterate, maintaining a moving average necessitates transporting this object along the manifold. 
The development of such ``inversion-free'' algorithms has also been popularised in the stochastic and quasi-Newton literature \citep{chau2024inversion}.
Adapting such an approach to Riemannian manifolds is appealing, as it is computationally efficient, highly general, and amenable to theoretical analysis. The new challenge is that score vectors are localised to parameter-specific vector spaces---their \textit{tangent spaces}. To maintain and update a moving approximation to the FIM, one must repeatedly move this object across tangent spaces using vector transport operations \citep{absil2009optimization}. This mirrors the transport of Hessian approximations in Riemannian quasi-Newton methods \citep{huang2015broyden}.
% The core technical difficulty lies in controlling the accumulated curvature-based distortions and the structural errors introduced by the approximation of parallel transport operations.

% The main contribution of this paper is the development of a fully intrinsic, inverse-free natural gradient method for optimisation over general Riemannian parameter spaces.
The main contribution of this paper is the development of an inversion-free natural gradient method for optimisation over Riemannian parameter spaces.
% Our framework overcomes both the geometric limitations of ambient embeddings and the computational burden of exact FIM inversion.
% In contrast to \citet{hu2024riemannian}, our approach is purely intrinsic.
In contrast to \citet{hu2024riemannian}, our approach is purely intrinsic: it does not require an ambient embedding space for its implementation or analysis, and assumes only standard regularity conditions on the manifold, distribution family, and objective function.
To our knowledge, this is the first practical natural gradient method applicable to general Riemannian manifolds.
% It explicitly treats the inverse FIM as a stochastic object localized to the tangent space of the current iterate, requiring no ambient space and relying strictly on standard regularity conditions of the manifold and the objective function.
%Our algorithm extends \citet{godichon2024natural} to general manifolds in full rigor, achieving a quadratic time natural gradient method in most manifolds encountered in practice.
% Our algorithm extends \citet{godichon2024natural} to general manifolds in full rigor.
% It maintains a moving approximation of the inverse FIM updated via vector transport operations.
Following \citet{godichon2024natural}, our algorithm maintains a moving approximation of the inverse FIM, which undergoes alternating transport and low-rank score vector updates.
Since this approximation averages over previous score vectors, it benefits from a larger effective sample size than per-iteration estimates, yielding a more stable natural gradient direction.
% We establish asymptotic convergence guarantees of the proposed algorithm. Specifically, the algorithm converges to a global or local minimiser almost surely at a rate of $O(\log{s}/s^\alpha)$, with $s$ the number of iterations.
% Similarly, the approximate FIM of the algorithm converges to the exact FIM almost surely.
We establish asymptotic results for the proposed algorithm: the iterates converge to a global or local minimiser almost surely at a rate of $O(\log{s}/s^\alpha)$ for $\alpha \in (\tfrac{2}{3},1)$, which matches the Euclidean setting. The approximate FIM is also shown to converge almost surely, albeit at a slower rate than the Euclidean.
% To further alleviate computational constraints, we propose a limited-memory variant of the algorithm that employs a low-rank estimate of the inverse FIM, yielding sub-quadratic storage complexity.
For large-scale applications, we propose a limited-memory variant based on a low-rank representation of the inverse FIM, with sub-quadratic storage complexity.
% Finally, we demonstrate the practical efficacy of our intrinsic method through challenging statistical applications, including variational Bayes with Gaussian distributions on the Bures-Wasserstein space, normalising flows \citep{berg2018sylvester} on the Stiefel manifold, and reduced-rank regression on the manifold of fixed-rank matrices \citep{yee2003reduced, meyer2011linear}.
Finally, we demonstrate our method on several statistical applications across a range of manifolds. This includes variational Bayes on the Bures-Wasserstein space, normalising flows \citep{berg2018sylvester} on the Stiefel manifold, and reduced-rank logistic regression \citep{yee2003reduced} on the manifold of fixed-rank matrices \citep{meyer2011linear}.

The paper is organised as follows.
\Cref{sec:Background} introduces the necessary preliminaries and background.
In \Cref{sec:Riemannian Natural Gradient}, we review the intrinsic formulation of the Fisher information on a smooth manifold, and the corresponding representation of the natural gradient.
\Cref{sec:Inversion-free Riemannian Natural Gradient} presents our inversion-free natural gradient method, followed by a rigorous convergence analysis in \Cref{sec:Convergence Analysis}.
We demonstrate the application of our method to the Bures-Wasserstein, Stiefel, and fixed-rank manifolds in \Cref{sec:Example: Bures-Wasserstein Natural Gradient,sec:reduced-rank-regression,sec:stiefel-experiment}.
% \Cref{sec:Example: Bures-Wasserstein Natural Gradient}, \Cref{sec:reduced-rank-regression} and \Cref{sec:stiefel-experiment}, respectively. 
Finally, \Cref{sec:Conclusion} concludes the paper. 
Proofs and technical details are provided in the Supplementary Material.

\section{Background}\label{sec:Background}

In this section we recall, informally, some elementary notions from Riemannian geometry, describe the statistical optimisation problems of interest, and review the natural gradient method. For details on Riemannian optimisation see \citet{absil2009optimization, boumal2023introduction}.

\subsection{Geometric preliminaries}

Let $\mathcal{M}$ denote a smooth $d$-dimensional manifold; this is a topological space which locally resembles $\mathbb{R}^d$, and is regular enough to support a rigorous notion of smoothness and differentiability for curves and functions. At each point $x \in \mathcal{M}$, the velocities $\dot{\gamma}(0)$ of smooth curves $\gamma:[0,1]\rightarrow\mathcal{M}$ with $\gamma(0)=x$ form a vector space, the \textit{tangent space} $T_x\mathcal{M}$. The \textit{tangent bundle} collects all tangent spaces $T\mathcal{M} = \{(x,v) : x \in \mathcal{M}, v \in T_x\mathcal{M}\}$. If $\Phi:\mathcal{M}\rightarrow\mathcal{N}$ is a smooth map between manifolds, then at each $x \in \mathcal{M}$ it induces a linear map $D\Phi(x): T_x\mathcal{M} \rightarrow T_{\Phi(x)}\mathcal{N}$ referred to as its \textit{differential}; this is analogous to the (Euclidean) Jacobian viewed as a linear map. A \textit{Riemannian metric} on $\mathcal{M}$ assigns a smoothly varying inner product to each tangent space, denoted $\langle \cdot, \cdot\rangle_x : T_x\mathcal{M} \times T_x\mathcal{M} \rightarrow \mathbb{R}$. The pair $(\mathcal{M}, \langle \cdot, \cdot\rangle_x)$ is referred to as a \textit{Riemannian manifold}. The induced norm is $\Vert v\Vert_x = \sqrt{\langle v, v\rangle_x}$. The \textit{Riemannian length} of a curve $\gamma : [0,1] \rightarrow \mathcal{M}$ is obtained by integrating the magnitude of the velocity: $\mathrm{len}(\gamma)= \int^1_0 \Vert \dot{\gamma}(t)\Vert_{\gamma(t)}dt$. The \textit{Riemannian distance} $d(x,y)$ is defined as the infimum of $\mathrm{len}(\gamma)$ over all smooth curves connecting $x,y\in \mathcal{M}$. A smooth curve $\gamma:[0,1] \rightarrow \mathcal{M}$ is a \textit{geodesic} if it experiences zero acceleration on the manifold; this is the analogue of a straight line in Euclidean space. For each $v \in T_x\mathcal{M}$, there exists a unique geodesic $\gamma_{x,v}(t)$ with $\gamma_{x,v}(0)=x$ and $\dot{\gamma}_{x,v}(0)=v$. The \textit{exponential map} at $x$ is $\exp_x(v) :=\gamma_{x,v}(1)$, generalizing the Euclidean operation $x \rightarrow x+v$. In practice, the exponential map must often be replaced by a \textit{retraction}; a smooth map $\mathcal{R} : T\mathcal{M} \rightarrow \mathcal{M}$ such that $\mathcal{R}_x(0)=x$ and $D\mathcal{R}_x(0)=\mathrm{Id}_{x}$ for all $x \in \mathcal{M}$. A retraction is \textit{second order} if $d(\mathcal{R}_x(v), \exp_x(v)) = O(\Vert v\Vert^3_x)$ for $(x,v) \in T\mathcal{M}$.

\subsection{Statistical optimisation}
In this paper we are interested in optimisation over a parametric family of probability distributions; this problem is central to statistics. Let $\mathcal{Q} = \{q_\theta : \theta \in \mathcal{M}\}$ denote such a family, and consider a dataset $\{y_i\}^{n}_{i=1}$. Maximum likelihood estimation \citep{fisher1922mathematical} minimises
\begin{equation}\label{eq:MLE loss}
    \mathcal{L}(\theta):=\E_{Y\sim \tilde{p}}\big[-\log q_\theta(Y)\big], \qquad \tilde{p}(y)=\frac{1}{n}\sum_{i=1}^n\delta_{y_i}(y).
\end{equation}
In variational Bayes \citep{wainwright2008graphical}, one seeks to approximate a target density $\pi(y)=\bar\pi(y)/C$, where $\bar\pi$ can be evaluated pointwise but the normalising constant $C$ may be unknown. This is done by minimising the negative evidence lower bound
\begin{equation}\label{eq:VB loss}
    \mathcal{L}(\theta):=\E_{Y\sim q_\theta}\big[\log{q_\theta(Y)} - \log{\bar\pi(Y)}\big].
\end{equation}
For both objectives, the gradient can be formulated as an expectation $\nabla \mathcal{L}(\theta)=\mathbb{E}_{Y \sim p}[g(Y;\theta)]$, where $p$ and $g$ are suitably chosen; see \Cref{sec:natural-gradient-step}. When $\mathcal{M}=\mathbb{R}^d$, the standard approach to optimizing \eqref{eq:VB loss} is stochastic gradient descent \citep{hoffman2013stochastic}

\begin{equation} \label{eq:stochastic-gd}
    \theta^{(k+1)} = \theta^{(k)} - \tau_k \widehat{\nabla \mathcal{L}}(\theta^{(k)}),
\end{equation}
where $\widehat{\nabla \mathcal{L}}$ is a stochastic estimate of the gradient, and $\tau_k>0$ is a step size.

This procedure often converges slowly in practice, and one must employ additional strategies like momentum \citep{polyak1964some, nesterov2013gradient}. Popular alternatives include adaptive gradient methods \citep{duchi2011adaptive, kingma2014adam}, and stochastic (quasi)-Newton methods \citep{byrd2016stochastic}; see also \citet{bottou2018optimization} for a comprehensive survey. The natural gradient method \citep{amari1998natural} is also notable, and we discuss it shortly.

The generalisation of Euclidean stochastic optimisation methods to the Riemannian setting has been the subject of considerable attention. \citet{bonnabel2013stochastic} studied the convergence of Riemannian stochastic gradient descent, which generalises the iteration \eqref{eq:stochastic-gd} to manifolds:
\begin{equation}
    \theta^{(k+1)} = \mathcal{R}_{\theta^{(k)}}(-\tau_k \widehat{\nabla \mathcal{L}}(\theta^{(k)})).
\end{equation}

Here, $\widehat{\nabla\mathcal{L}}(\theta_k)$ is a stochastic estimate of the \textit{Riemannian} gradient, and $\mathcal{R}$ is a retraction. Various extensions including adaptive Riemannian gradient methods \citep{becigneul2018riemannian, kasai2019riemannian}, and variance reduced methods \citep{zhang2016riemannian, sato2019riemannian} have been proposed. There are also Riemannian generalisations of stochastic (quasi)-Newton methods \citep{kasai2018riemannian}. For a somewhat recent overview, see \citet{hosseini2020recent}.

These methods exploit the geometry of the parameter space $\mathcal{M}$, and in the case of second-order methods the curvature of the objective. However, they do not directly account for how parameter perturbations affect the shape of the distribution, which is of primary concern in \eqref{eq:MLE loss} and \eqref{eq:VB loss}. If there is a mismatch between parameter scales and their influence on $q_\theta$, this can degrade optimisation. The natural gradient method addresses this directly.

\subsection{Natural gradient descent}

In natural gradient descent \citep{amari1998natural} on $\mathbb{R}^d$, one preconditions the Euclidean gradient with the inverse Fisher information $I_F(\theta):=\mathbb{E}_{Y\sim q_\theta}[\nabla\ell_Y(\theta) \nabla \ell_Y(\theta)^\top]$ where $\ell_y(\theta):=\log q_\theta(y)$

\begin{equation}
    \nabla^{\text{nat}}\mathcal{L}(\theta) = I_F^{-1}(\theta) \nabla \mathcal{L}(\theta).
\end{equation}
It is well-known that the Fisher information locally characterises the KL divergence.
% To understand why this benefits optimisation, we first note that the Fisher information locally characterises the KL divergence.
For a small perturbation $\delta \theta$, we have the following second-order expansion; see e.g. \citep{amari2016information}
\begin{equation} \label{eq:KL-second-order}
    \mathrm{KL}(q_\theta \Vert q_{\theta+\delta \theta}) = \frac{1}{2} (\delta \theta)^\intercal I_F(\theta) (\delta\theta) + o(\Vert \delta \theta\Vert^2).
\end{equation}

The Fisher information thus weights directions according to their effect on the distribution, as measured by the KL divergence.
Consequently, preconditioning by $I_F^{-1}(\theta)$ shrinks steps in directions where small parameter changes induce large changes in the distribution, and enlarge them in directions where the distribution is less sensitive. 

Formally, the natural gradient corresponds to the Riemannian gradient of $\mathcal{L}$ with respect to the Fisher-Rao metric on $\mathcal{Q}$. This geometric structure is intrinsic to the statistical model, so the natural gradient points in the same direction in distribution space regardless of the parametrisation. This invariance, together with the above rescaling behavior, often leads to more effective optimisation \citep{ollivier2017information}. For a review of natural gradient methods in modern machine learning, we refer to \citet{martens2020new}.

In addition to \citet{hu2024riemannian}, several authors have investigated the natural gradient method in the manifold setting. Closely related is \cite{tran2021variational}, who develop a general-purpose natural gradient algorithm which is applicable to embedded submanifolds and quotient manifolds, where $\mathcal{M} \subset \mathbb{R}^{d}$ and $\dim{\mathcal{M}} \leq d$. They approximate natural gradient directions in the ambient space, and then project the result onto the tangent space of $\mathcal{M}$. 
A notable limitation is that the Fisher information can be singular in the ambient space.

The remaining works are more specialised, and mostly apply to SPD manifolds. For example, \citet{lin2020handling} formulate a retraction-based natural gradient update for SPD matrices which enforces positive-definiteness. \citet{lin2023simplifying} propose a momentum-based approach on certain SPD (sub)manifolds which employs special local coordinate systems to simplify updates. \citet{magris2024manifold} develop a natural gradient algorithm for Gaussian VB, which performs retraction-based updates with respect to the precision matrix.

\section{Riemannian Natural Gradient Descent}\label{sec:Riemannian Natural Gradient}

In this section we define the Fisher information for a family of distributions whose parameters live on a Riemannian manifold, and formulate the natural gradient update.

\subsection{Fisher information on a Manifold}

The intrinsic formulation of the Fisher information on a Riemannian manifold is well established; see, for example, \citet{smith2005covariance,xavier2005intrinsic}. Let $\mathcal{Q}:=\{q_\theta : \theta \in \mathcal{M}\}$ be a family of densities on $\mathbb{R}^d$, where $\mathcal{M}$ is a smooth manifold. Define
\begin{equation}
\ell_y(\theta) := \log q_\theta(y), \quad \quad y \in \mathbb{R}^d.    
\end{equation}
The Fisher information at $\theta$ is a bilinear form (equivalently, a $(0,2)$-tensor) on $T_\theta \mathcal{M}$
\begin{equation}
  F_\theta[u,v] :=
  \mathbb{E}_{Y \sim q_\theta} \big[D\ell_Y(\theta)[u]\ D\ell_Y(\theta)[v]\big], \qquad u,v \in T_\theta\mathcal{M},
  \label{eq:Fisher-tensor-def}
\end{equation}
where $D\ell_Y(\theta) : T_\theta \mathcal{M} \rightarrow \mathbb{R}$ is the differential with respect to $\theta \in \mathcal{M}$. Provided the expectation exists, $F_\theta$ is clearly symmetric and non-negative definite. 

Suppose in addition that $\mathcal{M}$ possesses a ``baseline'' Riemannian metric $\langle \cdot,\cdot \rangle_\theta$. Let $G_\theta$ denote its matrix representation in local coordinates, so $\langle u, v\rangle_\theta = u^\top G_\theta v$ for all $u,v \in T_\theta \mathcal{M}$. We use $\nabla \ell_y(\theta)$ to denote the Riemannian gradient of $\ell_y(\cdot)$ with respect to this baseline metric; i.e. $D\ell_y(\theta)[u] = \langle u, \nabla \ell_y(\theta)\rangle_\theta$ for all $u \in T_\theta \mathcal{M}$.
We refer to $\nabla \ell_y(\theta)$ as the score vector throughout the paper.
Then, substituting into \eqref{eq:Fisher-tensor-def} yields
\begin{equation}
    F_\theta[u,v] = \mathbb{E}_{q_\theta}\big[\langle u, \nabla\ell_Y(\theta)\rangle_\theta \cdot \langle v, \nabla \ell_Y(\theta)\rangle_\theta \big] = u^\top G_\theta \mathbb{E}_{q_\theta}\left[\nabla\ell_Y(\theta) \, \nabla \ell_Y(\theta)^\top \right]G_\theta v.
\end{equation}
Thus, in local coordinates the bilinear map $F_\theta$ is represented by the
matrix
\begin{equation} \label{eq:Fisher-tensor-def-coords}
    \hat{F}_\theta := G_\theta \mathbb{E}_{q_\theta}\left[\nabla\ell_Y(\theta) \, \nabla\ell_Y(\theta)^\top  \right] G_\theta.
\end{equation}

When this expression is positive definite and varies smoothly in $\theta$, the Fisher information form defines a Riemannian metric on $\mathcal{M}$ (equivalently, $\mathcal{Q}$) in its own right. To streamline terminology, we will refer to \eqref{eq:Fisher-tensor-def} and its coordinate representations as the \textit{Fisher information metric} regardless of whether it satisfies these additional conditions.

The Fisher information metric can also be described by a self-adjoint linear map $I_F(\theta)$ %$: T_\theta \mathcal{M} \rightarrow T_\theta \mathcal{M}$ via the Riesz representation theorem
\begin{equation}
    F_\theta[u,v] = \langle u, I_F(\theta) v\rangle_\theta, \qquad \forall u,v\in T_\theta \mathcal{M}.
    \label{eq:Fisher-operator-def}
\end{equation}
In coordinates this operator has matrix representation
\begin{equation}
  \widehat I_F(\theta) = \mathbb{E}_{q_\theta} \left[\nabla \ell_Y(\theta)\nabla \ell_Y(\theta)^\top\right]\, G_\theta.
\end{equation}

Formally, $I_F(\theta)$ can be viewed as a $(1,1)$-tensor (or a linear map) obtained by raising one of the indices of $F_\theta$ with respect to the baseline metric. To disambiguate between $F_\theta$ and $I_F(\theta)$, we refer to the latter as the \textit{Fisher information operator} when necessary.

The Fisher metric retains many of its familiar properties from the Euclidean setting. For example, one has the equivalence between the outer product and Hessian representations \citep[Theorem 1]{smith2005covariance}. Note, here the second-order version of the Fisher information is defined using the \textit{Riemannian} Hessian \citep{absil2009optimization} interpreted as a bilinear form. Following \citet{smith2005covariance}, this expression is independent of the choice of affine connection.
\begin{equation}
    %\mathbb{E}_{q_\theta}[\nabla \ell_Y(\theta) \nabla \ell_Y (\theta)^{\top}] = -\mathbb{E}_{q_\theta}[\nabla^2 \ell_Y(\theta)].
    F_\theta = -\mathbb{E}_{q_\theta}[\nabla^2 \ell_Y(\theta)].
\end{equation}

% Since $F_\theta$ is a covariant tensor, we also have the same kind of ``change of variables'' rule as in Euclidean space. 
Similar to \eqref{eq:KL-second-order}, the Fisher information can also be shown to locally characterize the KL divergence on a Riemannian manifold. Let $\mathcal{R}$ denote a second-order retraction; in Appendix \ref{appendix:kl-fisher-manifold} we show the following local expansion under standard regularity conditions
% In the previous section, we mentioned that the Fisher information locally characterises the KL divergence. A similar property can be shown to hold on a Riemannian manifold. Let $\mathcal{R}$ denote a second-order retraction; under mild regularity conditions
\begin{equation}
  \mathrm{KL}(q_\theta \Vert q_{\mathcal{R}_\theta(v)})=\tfrac{1}{2}F_\theta [v,v]+ o(\|v\|_\theta^2), \qquad \forall (\theta, v) \in T\mathcal{M}.
\end{equation}

\subsection{Natural gradient on a Riemannian manifold}

To formulate the natural gradient descent step, let $f: \mathcal{M}\rightarrow \mathbb{R}$ denote a smooth function. The Riemannian gradient $\nabla f(\theta) \in T_\theta \mathcal{M}$ with respect to the baseline metric is defined by:
\begin{equation}
    Df(\theta)[v] = \langle v, \nabla f(\theta)\rangle_\theta, \qquad \forall v \in T_\theta \mathcal{M}.
\end{equation}
On $\mathbb{R}^d$ with the Euclidean metric, this reduces to the usual gradient. Provided that $F_\theta$ is positive, this differential can also be characterized by the Fisher metric

\begin{equation}
    Df(\theta)[v] = F_\theta(v, \nabla^{\text{nat}}f(\theta)), \qquad \forall v \in T_\theta \mathcal{M},
\end{equation}
Then $\nabla^{\text{nat}}f(\theta) \in T_\theta \mathcal{M}$ is the \textit{natural gradient}.
From \eqref{eq:Fisher-operator-def} we have the representation
\begin{equation}
    \nabla^{\text{nat}}f(\theta) = I_F^{-1}(\theta) \nabla f(\theta) = \left[\mathbb{E}_{q_\theta}[\nabla\ell_Y(\theta) \ \nabla \ell_Y(\theta)^\top] G_\theta\right]^{-1} \nabla f(\theta).
\end{equation}

Given a retraction $\mathcal{R}$, step size sequence $\tau_s > 0$, starting point $\theta^{(0)}$, a natural gradient descent iteration on $\mathcal{M}$ takes the form
\begin{equation}
  \theta^{(s+1)} = \mathcal{R}_{\theta^{(s)}}\Big(-\,\tau_s\,I_F^{-1}(\theta^{(s)})\,\nabla f(\theta^{(s)}) \Big).
  \label{eq:nat-grad-iteration}
\end{equation}
This update generalises the standard Euclidean scheme, where $\mathcal{M}$ is a vector space and the retraction is implicitly specified by the parametrisation.

\section{Inversion-free Riemannian Natural Gradient}\label{sec:Inversion-free Riemannian Natural Gradient}

In this section we present an approximate natural gradient descent method for arbitrary Riemannian manifolds. In many practical applications, the Fisher operator does not possess a convenient analytic expression, and must be repeatedly estimated and inverted. The proposed approach works by sampling a single (or small) number of score vectors at each iteration, and using a well-known matrix inversion identity \citep{sherman1950adjustment} to update a running approximation of the inverse Fisher operator. For a $d$-dimensional manifold, this operation has quadratic complexity, while inversion is typically cubic.

\begin{lemma}\label{lemma:sherman-morrison}
    Let $A \in \mathrm{GL}(n,\mathbb{R})$ and $u,v \in \mathbb{R}^n$ where $1+v^\intercal A^{-1} u \neq 0$, then
    \begin{equation}
        (A + uv^\intercal)^{-1} = A^{-1} - (1+v^\intercal A^{-1}u)^{-1} \cdot A^{-1}u v^\intercal A^{-1}.
        \label{eq:sherman-morrison}
    \end{equation}
\end{lemma}

This idea of approximating the natural gradient without matrix inversion was, to our knowledge, first mentioned in the Euclidean setting by \citet{amari2000adaptive}. More recently, its application to variational Bayes has been explored, still on $\mathbb{R}^d$ by \citet{godichon2024natural}; we now briefly review their method and adopt their notation. 

For each $s=1,2\ldots,$ let $\theta^{(s)}\in \mathbb{R}^d$ denote the $s$-th algorithm iterate, then define
\begin{equation}\label{eq:Hiform}
  H_{s+1}:= H_0+\sum_{k=0}^s\nabla \ell_{\bar{y}_{k+1}}(\theta^{(k)})\nabla \ell_{\bar{y}_{k+1}}(\theta^{(k)})^\top,\quad \bar{y}_{k+1}\sim q_{\theta^{(k)}}; \qquad \mathbf{H}_{s+1} := \frac{1}{s+1}H_{s+1}.
\end{equation}
where $H_0$ is some positive definite matrix such as $ H_0=\epsilon I$ with $\epsilon > 0$ and $I$ the identity. The matrix $\mathbf{H}^{-1}_{s+1}$ is used as a substitute for $I_F^{-1}(\theta^{(s)})$ in a natural gradient descent update
\begin{equation}
    \theta^{(s+1)} := \theta^{(s)} - \tau_{s+1} \mathbf{H}^{-1}_{s+1} \widehat{\nabla \mathcal{L}}(\theta^{(s)}),
\end{equation}
where $\tau_{s+1}$ is a step-size sequence, and $\mathcal{L}$ the (negative) evidence lower bound (ELBO). \cite{godichon2024natural} show that $\theta^{(s)}$ converges to a stationary point $\theta^*$ of $\mathcal{L}$, and that $\mathbf{H}^{-1}_s \rightarrow I_{F}^{-1}(\theta^*)$. The updated inverse $H_{s+1}^{-1}$ can be computed from $H_s^{-1}$ using \eqref{eq:sherman-morrison} as
\begin{equation}\label{eq:IF recursive}
H_{s+1}^{-1}= H^{-1}_s-\left(1+\phi^\top_{s+1} H_{s}^{-1}\phi_{s+1}\right)^{-1}
 H_s^{-1}\phi_{s+1}\phi_{s+1}^\top H_{s}^{-1}, \qquad \phi_{s+1}:=\nabla \ell_{\bar{y}_{s+1}}(\theta^{(s)}).
\end{equation}

We now generalize this procedure to a Riemannian manifold $\mathcal{M}$; let $\mathcal{L}:\mathcal{M}\rightarrow\mathbb{R}$ denote our objective function, $\mathcal{Q}:=\{q_\theta : \theta \in \mathcal{M}\}$ our family of densities, and $\theta^{(s)} \in \mathcal{M}$ the $s$-th algorithm iterate.
% In the following, we let $\mathcal{L}: \mathcal{M} \rightarrow \mathbb{R}$ denote a general objective function, $\{ q_\theta : \theta \in \mathcal{M}\}$ our family of densities with $\mathcal{M}$ a Riemannian manifold, and $\theta^{(s)}\in \mathcal{M}$ the $s$-th (current) iterate of the algorithm.
The algorithm consists of three repeated steps: transport the approximation to the inverse Fisher from $T_{\theta^{(s-1)}}\mathcal{M}$ to $T_{\theta^{(s)}}\mathcal{M}$, update it using score vector(s) from $q_{\theta^{(s)}}$, and then update the manifold iterate $\theta^{(s)}\rightarrow\theta^{(s+1)}$. We consider each step in more detail.

\subsection{Transporting Inverse Fisher Approximation}\label{sec:Transporting Inverse Fisher Approximation}

The main challenge in the Riemannian setting is that we cannot directly combine vectors from different tangent spaces; hence, we need a mechanism to move $\mathbf{H}^{-1}_s$ from $T_{\theta^{(s-1)}}\mathcal{M}$ to $T_{\theta^{(s)}}\mathcal{M}$. This is typically achieved using a vector transport operation, which generalises the notion of parallel transport along geodesics; see e.g. \citet[Section 8.1]{absil2009optimization}. 

Informally, a vector transport $\mathcal{T}$ associated with a retraction $\mathcal{R}$ is a smooth map which, given $(x,u) \in T\mathcal{M}$, defines a linear map $\mathcal{T}_{u_x}:T_x\mathcal{M} \rightarrow T_{\mathcal{R}_x(u)}\mathcal{M}$ where $\mathcal{T}_{0_x} := \mathrm{Id}_x$.

To aid clarity, for $x,y \in  \mathcal{M}$ where $\mathcal{R}^{-1}_x(y)$ is well defined, we denote
\begin{equation}
    \mathcal{T}_{x, y}:=\mathcal{T}_{\mathcal{R}^{-1}_x(y)} : T_x\mathcal{M} \rightarrow T_y\mathcal{M}, \qquad \mathcal{T}_{s,s\pm1}:=\mathcal{T}_{\theta^{(s)}, \theta^{(s\pm1)}}.
\end{equation}
Then, to move $H_s^{-1}$, we compose it with two transports
\begin{equation}
    H_{s+1/2}^{-1} := \mathcal{T}_{s,s-1}^* \circ H_s^{-1} \circ \mathcal{T}_{s,s-1} : T_{\theta^{(s)}}\mathcal{M} \rightarrow T_{\theta^{(s)}}\mathcal{M},
    \label{eq:fisher-transport}
\end{equation}
where $\mathcal{T}^*_{s,s-1}$ is the adjoint of $\mathcal{T}_{s,s-1}$ with respect to the baseline metric. It ensures that $\mathbf{H}_s^{-1}$ remains self-adjoint, which simplifies our analysis. In practice, one can replace the (adjoint) transport with any suitable linear map between the corresponding tangent spaces.

\begin{remark}
    %The implementation of \eqref{eq:fisher-transport} involves matrix multiplication; if the transport operators are dense matrices, then our per-iteration cost becomes $O(d^\kappa)$ where $2 < \kappa \leq 3$. For standard matrix manifolds, $\mathcal{T}$ can be chosen to have a convenient low-dimensional structure (e.g. a sum of Kronecker products). This results in a lower computation complexity, and can be substantially accelerated via parallel computing. We explore this in more detail in the experiments.
    % The implementation and computational cost of \eqref{eq:fisher-transport} are dictated by the underlying manifold structure. 
    % For standard matrix manifolds, the vector transport $\mathcal{T}$ typically possesses an exploitable algebraic structure (e.g., Kronecker product form). 
    % The presence of these structures circumvents dense matrix operations, resulting in a lower computational complexity that can be further accelerated via parallel computing. We explore this in more detail in the experiments.
    The implementation of \eqref{eq:fisher-transport} involves matrix multiplication; if the transport operators are dense matrices, then the cost is comparable to direct matrix inversion. However, one can often choose $\mathcal{T}$ with a convenient algebraic structure that reduces this cost. This is the case for most matrix manifolds of practical interest \citep{absil2009optimization}, where e.g. transports based on tangent-space projections typically decompose into a sum of Kronecker products. This results in a strictly lower computational complexity than dense matrix multiplication. Such transports are further amenable to evaluation using parallel computing.
\end{remark}

\subsection{Updating Inverse Fisher Approximation}

Following the transportation, our next task is to update $H_{s+\frac{1}{2}}^{-1}$ using score vectors from $q_{\theta^{(s)}}$
\begin{equation}
H_{s+1} = H_{s+\frac{1}{2}} + \phi_{s+1}\phi_{s+1}^\top G_{\theta^{(s)}}, \qquad \phi_{s+1} = \nabla_\theta \log{q_{\theta^{(s)}}(\bar{y}_{s+1})}, \qquad \bar{y}_{s+1} \sim q_{\theta^{(s)}}.
\end{equation}
Recall from the expression for $I_F(\theta)$ in \eqref{eq:Fisher-operator-def}, the additional $G_{\theta^{(s)}}$ term corresponds to the matrix representation of the baseline metric at $\theta^{(s)}$. To compute $H_{s+1}^{-1}$, we apply \eqref{eq:sherman-morrison}

\begin{equation} \label{eq:fisher-inversefree-update}
    H_{s+1}^{-1} = H_{s+1/2}^{-1} - (1 + \langle \phi_{s+1}, H_{s+1/2}^{-1}\phi_{s+1}\rangle_{\theta^{(s)}} )^{-1}\cdot H_{s+1/2}^{-1} \phi_{s+1} (G_{\theta^{(s)}}\phi_{s+1})^{\top} H_{s+1/2}^{-1}.
\end{equation}

Then $\mathbf{H}^{-1}_{s+1}:=(\tfrac{1}{s+1}H_{s+1})^{-1}$ is our updated approximation to $I_F^{-1}(\theta^{(s)})$. One can, of course, incorporate multiple score vectors per iteration via repeated applications of this operation.

\subsection{Natural Gradient Step} \label{sec:natural-gradient-step}

For the natural gradient step, let $\nabla \mathcal{L}$ denote the Riemannian gradient of the objective $\mathcal{L}$ and $\widehat{\nabla \mathcal{L}}$ its stochastic estimate. In our analysis, we assume these take the following form
\begin{equation} \label{eq:stochastic-gradient}
    \nabla \mathcal{L}(\theta^{(s)}) = \mathbb{E}_{Y \sim p}[g (Y; \theta^{(s)})], \qquad \widehat{\nabla \mathcal{L}}(\theta^{(s)}) = \frac{1}{B}\sum^B_{i=1} g(y_{s+1,i}; \theta^{(s)}),
\end{equation}
where $y_{s+1,i} \sim p$, and $p$ is some distribution which can possibly depend on $\theta^{(s)}$. This framework covers several important problems including the MLE problem in \eqref{eq:MLE loss} where $p$ is the empirical distribution and $g(y;\theta)=-\nabla_\theta \log{q_\theta(y)}$.
In the VB problem \eqref{eq:VB loss}, $p=q_{\theta^{(s)}}$ and 
\begin{align}\label{eq:VB g function}
    g(y;\theta)= \log{\frac{q_\theta(y)}{\bar\pi(y)}}\times \nabla_\theta \log{q_\theta (y)}.
\end{align}
% The gradient $\nabla \mathcal{L}(\theta^{(s)}) = \mathbb{E}_{Y \sim p}[g (Y; \theta^{(s)})]$ with $g(y;\theta)$ in \eqref{eq:VB g function} is a Riemannian generalization of the so-called (Euclidean) {\it score-function gradient} in the VB literature; the derivation is essentially the same as in the Euclidean setting \citep{godichon2024natural}.
% The reparameterization gradient \citep{kingma2013auto} also fits into this framework.

The gradient $\nabla \mathcal{L}(\theta^{(s)}) = \mathbb{E}_{Y \sim p}[g (Y; \theta^{(s)})]$ with $g(y;\theta)$ in \eqref{eq:VB g function} is a Riemannian generalization of the so-called (Euclidean) {\it score-function gradient} in the VB literature; see the derivation in Appendix \ref{appendix:score and reparameterization gradient}.
This appendix also presents a generalization of the {\it reparameterization-trick gradient} \citep{kingma2013auto}.

%is the variational density of the current parameter. Let $\pi(y)$ denote a (potentially un-normalized) target distribution, and $\mathcal{L}(\theta) = \mathbb{E}_{Y\sim q_\theta}[ \log{\tfrac{q_\theta(Y)}{\pi(Y)}}]$ the negative ELBO, its gradient is then
%\begin{equation}
%   \nabla \mathcal{L}(\theta) = \mathbb{E}_{Y\sim q_\theta}[\log{\tfrac{q_\theta(Y)}{\pi(Y)}}\times \nabla_\theta \log{q_\theta (Y)}].
%\end{equation}
%The details are identical to the Euclidean setting; see \citet{godichon2024natural}.

For a step size sequence $\tau_s$, the approximate natural gradient update is then
\begin{equation} \label{eq:approx-ngd-update}
    \theta^{(s+1)} := \mathcal{R}_{\theta^{(s)}}\left(v_{s+1}\right), \qquad v_{s+1} := -\tau_{s+1} \mathbf{H}^{-1}_{s+1} \widehat{\nabla \mathcal{L}}(\theta^{(s)}).
\end{equation}

\begin{algorithm}
    \caption{Riemannian Inverse Free Natural Gradient Descent} \label{alg:riemannian-if-ngd}
    \begin{algorithmic}
        \Require Initial parameter $\theta^{(0)}=\theta^{(-1)}$; $\epsilon>0$; step sizes $\tau_s$; transport $\mathcal{T}$; retraction $\mathcal{R}$.
        \State $H_0 = \epsilon I \in T_{\theta^{(0)}}\mathcal{M}$
        \For{$s= 0,1,...$}
            \State Transport $H_s^{-1}$ from $T_{\theta^{(s-1)}}\mathcal{M}$ to $T_{\theta^{(s)}}\mathcal{M}$ using \eqref{eq:fisher-transport}, yielding $H_{s+1/2}^{-1}$.
            \vspace{0.2cm}
            \State Sample $\bar{y}_{s+1} \sim q_{\theta^{(s)}}$ and let $\phi_{s+1} = \nabla_\theta \log{q_{\theta^{(s)}}(\bar{y}_{s+1})}\in T_{\theta^{(s)}}\mathcal{M}$.
            \vspace{0.2cm}
            \State Update $H_{s+1/2}^{-1}$ using $\phi_{s+1}$ via \eqref{eq:fisher-inversefree-update}, yielding $H_{s+1}^{-1}$; let $\mathbf{H}^{-1}_{s+1} = (s+1)H_{s+1}^{-1}$.
            \vspace{0.2cm}
            \State Compute stochastic gradient $\widehat{\nabla \mathcal{L}}(\theta^{(s)})$ using $y_{s+1,i} \sim p$ and \eqref{eq:stochastic-gradient}.
            \vspace{0.2cm}
            \State Compute update direction $v_{s+1} = -\tau_{s+1}\mathbf{H}_{s+1}^{-1}\widehat{\nabla \mathcal{L}}(\theta^{(s)})$.
            \vspace{0.2cm}
            \State Update manifold iterate $\theta^{(s+1)}=\mathcal{R}_{\theta^{(s)}}(v_{s+1})$.
            \vspace{0.2cm}
        \EndFor \newline
        \Return $\theta^{(s)}$
    \end{algorithmic}
\end{algorithm}

% % SHORTEST VERSION - Use this if we decide to add a third example?

\subsection{Limited memory approximation}\label{rem:limited-memory-approximation} For high dimensional problems, it can be infeasible to store $\mathbf{H}_s^{-1}$ as a dense $d\times d$ matrix. One alternative is to maintain a ``sliding window'' approximation, which only incorporates the $K \ll d$ most recent score vectors. Let $\theta \in \mathcal{M}$, and $\phi_1,...,\phi_K \in T_\theta \mathcal{M}$ denote these vectors after re-ordering the indices from newest to oldest. Then, by repeated application of the Sherman-Morrison formula to $H_K:=H_0+\sum^K_{s=1}\phi_s\phi_s^\top G_\theta$, we have the following vectorized representation of the inverse

\begin{equation} \label{eq:inv-moving-approximation}
    H^{-1}_K = \frac{1}{\epsilon}I-\sum^{K-1}_{s=0} c_s \psi_s \psi_s^\top G_\theta,
\end{equation}
where $\psi_s := H_s^{-1} \phi_{s+1}$ and $c_s:= (1+\langle \phi_{s+1}, \psi_{s}\rangle_\theta)^{-1}$. This representation requires storing $K$ $\psi$-vectors and scalars rather than a full matrix. In Appendix \ref{appendix:vectorized-fisher}, we show that transporting $H_K^{-1}$ reduces to transporting each $\psi_s$. Furthermore, the oldest score vector can be removed, and a new one incorporated via a recursive procedure with $O(dK)$ complexity. This approach resembles that of \citet[Remark 4.2]{godichon2024natural}, but tracks the exact inverse of the moving average rather than an approximation.

\section{Convergence Analysis}\label{sec:Convergence Analysis}

In this section we analyse the convergence of the Riemannian inversion-free natural gradient descent method described in the previous section (Algorithm \ref{alg:riemannian-if-ngd}). The results and their proofs are inspired by \citet{godichon2024natural}. We begin by presenting our assumptions.

\begin{assumption} \label{assumption:gradient}
    The Riemannian (stochastic) gradient of the objective takes the form
    \begin{equation}
    \nabla \mathcal{L}(\theta^{(s)}) = \mathbb{E}_{Y\sim p}[g(Y;\theta^{(s)})]\in T_{\theta^{(s)}} \mathcal{M}, \qquad \widehat{\nabla \mathcal{L}}(\theta^{(s)}) = \tfrac{1}{B}\sum^B_{i=1}g(y_{s+1,i};\theta^{(s)}),
\end{equation}
where $y_{s+1,i} \sim p$ and $p$ is some distribution which may optionally depend on $\theta$.
\end{assumption}

As mentioned earlier, this framework covers several important problems; for example, maximum likelihood estimation and variational Bayes. 

\begin{assumption} \label{assumption:retraction}
    The algorithm iterates are restricted to a compact set $\mathcal{X}\subset \mathcal{M}$ which is \textit{totally retractive} w.r.t. the second-order retraction $\mathcal{R}$. The transport is $\mathcal{T}_{x,y} := D\mathcal{R}_x(\mathcal{R}^{-1}_x(y))$.
\end{assumption}

The assumptions involving $\mathcal{X}$ are common in stochastic Riemannian optimisation methods involving retractions \citep{bonnabel2013stochastic, zhang2016first, sato2019riemannian}. The assumption of compactness is only used to ensure certain regularity properties of $\mathcal{T}$ (\Cref{lemma:vector-transport-bounds}, Appendix \ref{appx:helpful-results}), which are necessary for the proof of Theorem \ref{thm:global-convergence}, and can sometimes be omitted when $\mathcal{R}=\exp$ (\Cref{rem:no-compactness condition}, Appendix \ref{appx:helpful-results}). The notion of a totally retractive neighborhood is described in \citet{huang2015broyden}, and ensures that $\mathcal{T}$ and $\mathcal{R}^{-1}$ are well-defined wherever they appear. Second-order retractions are commonly available for many matrix manifolds \citep{absil2012projection}. To streamline our presentation, we assume that $\mathcal{T}=D\mathcal{R}$, which has been used in e.g. the literature on the Riemannian BFGS algorithm \citep{ring2012optimization, huang2015broyden}. However, this is not strictly necessary, but provides a neat sufficient condition for the more technical conclusions of Lemmas \ref{lemma:tripuraneni-averaging} - \ref{lemma:transport-pt} (Appendix \ref{appx:helpful-results}).

\begin{assumption} \label{assumption:step-size}
    The step size sequence is $\tau_k = \tfrac{c_\alpha}{(c_\alpha'+k)^\alpha}$ where $c_\alpha, c_\alpha' >0$ and $\alpha \in (\tfrac{1}{2},1)$.
\end{assumption}

\begin{assumption} \label{assumption:eigenvalues}
    The eigenvalues of $\mathbf{H}^{-1}_s$ are bounded above by $O(s^\beta)$ a.s., $\beta \in [0,\alpha-\tfrac{1}{2}).$
\end{assumption}

The above assumption is required to ensure that the iterates of the algorithm do not diverge. Such explicit eigenvalue bounds are common in the analysis of adaptive stochastic optimisation methods \citep{godichon2023non, boyer2023asymptotic, godichon2025adaptive}. In \citet{godichon2024natural}, the authors circumvent this assumption by incorporating a tapered noise sequence into the construction of $\mathbf{H}_s$, which prevents its eigenvalues from decreasing faster than $O(s^{-\beta})$. The strategy appears viable in the Riemannian setting, albeit with restrictions on $\mathcal{T}$, as repeated application of non-isometric vector transport can distort the decay rate of the noise sequence.

The following theorem concerns the global behavior of the iterates of the algorithm.
%demonstrates that the iterates of our algorithm converge almost-surely to a stationary point of the objective.

\begin{theorem} \label{thm:global-convergence}
    Given assumptions \ref{assumption:gradient}-\ref{assumption:eigenvalues}, and the following conditions on $\mathcal{X}$:
    \begin{enumerate}
        \item The objective $\mathcal{L}$ is differentiable and $L_0$-smooth\footnote{For all $x\in \mathcal{X}$ and $v\in T_x\mathcal{M}$, $t\rightarrow\mathcal{L}\circ\mathcal{R}_x(tv)$ is $L_0$-smooth for $t\in \mathbb{R}$ such that $\mathcal{R}(tv)\in \mathcal{X}$. The notions of ordinary/strict/strong convexity with respect to $\mathcal{R}$ are defined similarly.} with respect to $\mathcal{R}$.
        \item There exists a stationary point $\theta^*$ such that $\nabla \mathcal{L}(\theta^*) = 0$.
        \item There exist constants $C_0, C_1 \geq 0$ such that for all $\theta$ visited by the algorithm:
        \begin{equation} \label{assumption:g-2nd-moment}
            \mathbb{E}_{Y \sim p}\left[ \Vert g(Y; \theta)\Vert_\theta^2\right] \leq C_0 + C_1 (\mathcal{L}(\theta) - \mathcal{L}(\theta^*)).
        \end{equation}
        \item There exist constants $C_0',C_1'>0$ such that for all $\theta$ visited by the algorithm:
        \begin{equation} \label{assumption:q-4th-moment}
            \mathbb{E}_{Y \sim q_\theta}\left[ \Vert \nabla_\theta \log q_\theta(Y)\Vert^4_\theta\right] \leq C_0' + C_1' (\mathcal{L}(\theta) - \mathcal{L}(\theta^*))^2.
        \end{equation}
    \end{enumerate}
    Then the iterates $\theta^{(s)}$ generated by the algorithm satisfy:
    \begin{itemize}
        \item $\mathcal{L}(\theta^{(s)})-\mathcal{L}(\theta^*)$ converges almost surely to a finite random variable.
        \item $\min_{k=0}^s \Vert \nabla \mathcal{L}(\theta^{(k)})\Vert^2_{\theta^{(k)}} = o(s^{-(1-\alpha)})$ almost surely.
    \end{itemize}
\end{theorem}

 Note, if e.g. $\mathcal{L}$ is strictly convex with respect to $\mathcal{R}$ and has bounded level sets, then $\theta^{(s)}$ converges a.s. to a unique minimizer.
 % or satisfies a Riemannian Polyak-Łojasiewicz inequality\footnote{That is, we have $\tfrac{1}{2}\Vert \nabla \mathcal{L}(\theta)\Vert^2_{\theta} \geq \mu(\mathcal{L}(\theta) - \mathcal{L}(\theta^*))$. This holds if, e.g. $\mathcal{L}$ is $\mu$-strongly convex along geodesics.}, then $\theta^{(s)}$ converges almost surely to a minimizer. 
 See Appendix \ref{appx:proof-global-convergence} for a proof of Theorem \ref{thm:global-convergence}. 
 
 The following theorem shows that our method attains the usual rate of convergence for stochastic gradient algorithms with polynomial step size when $\alpha >\tfrac{2}{3}$. 
 The proof strategy resembles the approach in \citet{godichon2024natural}, albeit involving the linearised iterates $\mathcal{R}^{-1}_{\theta^*}(\theta^{(s)})\in T_{\theta^*}\mathcal{M}$ in the tangent space of the limit point. This linearisation process introduces an additional error term in our analysis, which becomes dominant when $\alpha \leq \tfrac{2}{3}$, leading to a sub-optimal convergence rate in this regime.
 % The reduced convergence rate for $\alpha \in (\tfrac{1}{2}, \tfrac{2}{3}]$ is a consequence of the curvature of the manifold.

\begin{theorem} \label{thm:convergence-rate}
    Given assumptions \ref{assumption:gradient}-\ref{assumption:eigenvalues}, conditions (3) and (4) from Theorem \ref{thm:global-convergence}, and:
    \begin{enumerate}
        \item The sequence of iterates $\theta^{(s)}$ converges to $\theta^*$ almost surely.
        \item The objective $\mathcal{L}$ is twice-differentiable in a neighborhood of $\theta^*$ with $\nabla \mathcal{L}(\theta^*) = 0$.
        \item The Riemannian Hessian\footnote{With respect to the Levi-Civita connection of the base metric.} $\nabla^2 \mathcal{L}(\theta^*)$ and $I_F(\theta^*)$ are positive-definite.
        \item There exists a constant $\eta > \tfrac{1}{\alpha} -1$ and $C_{\eta,0}, C_{\eta, 1}>0$ such that for all $\theta \in \mathcal{X}$:
        \begin{equation}
            \mathbb{E}_{Y \sim p}[\Vert g(Y;\theta)\Vert^{2+2\eta}_{\theta}] \leq C_{\eta, 0} + C_{\eta, 1}(\mathcal{L}(\theta) - \mathcal{L}(\theta^*))^{1+\eta}.
        \end{equation}
        \vspace{-1.5cm}
    \end{enumerate}
    Let $d(\cdot, \cdot)$ denote the Riemannian distance with respect to the baseline metric, then for $\delta>0$
        \begin{equation}
            d(\theta^{(s)},\theta^*)^2 = O\left(\max{\left(\frac{\log{s}}{s^\alpha}, \frac{\log{s}^{4+\delta}}{s^{4\alpha-2}}\right)} \right), \qquad \text{a.s.}
        \end{equation}
    In particular, we have $d(\theta^{(s)},\theta^*)^2 = O(\tfrac{\log{s}}{s^\alpha})$ when $\alpha > \tfrac{2}{3}$.
\end{theorem}

See Appendix \ref{appx:convergence-rate-proof} for a proof of \Cref{thm:convergence-rate}. Finally, we can prove that our approximation to the Fisher operator is asymptotically exact. In \citet{godichon2024natural} this follows from consistency $\theta^{(s)} \rightarrow \theta^{*}$. Here, our strategy involves formulating a recursive expression for $\mathbf{H}_s$ localized in $T_{\theta^*}\mathcal{M}$, which requires transportation along geodesic triangles $\theta^* \rightarrow \theta^{(n)} \rightarrow \theta^{(n+1)} \rightarrow \theta^*$. This produces errors due to holonomy; loosely, the failure of parallel transport around a loop to return a vector to its original state. The errors are proportional to the ``area'' of these triangles, and controlling them requires that $\alpha > \tfrac{2}{3}$ and the utilisation of the preceding theorem. It is intriguing to consider whether a non-localised analysis is possible, and whether this can yield convergence rates in the $\alpha \leq \tfrac{2}{3}$ regime.

% These errors are a consequence of holonomy; loosely, the failure of parallel transport along a closed loop to return a vector to its original state. The condition ensures the iterate trajectory is sufficiently regular for the resulting errors to remain negligible.

\begin{theorem} \label{thm:fisher-consistency}
    Assuming conditions of Theorem \ref{thm:convergence-rate}, and $I_F(\theta)$ is locally Lipschitz\footnote{We take this to mean that there exists an open neighborhood $U \subseteq \mathcal{M}$ around $\theta^*$ and a constant $K>0$ s.t. $\forall \theta \in U$ we have $\Vert \Gamma_{\theta,\theta^*} \circ I_F(\theta) \circ \Gamma_{\theta^*,\theta} -I_F(\theta^*)\Vert_{\mathrm{op}}\leq K d(\theta,\theta^*)$. See e.g. Section 10.4 in \citet{boumal2023introduction}} at $\theta^*$
    \begin{equation}
        \Vert\Gamma_{\theta^{(s)}, \theta^*} \circ \mathbf{H}_{s+1} \circ \Gamma_{\theta^*,\theta^{(s)}} - I_F(\theta^*)\Vert_{\text{op}} = O\left(\frac{\log{s}^{1+\delta}}{s^{3\alpha/2-1}}\right), \qquad \text{a.s.}
    \end{equation}
    For $\alpha \in (\tfrac{2}{3}, 1)$, where $\delta>0$ and $\Gamma_{x,y}$ denotes geodesic parallel transport from $x\rightarrow y$.
\end{theorem}

% Finally, we prove that the averaged iterate $\bar{\theta}^{(s)}$ is asymptotically efficient, and that the averaged iterates linearized in the tangent space of the stationary point $\mathcal{R}^{-1}_{\theta^*}(\bar{\theta}^{(s)})$ satisfy a central limit theorem.

% \begin{theorem} \label{thm:averaged-convergence-rate}
%     Assume the following conditions hold:
%     \begin{itemize}
%         \item The conditions stated in Theorem \ref{thm:convergence-rate}.
%         \item The function $\theta \rightarrow \Sigma(\theta) := \mathbb{E}_{Y \sim q_\theta}\left[\langle g(Y;\theta), \cdot \rangle_\theta \langle g(Y;\theta), \cdot\rangle_\theta \right]$ is continuous at $\theta^*$.
%         \item We have $\alpha \in (3/4,1)$.
%     \end{itemize}
%     Then, letting $\bar{\theta}^{(s)} = R_{\theta^{(s)}}(w_s)$
%     \begin{equation}
%         d(R_{\theta^{(s)}}(w_s), \theta^*)^2 = O\big( \frac{\log{s}}{s}\big), \qquad \text{a.s.}
%     \end{equation}
%     Furthermore, in an orthornomal basis on $T_{\theta^*} \mathcal{M}$
%     \begin{equation}
%         \sqrt{Bs} R_{\theta^*}^{-1}(\bar{\theta}^{(s)}) \underset{\text{law}}{\rightarrow} \mathcal{N}(0, \nabla^2 \mathcal{L}(\theta^*)^{-1} \Sigma (\theta^*) \nabla^2 \mathcal{L}(\theta^*)^{-1})
%     \end{equation}
% \end{theorem}

See Appendix \ref{appx:fisher-consistency-proof} for a proof of \Cref{thm:fisher-consistency}.

\begin{remark}
    In \cite{godichon2024natural}, they incorporate a Polyak-Ruppert averaging step, and demonstrate that the averaged iterate $\bar{\theta}^{(s)}$ converges to $\theta^*$ in the squared Euclidean distance as $O(\log{s}/s)$ almost surely, and that $\sqrt{s}(\bar{\theta}^{(s)}-\theta^*)$ satisfies a central limit theorem. Polyak-Ruppert averaging has been studied in the manifold setting; see e.g. \citet{tripuraneni2018averaging}. However, we leave this direction to future work.
\end{remark}

\section{Example: Gaussian Variational Inference}\label{sec:Example: Bures-Wasserstein Natural Gradient}

In this section we study Gaussian variational inference for a Bayesian logistic regression model. The natural gradient has a simple closed expression for Gaussian distributions, so we can assess the quality of the inverse-free updates relative to their ground truth, and gauge the effect of the geometry. Specifically, we compare exact Fisher-preconditioned natural gradient methods against their approximate (i.e., inverse-free) counterparts in both a Euclidean and a Riemannian baseline geometry. Full details for this section are provided in Appendix \ref{appx:bw-nat-grad}.

\paragraph{Model \& Objective:} Let $\mathcal{D}=\{x_i,y_i\}^n_{i=1}$ denote a classification dataset with features $x_i \in \mathbb{R}^d$ and binary responses $y_i \in \{0,1\}$. Let $\sigma^2>0$ be the prior variance; we consider
\begin{equation} \label{eq:logistic-regression}
    y_i \mid x_i, \beta \sim \mathrm{Bernoulli}\left((1+\exp(-\beta^\top x_i))^{-1}\right), \qquad \beta \sim \mathcal{N}(0, \sigma^2I_d).
\end{equation}
We approximate the posterior $\pi(\beta \mid \mathcal{D})$ with a Gaussian variational density $q_{\mu,\Sigma}$ by minimising the KL divergence (via the negative ELBO); see \eqref{eq:VB loss}.%; see Section \ref{sec:Background} for more details.

% \begin{equation} \label{eq:gauss-kl-minimize}
%     (\mu^*, \Sigma^*)=\argmin_{(\mu,\Sigma) \in \mathbb{R}^d \times S_d^{++}}\big\{\mathcal{L}(\mu,\Sigma) := \mathbb{E}_{\beta \sim q_{\mu,\Sigma}}[\log{q}_{\mu,\Sigma}(\beta)-\log p(\beta, \mathcal{D})]\big\}.
% \end{equation}

\paragraph{Baseline Manifold:} For the Riemannian inverse-free algorithm, we employ the Bures-Wasserstein manifold $\mathrm{BW}(\mathbb{R}^d)$ as the baseline Riemannian structure. This corresponds to the 2-Wasserstein space on $\mathbb{R}^d$ restricted to Gaussian distributions \citep{takatsu2011wasserstein}. \citet{lambert2022variational} and \citet{diao2023forward} studied variational inference on $\mathrm{BW}(\mathbb{R}^d)$ via Riemannian (proximal) gradient descent methods. The KL divergence is geodesically convex on $\mathrm{BW}(\mathbb{R}^d)$ when the $\pi$ is log-concave\footnote{That is, $\pi \propto \exp(-V)$ where $V:\mathbb{R}^d\rightarrow \mathbb{R}$ is convex. This is true for the model in \eqref{eq:logistic-regression}}. However, this is not guaranteed for the Fisher geometry, or the Euclidean covariance parameterisation \citep{challis2013gaussian}.

\paragraph{Algorithms:} The comparison follows a $2 \times 3$ factorial design based on the following update rule, with step size $\tau_s$, and $\widehat{\nabla \mathcal{L}}$ a stochastic gradient associated with the baseline geometry 
\begin{equation}
\theta_{s+1} = \mathcal{R}_{\theta_s} \left(-\tau_s \cdot  P_s[\widehat{\nabla \mathcal{L}}(\theta_s)]\right), \quad \mathrm{where} \quad \theta_s:=(\mu_s, \Sigma_s).
\end{equation}
The first choice concerns the baseline geometry and retraction $\mathcal{R}$; here we have:
\begin{itemize}[nosep, leftmargin=*]
    \item \textbf{Euclidean (``Euc'')}: additive updates in $(\mu,\Sigma)$ using full covariance parametrisation.
    \item \textbf{Bures-Wasserstein (``BW'')}: updates to $(\mu,\Sigma)$ use the $\mathrm{BW}(\mathbb{R}^d)$ exponential map.
\end{itemize}
The second choice is the linear operator $P_s$ (preconditioner):
\begin{itemize}[nosep, leftmargin=*]
    \item \textbf{Identity (``GD'')}: $P_s=\text{Id}$, Riemannian gradient in baseline geometry.
    \item \textbf{Exact (``NGD'')}: $P_s = I^{-1}_F(\theta^{(s)})$, natural gradient direction with exact Fisher.
    \item \textbf{Inverse-Free (``NGD Approx.'')}: $P_s = \mathbf{H}^{-1}_{s+1}$, Algorithm \ref{alg:riemannian-if-ngd} in baseline geometry.
\end{itemize}
The methods are indexed as (Geometry)-(Preconditioner).

\paragraph{Implementation:} For each method, the update direction is based on a stochastic estimate of the \textit{score function gradient} \eqref{eq:VB g function} with a Monte-Carlo sample of $100$. The step size takes the form $\tau_s = c_0/(100+s)^{\alpha}$, where $c_0,\alpha$ are selected via grid search for each method and dataset. Following covariance updates, we clip eigenvalues to $[10^{-8},\infty)$.
% ; this is also necessary on $\mathrm{BW}(\mathbb{R}^d)$ since the manifold is not complete. 
For the inverse-free method on $\mathrm{BW}(\mathbb{R}^d)$ we employ the differentiated exponential map as the transport.

% % ------------------------------------------------------------------------
% % The following code shows the figures in a 3 x 2 block arrangement.
% % ------------------------------------------------------------------------

% \vspace{0.5em}
% \begin{figure}[h]
%     \centering
%     % Row 1
%     \begin{subfigure}{0.4\textwidth}
%         \includegraphics[width=\linewidth]{logreg_ionosphere_cov_2.pdf}
%     \end{subfigure} 
%     \hspace{1em}
%     \begin{subfigure}{0.4\textwidth}
%         \includegraphics[width=\linewidth]{logreg_ionosphere_bw_2.pdf}
%     \end{subfigure}

%     \vspace{0.5em}

%     % Row 2
%     \begin{subfigure}{0.4\textwidth}
%         \includegraphics[width=\linewidth]{logreg_breastcancer_cov_2.pdf}
%     \end{subfigure}
%     \hspace{1em}
%     \begin{subfigure}{0.4\textwidth}
%         \includegraphics[width=\linewidth]{logreg_breastcancer_bw_2.pdf}
%     \end{subfigure}
    
%     \vspace{0.5em}
    
%     % Row 3
%     \begin{subfigure}{0.4\textwidth}
%         \includegraphics[width=\linewidth]{logreg_sonar_cov_2.pdf}
%     \end{subfigure}
%     \hspace{1em}
%     \begin{subfigure}{0.4\textwidth}
%         \includegraphics[width=\linewidth]{logreg_sonar_bw_2.pdf}
%     \end{subfigure}

%     \caption{NELBO versus iteration for the Bayesian logistic regression model. Left column: Euclidean covariance parameterisation. Right column: Bures-Wasserstein manifold. Datasets from top to bottom: Ionosphere, Breast Cancer, Sonar.}
%     \label{fig:logreg-small}
% \end{figure}

\vspace{0.5em}
\begin{figure}[h]
    \centering
    % Row 1
    \begin{subfigure}{0.32\textwidth}
        \includegraphics[width=\linewidth]{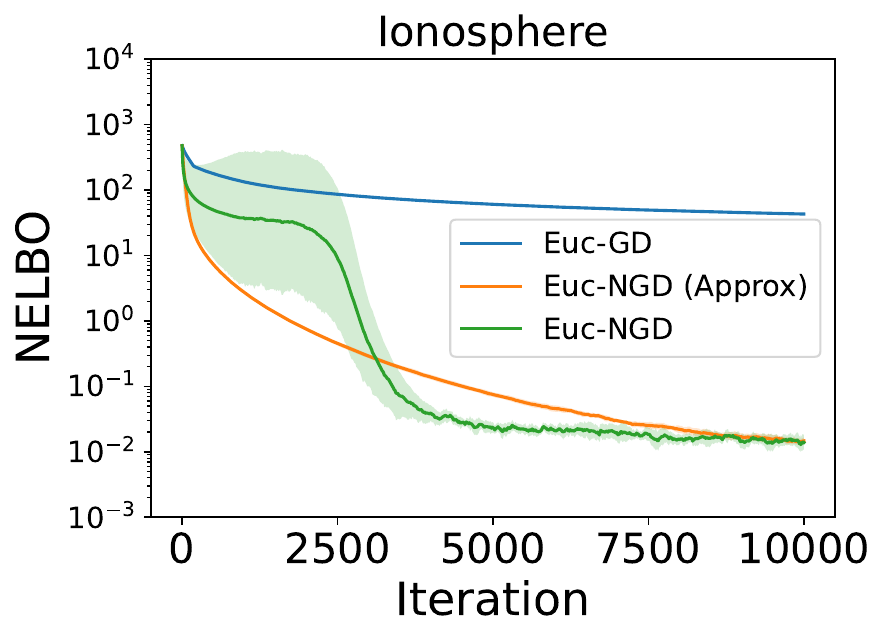}
    \end{subfigure}
    \hfill
    \begin{subfigure}{0.32\textwidth}
        \includegraphics[width=\linewidth]{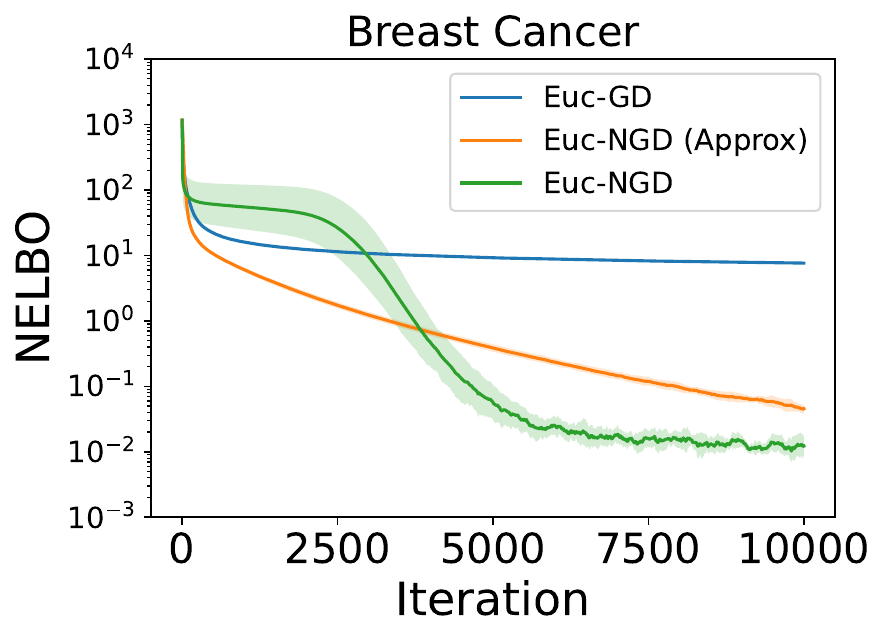}
    \end{subfigure}
    \hfill
    \begin{subfigure}{0.32\textwidth}
        \includegraphics[width=\linewidth]{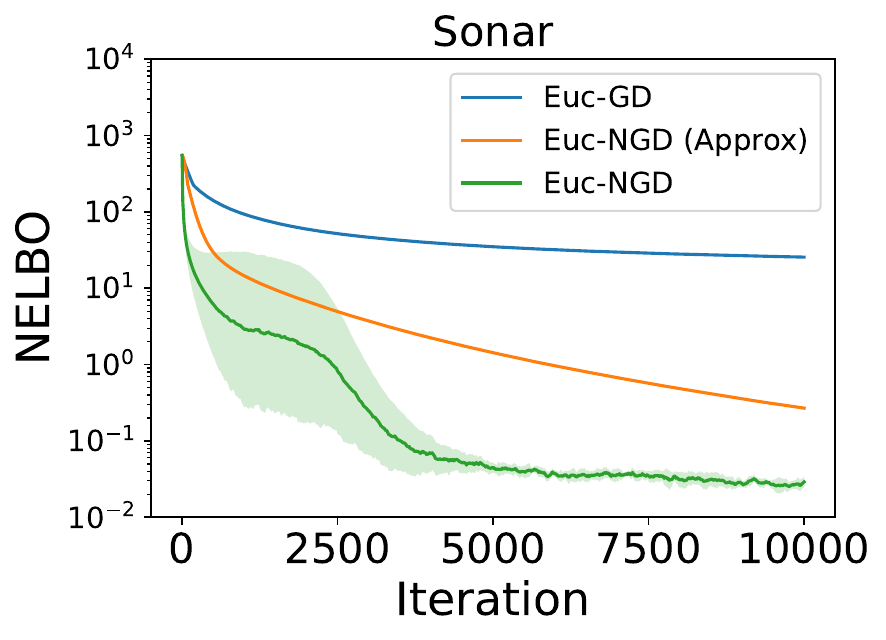}
    \end{subfigure}

    \vspace{0.5em} % vertical space between rows

    % Row 2
    \begin{subfigure}{0.32\textwidth}
        \includegraphics[width=\linewidth]{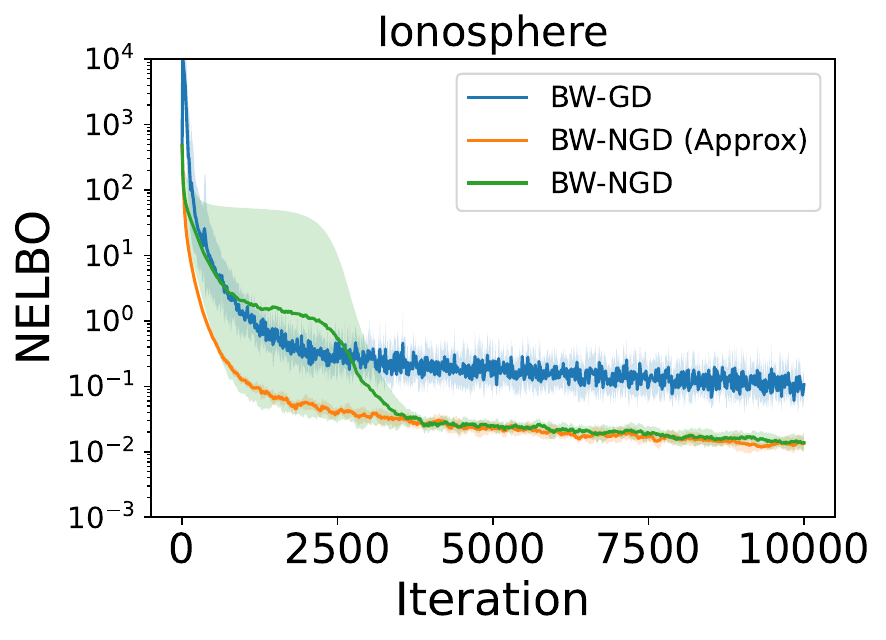}
    \end{subfigure}
    \hfill
    \begin{subfigure}{0.32\textwidth}
        \includegraphics[width=\linewidth]{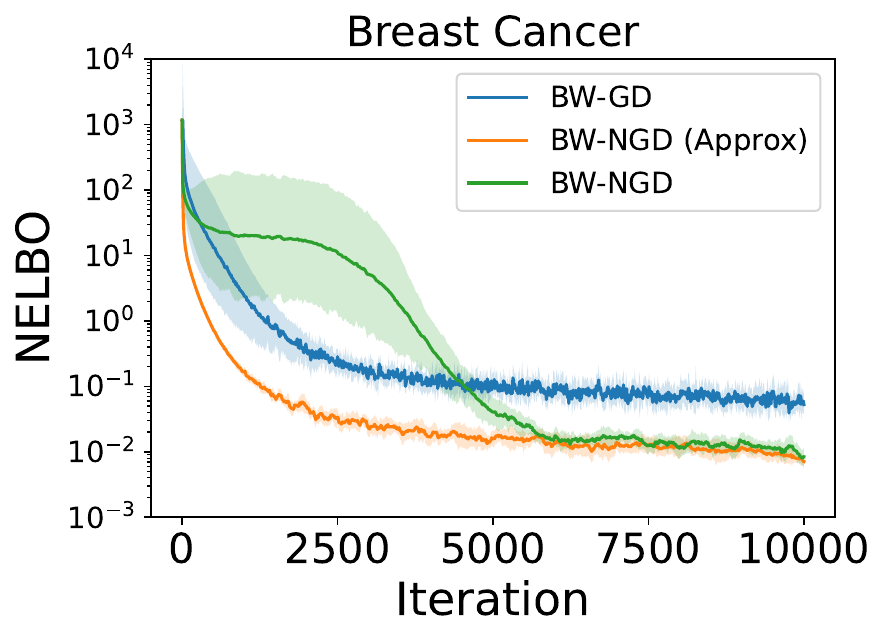}
    \end{subfigure}
    \hfill
    \begin{subfigure}{0.32\textwidth}
        \includegraphics[width=\linewidth]{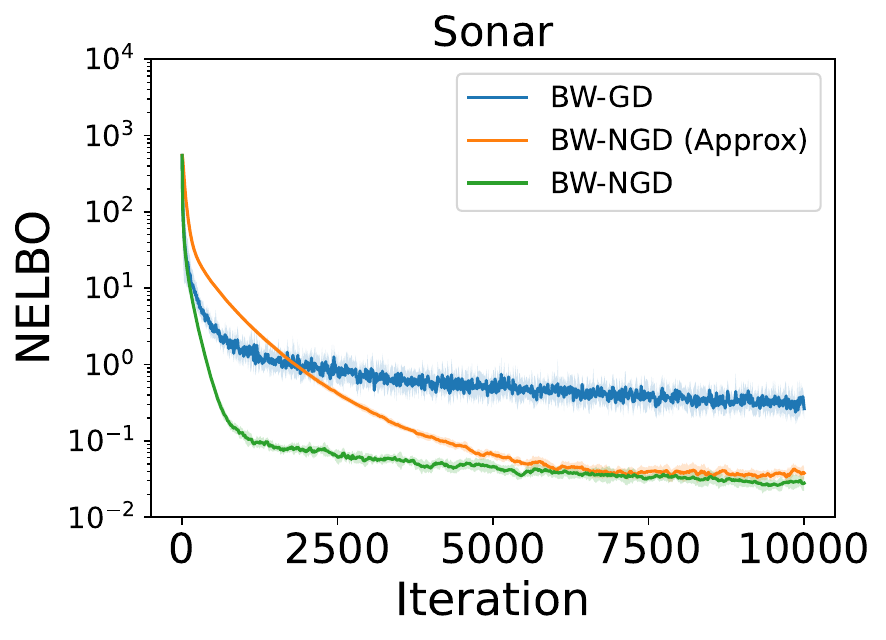}
    \end{subfigure}

    \caption{NELBO versus iteration for the Bayesian logistic regression model. Top row: Euclidean covariance parameterisation. Bottom row: Bures-Wasserstein manifold. Datasets from left to right: Ionosphere, Breast Cancer, Sonar.}
    \label{fig:logreg-small}
\end{figure}

\vspace{-0.5cm}
\paragraph{Smaller Datasets:} In \Cref{fig:logreg-small}, we consider three standard datasets from the UCI repository\footnote{\url{https://archive.ics.uci.edu/ml/index.php}}:  Ionosphere $(n=351, p=34)$, Breast Cancer Wisconsin Diagnostic $(n=569, p=30)$, and Sonar, Mines vs. Rocks $(n=208, p=60)$. The first row depicts the Euclidean methods, while the second row employs $\mathrm{BW}(\mathbb{R}^d)$ as the baseline Riemannian structure. Each figure reports the mean NELBO and standard error across ten runs initialized at $(\mu,\Sigma)=(0,I)$. The values are relative to a long run of exact natural gradient descent. 
%On all three datasets, BW-NGD Approx. converges more quickly than Cov-NGD Approx. and achieves a lower final NELBO, which is comparable to the exact methods. Both exact natural gradient methods (Cov-NGD and BW-NGD) exhibited considerable cross-run variation in early iterations but eventually converged to the same NELBO value.
We draw several conclusions.
First, regardless of the baseline metric, incorporating the Fisher information improves training efficiency. This is evidenced by the superior performance of Euc-NGD relative to Euc-GD and of BW-NGD relative to BW-GD.
Second, the approximate natural gradient performs comparably to its exact counterpart, albeit with a slightly longer initial stabilization phase.
Third, BW-NGD (Approx.) converges faster than Euc-NGD (Approx.) and attains a better final ELBO, supporting the effectiveness of the BW baseline metric.

\begin{figure}[H]
  \centering
  \begin{subfigure}[b]{0.524\textwidth}
    \centering
    \includegraphics[width=\linewidth]{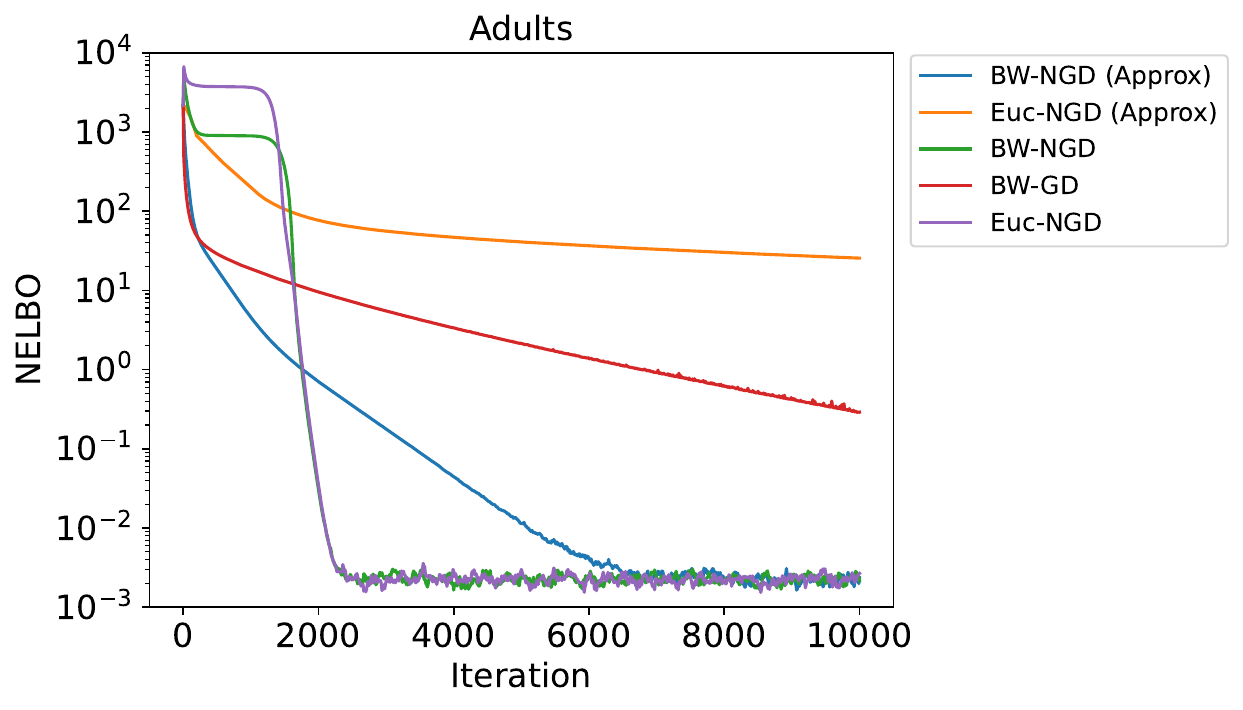}
    \label{fig:logreg-mushrooms}
  \end{subfigure}
  \begin{subfigure}[b]{0.4\textwidth}
    \centering
    \includegraphics[width=\linewidth]{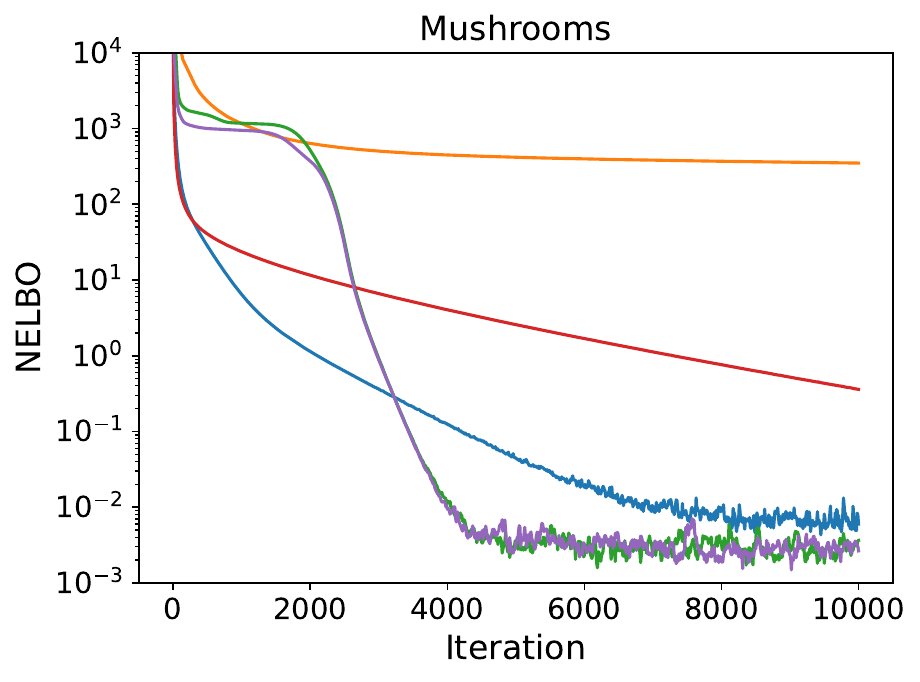}
    \label{fig:logreg-adults}
  \end{subfigure}
  \vspace{-0.5cm}
  \caption{Comparison of approximate/exact natural gradient algorithms on larger datasets.}
  \label{fig:logreg-large}
\end{figure}

\paragraph{Larger Datasets:} The discrepancy between the Riemannian and Euclidean inverse-free approaches becomes even more apparent in higher dimensions. In \Cref{fig:logreg-large}, we apply these methods on two LIBSVM datasets: Mushrooms $(n=8124, p=112)$, and the training subset of Adult ``a1a'' $(n=1605, p=123)$. Here Euc-NGD (Approx.)~was considerably more unstable, and would diverge unless the initial step size was set to a small value ($\tau_0 <10^{-4}$). Conversely, BW-NGD (Approx.)~remained stable at larger step sizes ($\tau_0 \approx 10^{-2}$), and yielded a final ELBO value comparable to the exact methods.
This again demonstrates the effectiveness of the BW baseline metric compared to the Euclidean metric.

\section{Example: Stiefel Manifold Natural Gradient}\label{sec:stiefel-experiment}

In this section we consider variational Bayes inference with a normalising flow as the variational family \citep{rezende2015variational}. Normalizing flows are highly flexible, making them well-suited for approximating complex posterior distributions. We employ the Sylvester flow variant proposed by \citet{berg2018sylvester}, which employs weight matrices constrained to have orthonormal columns. This yields a variational family with no closed-form Fisher information, and parameters which lie on the Stiefel manifold. Details are in Appendix~\ref{appx:Steifel-manifold-experiment}.

\paragraph{Model \& Objective:} We consider a Bayesian neural network for binary classification,
\begin{equation*}
    y_i\mid w, x_i \sim \mathrm{Bernoulli}(\pi_w(x_i)), \qquad w \sim \mathcal{N}(0,cI).
\end{equation*}
where $\pi_w(x_i) \in [0,1]$ is the output of a single-hidden-layer network with $10$ hidden units, input $x_i$, weights $w$, and with a ridge-type regularization prior on $w$. The task is to approximate the posterior distribution of $w$ using variational Bayes. 

The variational distribution $q_\theta$ is based on a two-layer neural network,
\begin{equation}\label{eq:SylvesterNF}
    \epsilon \sim \mathcal{N}_d(0,I), \quad Z = \sigma(W_1\epsilon + b_1), \quad Y = W_2 Z + b_2,
\end{equation}
where $W_1, W_2 \in \mathbb{R}^{d \times d}$, $b_1, b_2 \in \mathbb{R}^d$, $\sigma(\cdot)$ is the sigmoid activation, and we impose the orthogonality constraints $W_1^\top W_1 = W_2^\top W_2=I_d$. Following \citet{berg2018sylvester}, this simplifies the Jacobian determinants in the density of $q_\theta$, enabling efficient computation of the ELBO gradient. The variational parameter $\theta = (W_1, W_2, b_1, b_2)$ thus belongs to the product manifold $\mathcal{M} = \mathrm{St}(d,d) \times \mathrm{St}(d,d) \times \mathbb{R}^d \times \mathbb{R}^d$, where $\mathrm{St}(d,d)$ is the Stiefel manifold.

\paragraph{Stiefel Manifold:} The Stiefel manifold $\mathrm{St}(p,n) = \{ W \in \mathbb{R}^{n\times p} : W^\top W = I_p\}$ is embedded in $\mathbb{R}^{n\times p}$; therefore, tangent spaces can be identified with linear subspaces in the ambient space.
% \begin{equation}
%     T_W\mathrm{St}(p,n) = \{ Z \in \mathbb{R}^{n\times p} : Z^\top W + W^\top Z = 0_{p\times p}\}.
% \end{equation}
For a differentiable function $F:\mathbb{R}^{n\times p} \rightarrow \mathbb{R}$, the Riemannian gradient can be obtained by projecting the Euclidean gradient $\nabla^\mathcal{E}_WF$ onto the tangent space \citep{absil2009optimization}
\begin{equation}\label{eq:stiefel-rie-grad}
    \nabla_W F(W) = \nabla_W^\mathcal{E} F(W) - W\,\mathrm{sym}(W^\top \nabla_W^\mathcal{E} F(W)) \in T_W \mathrm{St}(p,n).
\end{equation}
The Riemannian score function and ELBO gradient can thus be obtained from their Euclidean counterparts via \eqref{eq:stiefel-rie-grad}; closed-form expressions for the latter are available in \citet{godichon2024natural}. For the retraction, we use the Cayley retraction of \citet{zhu2017riemannian}, together with its associated isometric vector transport (see their Lemma~3).

\paragraph{Algorithms:} We implement our Inverse-Free Riemannian Natural Gradient Descent (Inverse-Free RNGD) from Algorithm~\ref{alg:riemannian-if-ngd} with the inverse Fisher approximation based on a sliding window of $K=200$ $\psi$-vectors and $\epsilon=1000$ in \eqref{eq:inv-moving-approximation}; see \Cref{rem:limited-memory-approximation}. 
The sliding window update procedure is described in Appendix~\ref{appendix:vectorized-fisher}, where 10 new score vectors are generated at each current iterate to compute the new $\psi$-vectors, with the rest $K-10$ vectors transported from the previous tangent space. 
We compare the Inverse-Free RNGD with the \textit{Riemannian stochastic gradient descent} (RSGD) that computes the Riemannian gradient of the ELBO via \eqref{eq:stiefel-rie-grad} and updates the Stiefel components using the Cayley retraction. Further details are in Appendix~\ref{appx:Steifel-manifold-experiment}.

%\paragraph{Algorithms:} We compare three methods. \textit{Riemannian stochastic gradient descent} (RSGD) computes the Riemannian gradient of the ELBO via \eqref{eq:stiefel-rie-grad} and updates the Stiefel components using the Cayley retraction. The \textit{Riemannian natural gradient} (Riemannian NGD) method is Algorithm~\ref{alg:riemannian-if-ngd}, where the inverse Fisher approximation is based on a sliding window of $K=500$ score vectors; see Remark~\ref{rem:limited-memory-approximation}. 
%The sliding window update procedure is described in Appendix~\ref{appendix:vectorized-fisher}. The \textit{extrinsic natural gradient} method (extrinsic IFVB) applies the inversion-free approximation of \citet{godichon2024natural} in the ambient Euclidean space. To enforce the orthogonality constraint, the resulting update direction is projected onto the tangent space, and the manifold parameters are updated using the retraction. Further details are in Appendix~\ref{appx:Steifel-manifold-experiment}.

\begin{figure}
    \centering    
    \includegraphics[width=1\linewidth]{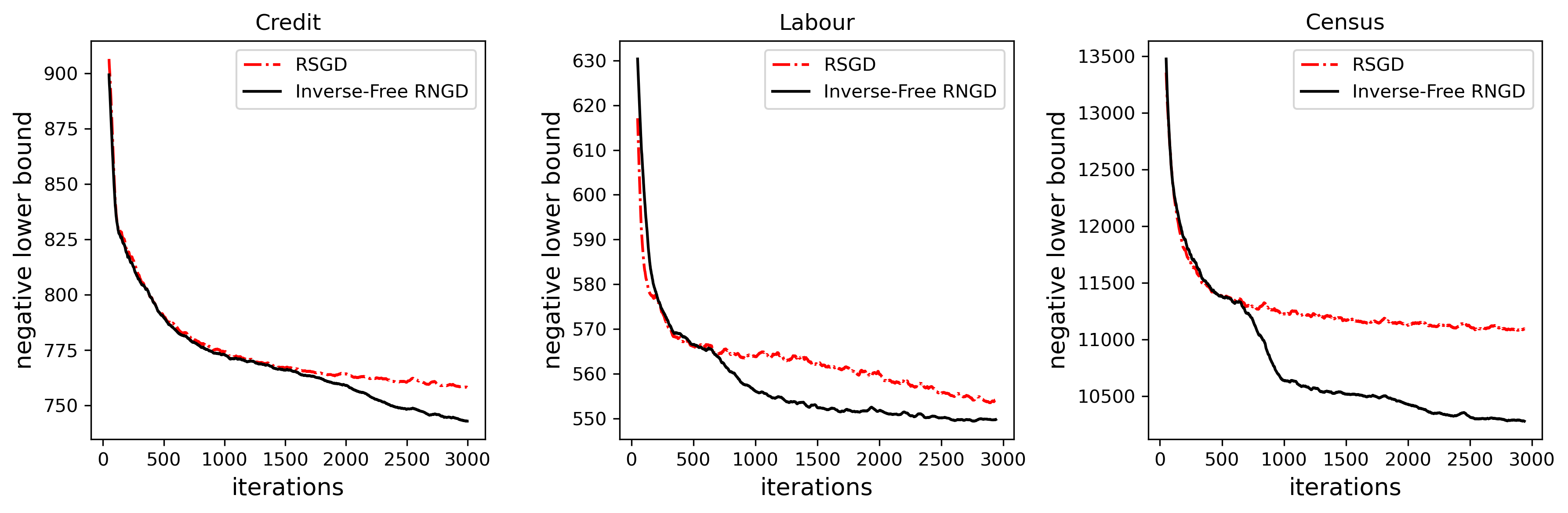}
    \caption{Negative lower bound plots of the VB methods}
    \label{fig:normalizing_flow_VB}
\end{figure}

\paragraph{Results:} We use several datasets from the UCI Machine Learning Repository: the German Credit, Woman Labor Force, and Census datasets. Figure~\ref{fig:normalizing_flow_VB} plots the negative ELBO over the course of training for each method. 
The Inverse-Free RNGD method attains the lowest negative lower bound across all datasets.

\section{Example: Fixed-Rank Manifold Natural Gradient}\label{sec:reduced-rank-regression}

In this section we examine reduced-rank multinomial logistic regression, where the matrix of regression coefficients is constrained to have fixed rank. Reduced-rank regression has a long history in statistics, with the classical formulation for multivariate linear models presented by \citet{anderson1951estimating} and \citet{izenman1975reduced}. Later, \citet{yee2003reduced} extended the idea to vector generalised linear models. The rank constraint can be implicitly enforced by working on the manifold of fixed-rank matrices \citep{meyer2011linear, mishra2014fixed}. Further details for this section are provided in Appendix~\ref{appx:Fixed-rank-manifold-experiment}.

\paragraph{Model \& Objective:} Let $\mathcal{D}=\{(x_i, y_i)\}^n_{i=1}$ denote a classification dataset with features $x_i \in \mathbb{R}^d$ and labels $y_i \in \{1,...,K\}$. The multinomial logistic regression model specifies
\begin{equation}
    P(y=j \mid x, B, \alpha) = \frac{\exp(\alpha_{j}+B_j^\top x)}{1+\sum^{K-1}_{k=1}\exp(\alpha_k+B^\top_kx)}, \qquad j=1,...,K-1,
\end{equation}
where $B \in \mathbb{R}^{d\times K-1}$ is the coefficient matrix and $\alpha \in \mathbb{R}^{K-1}$ the intercept. The fixed-rank constraint $\mathrm{rank}(B)=r<\min(d,K-1)$ reduces the number of free parameters from $d(K-1)$ to $r(d+K-1-r)$, and can yield a more parsimonious model when the class structure depends on a low-dimensional subspace of the features. The objective is the negative log-likelihood.

\paragraph{Baseline Manifold:} Let $\mathcal{M}_r = \{B \in \mathbb{R}^{d\times (K-1)} : \mathrm{rank}(B) = r\}$ denote the set of rank-$r$ matrices. This is an embedded submanifold of $\mathbb{R}^{d\times (K-1)}$; we briefly review its main properties, and refer to \citet{vandereycken2013low} for more information. Let $B \in \mathcal{M}_r$, with SVD $B=U\Sigma V^\top$ where $U \in \mathrm{St}(r,d)$, $V \in \mathrm{St}(r,K-1)$, and $\Sigma = \mathbb{R}^{r\times r}$ is diagonal with non-increasing entries. The tangent space $T_B \mathcal{M}_r$ can be identified with matrices $\xi \in \mathbb{R}^{d\times (K-1)}$ such that $(I_d - UU^\top)\xi(I_K-VV^\top)=0$.
% Similar to \Cref{sec:stiefel-experiment}, the Riemannian gradient of a function is obtained via orthogonal projection onto this subspace. 
For the retraction, we employ the metric projection $\mathcal{R}_B(\xi) = \mathrm{trunc}_r(B+\xi)$, which truncates the SVD of $B+\xi$ to its largest $r$ singular values. The vector transport is the orthogonal projection onto the new tangent space $\mathcal{T}_{B,B'}(\xi) = \mathrm{Proj}_{B'}(\xi)$.

\paragraph{Algorithms:} The choice of baseline algorithms is similar to \Cref{sec:stiefel-experiment}, although we do not employ the low-rank approximation to the inverse Fisher. \textit{Riemannian stochastic gradient descent} (RSGD) projects the (stochastic) Euclidean gradient of the log-likelihood onto $T_B\mathcal{M}_r$ and performs a retraction step. For each algorithm, the objective gradient is computed using a minibatch of $128$ observations. The \textit{Inverse-Free Riemannian Natural Gradient Descent} method (IF-RNGD) follows Algorithm 1, where the inverse Fisher approximation operates on the tangent space of $\mathcal{M}_r$. The \textit{Extrinsic Inverse-Free Natural Gradient Descent} method (Extrinsic IF-NGD) follows the Euclidean version of \Cref{alg:riemannian-if-ngd} in \citet[Example 3]{godichon2024natural}. Here, the low-rank condition is enforced by projecting the update direction onto the tangent space of $\mathcal{M}_r$, and performing the update using a retraction operation.

\begin{figure}[h]
    \centering
    % Row 1
    \begin{subfigure}{0.32\textwidth}
        \includegraphics[width=\linewidth]{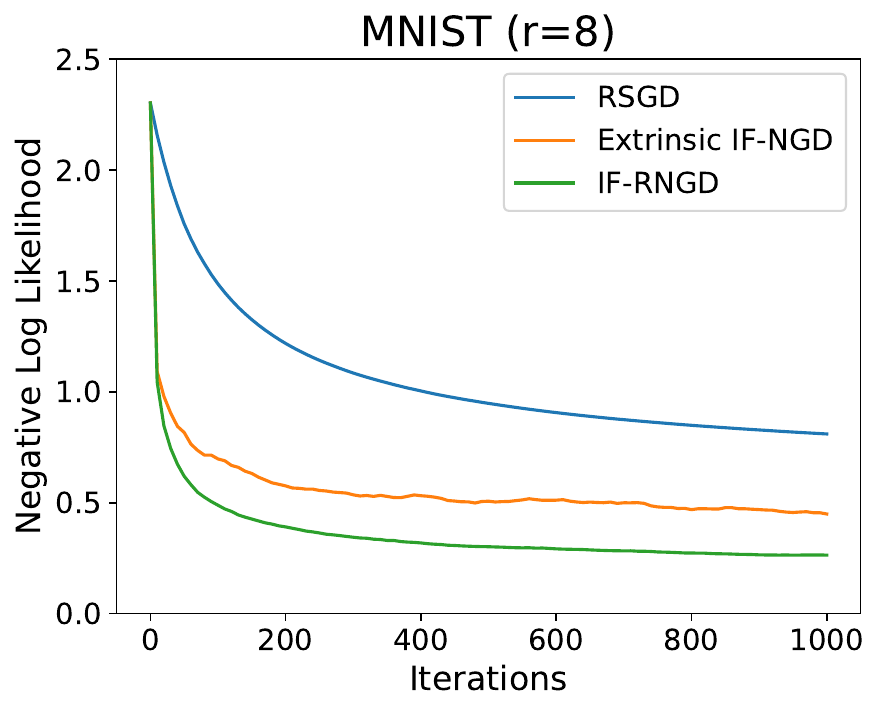}
    \end{subfigure}
    \hfill
    \begin{subfigure}{0.32\textwidth}
        \includegraphics[width=\linewidth]{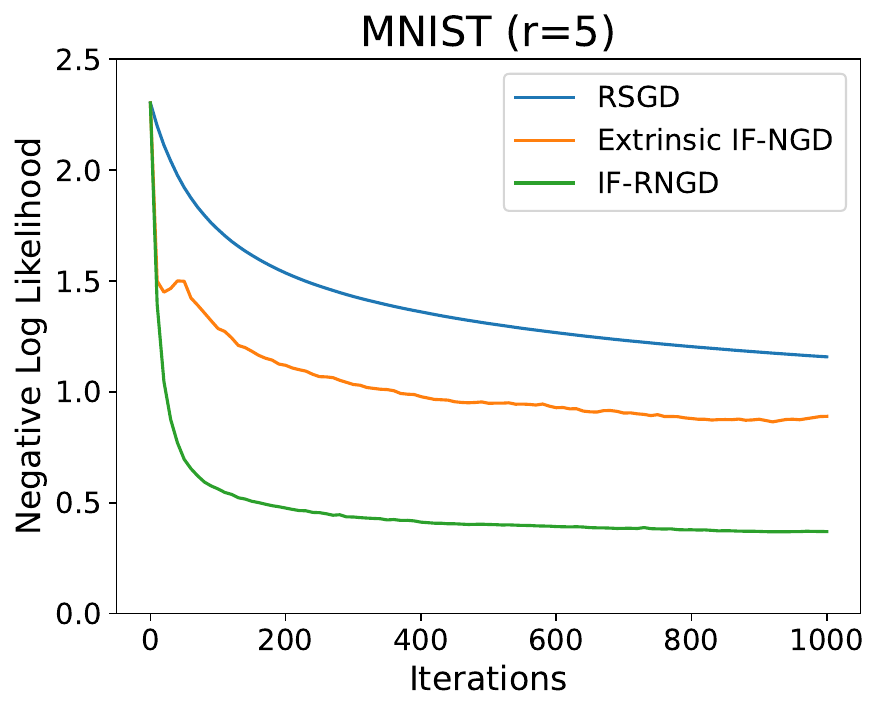}
    \end{subfigure}
    \hfill
    \begin{subfigure}{0.32\textwidth}
        \includegraphics[width=\linewidth]{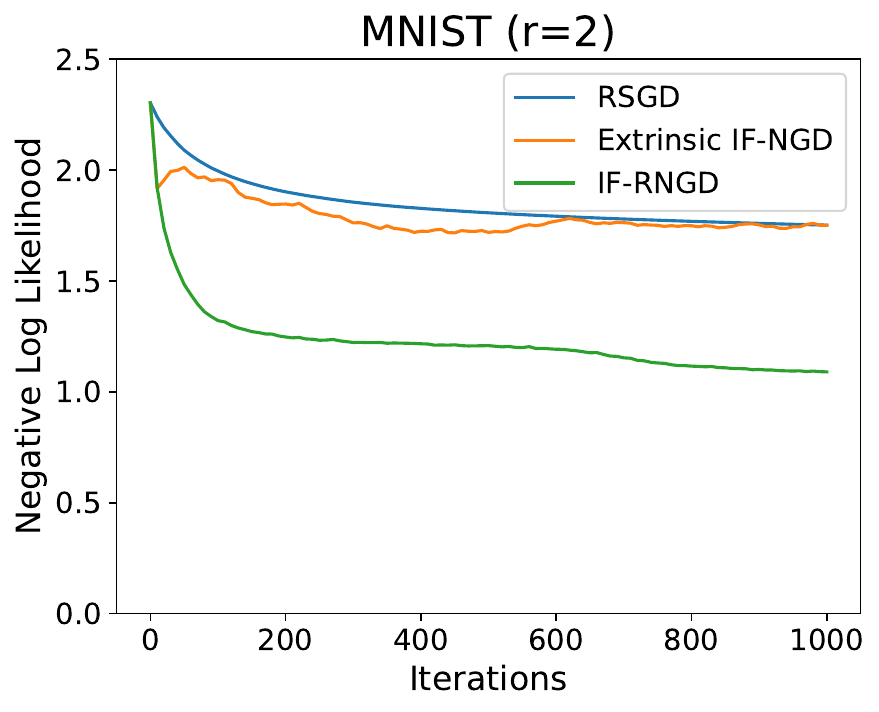}
    \end{subfigure}

    \vspace{0.5em} % vertical space between rows

    % Row 2
    \begin{subfigure}{0.32\textwidth}
        \includegraphics[width=\linewidth]{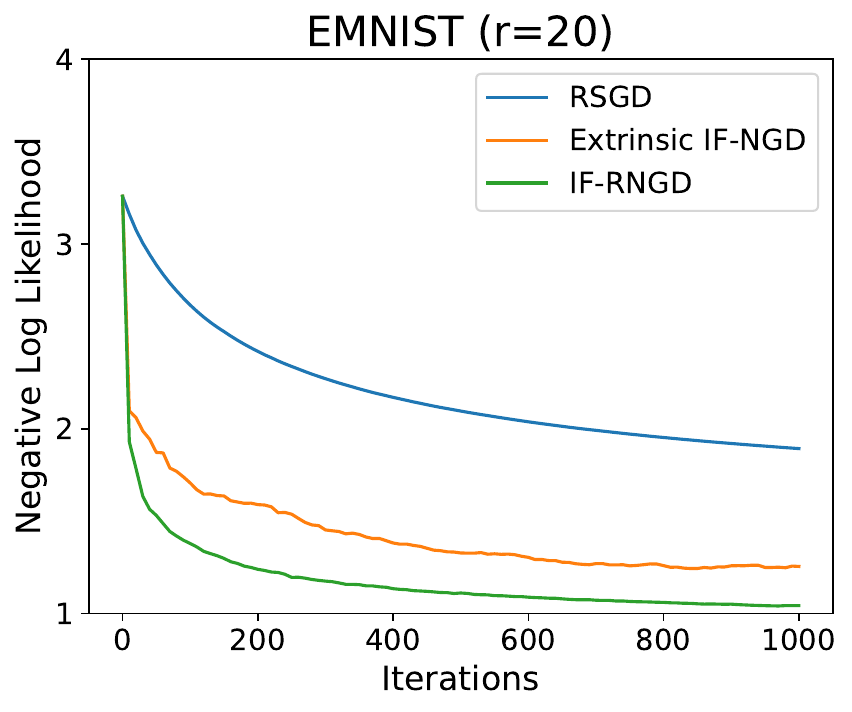}
    \end{subfigure}
    \hfill
    \begin{subfigure}{0.32\textwidth}
        \includegraphics[width=\linewidth]{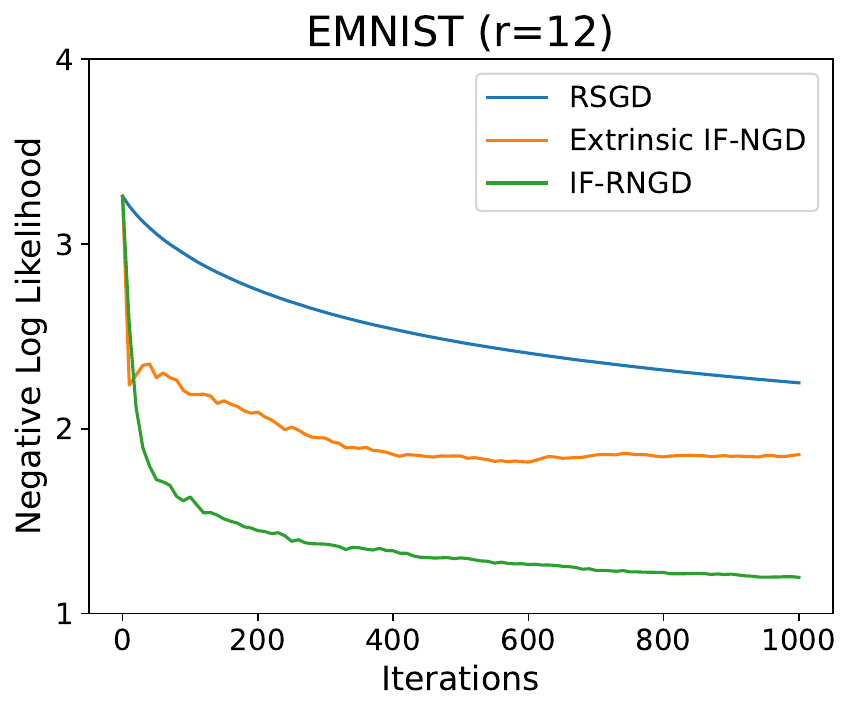}
    \end{subfigure}
    \hfill
    \begin{subfigure}{0.32\textwidth}
        \includegraphics[width=\linewidth]{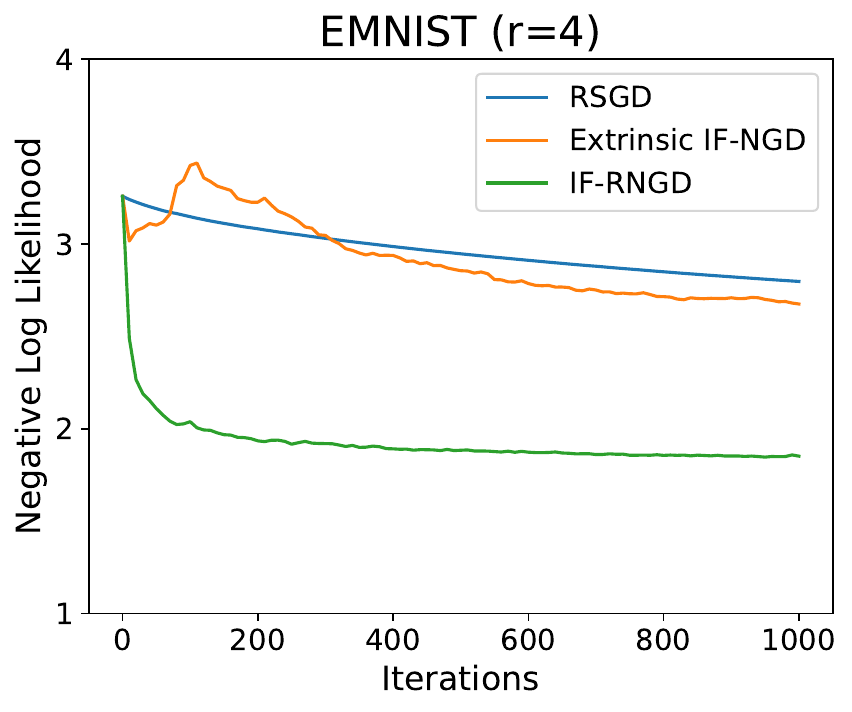}
    \end{subfigure}

    \caption{Negative log-likelihood over iteration for different rank constraints. Top row: MNIST dataset, $r \in \{8, 5, 2\}$. Bottom row: EMNIST dataset, $r \in \{20, 12, 4\}$.}
    \label{fig:nll-mnist}
\end{figure}
\paragraph{Results:} We apply our methods to the MNIST $(n=60,000, d=784,K=10)$ and EMNIST $(n=145,000, d=784, K=26)$ datasets, which contain a large number of handwritten digits and letters, respectively. We consider coefficient matrices with ranks $r \in \{8, 5, 2\}$ for MNIST, and $r \in \{20, 12, 4\}$ for EMNIST. \Cref{fig:nll-mnist} plots the (mean) negative log-likelihood over the training set for each dataset, method, and rank. The IF-RNGD method attains the best log-likelihood across all scenarios. The extrinsic IF-NGD method performs slightly worse when $r\approx K$, and the performance gap widens as $r$ becomes smaller; eventually, its performance becomes comparable to RSGD.

\section{Conclusion}\label{sec:Conclusion}

% % Tightened version of the conclusion should we require the space (this may be too terse).

This paper developed a stochastic natural gradient method for optimisation over probability distributions whose parameters lie on a Riemannian manifold.
This formulation accommodates several constrained parameter spaces which commonly arise in statistics, including SPD matrices, the Stiefel manifold, and the Grassmann manifold.
We proposed a Riemannian extension of an inversion-free approximate natural gradient method \citep{amari2000adaptive, godichon2024natural}, which streamlines the adaptive estimation of the inverse Fisher information matrix. We also proposed a limited-memory variant which reduces storage complexity from quadratic to sub-quadratic in the manifold dimension.
In contrast to previous work \citep{tran2021variational, hu2024riemannian}, our approach is purely intrinsic, and yields almost-sure convergence rates in the stochastic setting for both the parameter iterates and approximate Fisher information.
% Finally, we presented three examples which demonstrated the efficacy of this geometric approach relative to its Euclidean counterparts.
An interesting continuation of our research concerns its application beyond finite-dimensional models. In particular, the Wasserstein space of probability measures possesses a formal Riemannian structure \citep{villani2008optimal}, which provides notions of gradient, geodesics, and parallel transport. This suggests a possible route toward adapting our methodology to the non-parametric setting.

\newpage

\section*{Competing interests}
% JASA requires a disclosure/competing interests statement.
The authors report there are no competing interests to declare.
% Adjust as appropriate.

\section*{Acknowledgments}
% Acknowledge colleagues, referees, temporary affiliations, etc.

Draca's research was funded by an Australian Research Training Program scholarship, and a Data61 top-up scholarship. Draca also thanks Associate Professor John Ormerod for helpful feedback on an earlier draft of this work. Tran thanks Associate Professor Duy Nguyen for fruitful discussions during the early stages of this work. 

% --------------------------------------------------
% References
% --------------------------------------------------

\bibliographystyle{agsm}
\bibliography{ref}

% --------------------------------------------------
% Supplementary material (Appendices)
% --------------------------------------------------

\newpage
\section*{Supplementary Material}
\addcontentsline{toc}{section}{Supplementary material}
\renewcommand{\thesubsection}{\Alph{subsection}}

This supplementary material contains proofs for all theoretical results, complementary knowledge omitted from the main text, and further experimental details.
\Cref{appx:helpful-results} presents intermediate lemmas that facilitate the subsequent proofs.
\Cref{appx:eigenbound-fisherapprox} proves limiting bounds for the spectrum of the approximate Fisher information matrix used in the main proofs.
\Cref{appx:proof-global-convergence,appx:convergence-rate-proof,appx:fisher-consistency-proof} provide the proofs of the theoretical results presented in the main text.
% \Cref{appendix:kl-fisher-manifold,appendix:score and reparameterization gradient,appendix:vectorized-fisher} contains complementary material referred to in the main text.
\Cref{appendix:kl-fisher-manifold,appendix:score and reparameterization gradient,appendix:vectorized-fisher} contains complementary material referred to in the main text.
\Cref{appx:bw-nat-grad,appx:Steifel-manifold-experiment} offer an in-depth description of our experimental setup.

\subsection{Helpful Results} \label{appx:helpful-results}

The Robbins-Siegmund theorem is a basic tool for working with stochastic sequences, and will be used repeatedly in our proofs. Here, it is lightly paraphrased from \citet{duflo2013random}.

\begin{theorem}[Robbins-Siegmund Theorem] Suppose that $(V_n), (\beta_n), (\chi_n)$, and $(\eta_n)$ are four non-negative sequences adapted to the filtration $\mathbb{F}=(\mathcal{F}_n)$ and that:
\begin{equation}
    \mathbb{E}[V_{n+1}\vert \mathcal{F}_n] \leq V_n (1 + \beta_n) + \chi_n - \eta_n.
\end{equation}
Then, on $\left\{ \sum\beta_n < \infty \text{ and } \sum\chi_n < \infty\right\}$, almost surely $(V_n)$ converges to a finite random variable $V_\infty$ and the series $\sum\eta_n$ converges.
\end{theorem}

Throughout the following lemmas, we will assume that $\mathcal{M}$ is a Riemannian manifold, and $\mathcal{R}$ is a second-order retraction on $\mathcal{M}$. For $(x,v) \in T\mathcal{M}$ we let $\mathcal{T}_{v_x}[w] := D\mathcal{R}_x(v)[w]$ denote the transport corresponding to the differentiated retraction. To simplify our presentation, we will assume that $\mathcal{R}_x$ and $\mathcal{R}^{-1}_x$ are well-defined\footnote{This is guaranteed within a small neighborhood of $x$ using similar reasoning as for the exponential map.} wherever they appear.

\begin{lemma} \label{lemma:retraction-ineq}
% Let $\mathcal{M}$ be a Riemannian manifold and $\mathcal{R}$ a retraction. For any $x \in \mathcal{M}$ there exists a neighborhood $\mathcal{U}\ni x$ and constants $c_1, c_2, \delta>0$ such that $\forall v \in T_x \mathcal{M}$ with $\Vert v \Vert_x < \delta$
For any $x \in \mathcal{M}$ there exist $c_1, c_2, \delta>0$ such that $\forall v \in T_x\mathcal{M}$ with $\Vert v \Vert_x <\delta$
    \begin{equation}
        c_1 \Vert v \Vert_x \leq d(x, \mathcal{R}_x(v)) \leq c_2 \Vert v \Vert_x.
    \end{equation}
\end{lemma}
\begin{proof}
    See e.g., Proposition 7.1.3 in \citet{absil2009optimization}, or Lemma 3.3.3 in \citet{Huang2013optimizationAO}.
\end{proof}
The following result is paraphrased from Lemma 6 in \citet{tripuraneni2018averaging}. Here, the big-O notation bounds the norm of a hidden linear operator.

\begin{lemma} \label{lemma:tripuraneni-averaging}
    % Let $\mathcal{M}$ be a Riemannian manifold, $\mathcal{R}$ a second-order retraction, $\mathcal{T}_{v_x}[w] := D\mathcal{R}_x(v)[w]$, and $\Gamma^\mathcal{R}_{v_x}$ parallel transport along the curve $\gamma(t) = \mathcal{R}_x(tv)$ from $t=0$ to $t=1$.
    % Then,
        Let $\Gamma^\mathcal{R}_{v_x}$ denote parallel transport along $\gamma(t) = \mathcal{R}_x(tv)$ for $t \in [0,1]$, then
    \begin{equation}
        G_x(v) := [\Gamma^\mathcal{R}_{v_x}]^{-1}\circ \mathcal{T}_{v_x} = \mathrm{Id}_x + \frac{1}{2}K_x[v, v, \cdot] + O(\Vert v \Vert^3_{x}),
    \end{equation}
    where $K_x[v,v,\cdot]: =\tfrac{d^2}{dt^2}G_x(tv)\vert_{t=0}$ is a trilinear map $T_x\mathcal{M} \times T_x \mathcal{M} \times T_x\mathcal{M} \rightarrow T_x \mathcal{M}$.
\end{lemma}

The following related lemma bounds the operator norms of $\mathcal{T}$ and its inverse.

\begin{lemma} \label{lemma:vector-transport-bounds}
        % Let $\mathcal{M}$ be a Riemannian manifold, $\mathcal{R}$ a second-order retraction, and $\mathcal{T}_{v_x}[w] := DR_x(v)[w]$. If $\mathcal{C}\subseteq\mathcal{M}$ is compact $\exists\delta,c> 0$ such that $\forall (x,v) \in T\mathcal{C}$ with $\Vert v\Vert_x < \delta$
        If $\mathcal{X}\subseteq\mathcal{M}$ is compact then $\exists\delta,c> 0$ such that $\forall (x,v) \in T\mathcal{X}$ with $\Vert v\Vert_x < \delta$
        \begin{equation}
        (1 - c \Vert v \Vert^2_x)\Vert w \Vert_x \leq \Vert \mathcal{T}_{v_x}[w]\Vert_{R_x (v)} \leq (1 + c \Vert v \Vert_x^2) \Vert w \Vert_x, \ \ \ \ \forall w \in T_x \mathcal{M}.
    \end{equation}
\end{lemma}
\begin{proof}
    Let $G_x(v)$ be defined as in \Cref{lemma:tripuraneni-averaging}. 
    Because $\mathcal{R}$ is a second-order retraction, the first derivative of $G_x$ evaluated at $v=0$ vanishes\footnote{This is explained in more detail in the proof of Lemma 6 in \citet{tripuraneni2018averaging}.}. 
    We have the truncated Taylor expansion:
    \begin{equation}
        G_x(v)=\text{Id}_{x}+\frac{1}{2}D^2G_x(t_*v)[v,v,\cdot], \quad \quad t_*\in[0,1].
    \end{equation}
    Hence
    \begin{equation}
        \Vert \mathcal{T}_{v_x}[w] - \Gamma^{\mathcal{R}}_{v_x}[w]\Vert_{\mathcal{R}_x(v)} = \Vert G_x(v)[w]-w\Vert_x \leq \Vert D^2G_x(t_*v)\Vert_{\text{op}}\Vert v\Vert^2_x \Vert w\Vert_x.
    \end{equation}
    %By the smoothness of $\Gamma^\mathcal{R}$ and $\mathcal{R}$, and since $T\mathcal{C}$ is compact, let $c<\infty$ denote the supremum of $\Vert D^2G_x(u)\Vert_{\text{op}}$ over $(x,u) \in T\mathcal{C}$ with $\Vert u\Vert_x < \delta$. 
    By the smoothness of $\Gamma^\mathcal{R}$ and $\mathcal{R}$, let $c<\infty$ denote the supremum of $\Vert D^2G_x(u)\Vert_{\text{op}}$ over the compact set $\{(x,u) \in T\mathcal{X} \mid \Vert u\Vert_x \leq \delta\}$.
    Since $\Gamma^\mathcal{R}$ is isometric, we can evaluate the norm as
    \begin{align}
        \Vert \mathcal{T}_{v_x}[w]\Vert_{\mathcal{R}_x(v)} = \Vert G_x(v)[w]\Vert_x = \Vert (G_x(v)[w]-w)+w\Vert_x.
    \end{align}
    %$\Vert \mathcal{T}_{v_x}[w]\Vert_{\mathcal{R}_x(v)} = \Vert G_x(v)[w]-w+w\Vert_x$. 
    The result follows via the ordinary and reversed triangle inequalities.
\end{proof}

\begin{remark} \label{rem:no-compactness condition}
     If $\mathcal{M}$ has bounded sectional curvature, non-zero injectivity radius, and $\mathcal{R}=\exp$,  \Cref{lemma:vector-transport-bounds} holds on $\mathcal{M}$ without compactness \citep[Proposition A.3]{criscitiello2023accelerated}.
\end{remark}

% \DD{ It is known that when $\mathcal{R}=\exp$, we have $D^2G_x(0)[v,v] = K_x(v,v,\cdot)=\tfrac{1}{3}R_x(v, \cdot)v$, where $R_x$ is the Riemannian curvature tensor. In this case I suspect it is possible to bound $\Vert D^2G_x(u)\Vert_{\text{op}}$ globally w.r.t $x \in \mathcal{M}$ for $u \in T_x \mathcal{M}$ sufficiently small when $\mathcal{M}$ has finite sectional curvature, and the injectivity radius has a lower bound. This should allow us to drop the compactness assumption on $\mathcal{M}$, since this is the only place where it is used.}

The next lemma shows that $\mathcal{T}$ agrees with the differentiated exponential map up to second-order terms. The assumption that $\mathcal{X} \subseteq \mathcal{M}$ is a normal neighborhood\footnote{This terminology is standard in Riemannian geometry; see e.g. \citet{do1992riemannian}.} of $x \in \mathcal{X}$ ensures that there is a unique minimising geodesic from $x$ to each $y \in \mathcal{X}$.

\begin{lemma} \label{lemma:DR-Dexp-equality}
     % Let $\mathcal{M}$ be a Riemannian manifold, $\mathcal{R}$ a second-order retraction, and $\mathcal{T}_{x,y}[w] := D\mathcal{R}_x(\mathcal{R}^{-1}_x(y))[w]$. Let $\mathcal{X} \subset \mathcal{M}$ with $x \in \mathcal{X}$ so $\exp_x^{-1}$ and $\mathcal{R}^{-1}_x$ are well-defined on $\mathcal{X}$. Then
     Let $\mathcal{X}\subseteq \mathcal{M}$ be a compact normal neighborhood of $x \in \mathcal{M}$, then
     \begin{equation}
         \mathcal{T}^{-1}_{x,y} \circ D \exp_x(\exp_x^{-1}(y)) = \mathrm{Id}_{x} + O(d(x,y)^2), \qquad \forall y \in \mathcal{X}.
     \end{equation}
\end{lemma}
\begin{proof} 
    The proof relies on the notion of a second covariant derivative of a smooth function $F:\mathcal{M} \rightarrow \mathbb{R}^d$; see e.g. \citet[Section 5.6]{absil2009optimization}. 
    For vector fields $X, Y$ on $\mathcal{M}$:
    \begin{equation} \label{eq:second-cov-derivative}
        (\nabla D F)[X,Y]:=X(DF[Y])-DF[\nabla_XY].
    \end{equation}
    Here $\nabla_XY$ is the Levi-Civita covariant derivative on $\mathcal{M}$, and $X(DF[Y])$ means apply $X$ to the components of the $\mathbb{R}^d$-valued function $DF[Y]$. 
    Let $v = \exp_x^{-1}(y)$ and define:
    \begin{align}
        G(v) & := D(\mathcal{R}_x^{-1} \circ \exp_x)(v) = D\mathcal{R}_x^{-1}(\exp_x(v)) \circ D\exp_x(v) \\
        & = [D\mathcal{R}_x(\mathcal{R}_x^{-1} \circ \exp_x(v))]^{-1} \circ D\exp_x(v) = \mathcal{T}^{-1}_{x,y} \circ D\exp_x(\exp_x^{-1}(y)).
    \end{align}
    Clearly $G(0) = \text{Id}$. 
    Next, we show $DG(0)=0$. The main result then follows by a Taylor series argument.
    Differentiating $G$ along $v$, we have:
    \begin{align}
        \frac{d}{dt}G(tv)[w]\vert_{t=0} & = (\nabla D\mathcal{R}^{-1}_x)[v,D\exp_x(0)[w]] + D\mathcal{R}^{-1}_x(x)[\nabla_v(D\exp_x[w])] \\
        & = (\nabla D \mathcal{R}_x^{-1})[v,w]+\nabla_v(D\exp_x[w]).
    \end{align}
    Here $\nabla_v(D\exp_x[w])$ is the ordinary covariant derivative of $y\rightarrow D\exp_x(\exp^{-1}_x(y))[w]$ viewed as a vector field. Consider $w = D\mathcal{R}^{-1}_x(\mathcal{R}_x(tv))[D\mathcal{R}_x(tv)[w]]$; differentiating both sides of this equation at $t=0$ yields that
    \begin{align}
        0 & = (\nabla D\mathcal{R}^{-1}_x)[v,D\mathcal{R}_x(0)[w]]+D\mathcal{R}_x^{-1}(x)[\nabla_v(D\mathcal{R}_x[w])] \\
        & = (\nabla D\mathcal{R}^{-1}_x)[v,w]+\nabla_v(D\mathcal{R}_x[w]).
    \end{align}
    Combining these equations, it suffices to show that $\nabla_v (D \mathcal{R}_x[w]) = \nabla_v(D\exp_x[w])=0$ for $w,v \in T_x\mathcal{M}$. Evidently, $(v,w) \rightarrow \nabla_v (D\mathcal{R}_x[w])$ is bilinear, and furthermore, we have
    \begin{equation}
        \frac{D^2}{dt^2}\mathcal{R}_x(tv)\vert_{t=0} := \nabla_v (D\mathcal{R}_x[v]) = 0. \label{eq:R_x_quandratic}
    \end{equation}    
    % which is often taken as the defining property of a second order retraction. That $(v,w) \rightarrow \nabla_v (D\mathcal{R}_x[w])$ is symmetric can be shown by relating this map to the \textit{second fundamental form} of $\mathcal{R}_x$, which generalises \eqref{eq:second-cov-derivative} to smooth maps between Riemannian manifolds. 
    This ``zero acceleration'' condition is often taken as the defining property of a second order retraction \citep{absil2009optimization}. 
    % That $(v,w) \rightarrow \nabla_v (D\mathcal{R}_x[w])$ is symmetric follows from the torsion-free property of the Levi-Civita connection\footnote{Specifically, we can relate this map to the \textit{second fundamental form} of $\mathcal{R}_x$, which generalises \eqref{eq:second-cov-derivative} to smooth maps between Riemannian manifolds. The second fundamental form is known to be symmetric; see e.g. the preliminary material in \citet{akyol2019conformal}.}. 
    The symmetry of $(v,w) \rightarrow \nabla_v (D\mathcal{R}_x[w])$ can be verified e.g. using local coordinates, but the derivation is tedious. A concise explanation is possible\footnote{For example, we can relate the map $(v,w) \rightarrow \nabla_v (D\mathcal{R}_x[w])$ to the \textit{second fundamental form} of $\mathcal{R}_x$, which generalises \eqref{eq:second-cov-derivative} to smooth maps between Riemannian manifolds; see e.g. the preliminary material in \citet{akyol2019conformal}. The symmetry of this map follows from the symmetry of the second fundamental form.}, but requires some more sophisticated geometric machinery.
    Provided $(v,w) \rightarrow \nabla_v(D\mathcal{R}_x[w])$ is a symmetric bilinear form, then by \eqref{eq:R_x_quandratic} and polarization it vanishes everywhere. The same argument applies to the exponential map, yielding $DG(0)=0$. The Taylor expansion is thus $G(v) = \text{Id}_x + O( \Vert v \Vert_x^2)$, where $\Vert v \Vert_x = \Vert \exp_x^{-1}(y)\Vert_x = d(x,y)$, which concludes the proof.
    % In particular, the second fundamental form is known to be symmetric \citep{akyol2019conformal}.
    % Because the map is a symmetric bilinear form, the fact that its quadratic form in \eqref{eq:R_x_quandratic} vanishes implies by polarization that it vanishes identically for all $v,w$. The same argument applies to the exponential map.
    % Therefore, $DG(0) = 0$.
    % By the Taylor series expansion, $G(v) = \mathrm{Id}_x + DG(0) + O( \Vert v \Vert_x^2)$.
    % Noting that $\Vert v \Vert_x = d(x,y)$ concludes the proof.
\end{proof}

An immediate consequence of the previous lemma is that $\mathcal{T}$ agrees with parallel transport up to second order terms. In the following lemmas, we define $\mathcal{T}_{x,y}[w] := D\mathcal{R}_x(\mathcal{R}^{-1}_x(y))[w]$, and let $\Gamma_{x,y}$ denote the parallel transport along the geodesic connecting $x,y \in \mathcal{M}$.

% In the following lemmas, we assume that $\mathcal{X} \subseteq \mathcal{M}$ is a uniquely geodesic subset, and $\Gamma_{x,y}$ denotes the parallel transport operation along the geodesic connecting $x,y \in \mathcal{X}$.

\begin{lemma} \label{lemma:transport-pt}
    % Let $\mathcal{M}$ be a Riemannian manifold, $\mathcal{R}$ a second-order retraction, and $\mathcal{T}_{x,y}[w]=D\mathcal{R}_x(\mathcal{R}^{-1}_x(y))[w]$. Let $\Gamma_{x,y}$ denote the parallel transport along the geodesic connecting $x,y\in \mathcal{M}$. Let $\mathcal{X}\subset \mathcal{M}$ with $x \in \mathcal{X}$ so that $\mathcal{R}^{-1}_x, \exp^{-1}_x$ are well-defined on $\mathcal{X}$. Then
     Let $\mathcal{X}\subseteq \mathcal{M}$ be a compact normal neighborhood of $x \in \mathcal{M}$, then
    \begin{equation}
        \mathcal{T}^{-1}_{x,y} \circ \Gamma_{x,y} = \mathrm{Id}_{x} + O(d(x,y)^2), \qquad \forall y \in \mathcal{X}.
    \end{equation}
\end{lemma}
\begin{proof} 
    We have
    \begin{equation}
        \mathcal{T}^{-1}_{x,y} \circ \Gamma_{x,y} = \underbrace{\mathcal{T}^{-1}_{x,y} \circ D\exp_x(\exp_x^{-1}(y))}_{=(a)} \circ \underbrace{[D\exp_x(\exp_x^{-1}(y))]^{-1} \circ \Gamma_{x,y}}_{=(b)} .
    \end{equation}
    From \Cref{lemma:DR-Dexp-equality} $(a) = \text{Id}_x + O(d(x,y)^2)$, and by \Cref{lemma:tripuraneni-averaging} the same is true for $(b)$.
\end{proof}

The following result has been used in various works; here it is paraphrased from \citet[Lemma 2]{han2023riemannian}. The idea is attributed to \citet{karcher1977riemannian}. 

\begin{lemma} \label{lemma:karcher-pt}
    % Let $\mathcal{M}$ be a Riemannian manifold, and $\mathcal{X}\subseteq\mathcal{M}$ a compact subset with unique geodesics. There exists a curvature and diameter-dependent constant $C$, such that
    Let $\mathcal{X}\subseteq\mathcal{M}$ be a compact set where each pair $x,y \in \mathcal{X}$ is connected by a unique geodesic. There exists a curvature and diameter-dependent constant $C>0$ such that
    \begin{equation}
        \Vert \Gamma_{y,z} \Gamma_{x,y} u - \Gamma_{x,z} u \Vert_z \leq C d(x,y) d(y,z) \Vert u\Vert_x, \ \ \ \ \forall x,y,z \in \mathcal{X}, \forall u \in T_x\mathcal{M}.
    \end{equation}
\end{lemma}

The following will be crucial for proving the consistency of our Fisher approximation.

\begin{corollary} \label{lemma:triangle-transport}
    % Let $\mathcal{M}$ be a Riemannian manifold, $\mathcal{R}$ a second-order retraction, and $\mathcal{T}_{x,y}[w]=D\mathcal{R}_x(\mathcal{R}^{-1}_x(y))[w]$. Let $\Gamma_{x,y}$ denote the parallel transport along the geodesic connecting $x,y\in \mathcal{M}$. Then for each $x \in \mathcal{M}$ there exists some open neighborhood $\mathcal{X}$ such that
    For each $x \in \mathcal{M}$ there exists some neighborhood $\mathcal{X}\subseteq \mathcal{M}$ such that
    \begin{equation}
        \Gamma_{z,x} \circ \mathcal{T}^{-1}_{z,y} \circ \Gamma_{x,y} = \text{Id}_x + O(d(x,y)d(y,z) + d(y,z)^2), \qquad \forall y,z \in \mathcal{X}.
    \end{equation}
\end{corollary}
\begin{proof}
    We can take $\mathcal{X}$ to be a compact totally normal\footnote{Loosely, $\mathcal{X}$ is a normal neighborhood for each point in $\mathcal{X}$. Such a set is guaranteed to exist around any point; see e.g. section 3 in \citet{do1992riemannian}.} neighborhood of $x \in \mathcal{M}$, which ensures the assumptions of \Cref{lemma:karcher-pt}. From \Cref{lemma:transport-pt} we have $\Gamma_{z,x} \circ \mathcal{T}^{-1}_{z,y} \circ \Gamma_{x,y} = \Gamma_{z,x} \circ \Gamma_{y,z} \circ \Gamma_{x,y} + O(d(y,z)^2)$. Then from \Cref{lemma:karcher-pt}, we have $\Gamma_{z,x} \circ \Gamma_{y,z} \circ \Gamma_{x,y} = \text{Id}_x + O(d(x,y)d(y,z))$. Combining these yields the result.
\end{proof}

\newpage
\subsection{Eigenvalue Bounds of Fisher Approximation} \label{appx:eigenbound-fisherapprox}

In the following lemma, the upper bound can be shown using weaker assumptions; see the proof of \Cref{thm:global-convergence}. It is included here for completeness.

\begin{lemma} \label{lemma:eigenvalue-lower-bound}
    % Suppose that assumption \ref{assumption:eigenvalues}, and conditions (3) and (4) of Theorem \ref{thm:global-convergence} hold.
    Suppose that the conditions of \Cref{thm:global-convergence} hold, the iterates generated by \Cref{alg:riemannian-if-ngd} converge $\theta^{(k)} \rightarrow \theta^*$ almost-surely, and that $I_F(\theta)$ is continuous in a neighborhood of $\theta^*$ with $\lambda_{\min}(I_F(\theta^*)) > 0$. It follows that:
    \begin{equation}
        0 < \liminf_{s\rightarrow \infty} \lambda_{min}(\mathbf{H}_s), \qquad \limsup_{s\rightarrow \infty} \lambda_{\max}(\mathbf{H}_s) < \infty, \qquad \text{a.s.}
    \end{equation}
\end{lemma}

\begin{proof}
   Define the composite transport $\mathcal{T}_{[k,s]} : T_{\theta^{(k)}}\mathcal{M} \rightarrow T_{\theta^{(s)}}\mathcal{M}$ such that for $k<s$
    \begin{equation}
        \mathcal{T}_{[k,s]} := \mathcal{T}_{s-1,s} \ \circ \ ... \ \circ \ \mathcal{T}_{k,k+1}, \qquad \mathcal{T}_{[s,k]} = \mathcal{T}_{k+1,k} \ \circ \ ... \ \circ \ \mathcal{T}_{s,s-1},
    \end{equation}
    and $\mathcal{T}_{[s,s]} = \text{Id}_{\theta^{(s)}}$. Then, we use the following notation for the transported score vectors
    \begin{equation} \label{eq:transported-score-vectors}
        \phi_{k,s} = \mathcal{T}_{[s,k]}^{-1}[\nabla_\theta \log{q_{\theta^{(k)}}(\bar{y}_{k+1})}], \qquad \tilde{\phi}_{k,s} = G_{\theta^{(s)}} \phi_{k,s},
    \end{equation}
    where $G_{\theta^{(s)}}$ is the matrix representation of the baseline metric at $\theta^{(s)}$.
    Consequently, 
    \begin{equation}
        \mathbf{H}_{s+1} = \frac{1}{s+1}\mathcal{T}^{-1}_{[s,0]} H_0 \mathcal{T}^{-*}_{[s,0]}+\frac{1}{s+1}\sum^{s}_{k=0} \phi_{k,s} \tilde{\phi}_{k,s}^\intercal.
    \end{equation}
    Denote $\Phi_{k,s} := \mathcal{T}^{-1}_{[s,k]} \circ I_F(\theta^{(k)}) \circ \mathcal{T}^{-*}_{[s,k]}$, then writing $\mathbf{H}_{s+1}=R_{s+1}+M_{s+1}$ where
    % \begin{equation}
    %     \Phi_{k,s} := \mathcal{T}^{-1}_{[s,k]} \circ I_F(\theta^{(k)}) \circ \mathcal{T}^{-*}_{[s,k]}.
    % \end{equation}
    % We split $\mathbf{H}_{s+1}$ into the sum
    \begin{equation}
        R_{s+1} = \frac{1}{s+1}\sum^s_{k=0}\Phi_{k,s}, \qquad M_{s+1} = \frac{1}{s+1} \big\{\mathcal{T}^{-1}_{[s,0]} H_0 \mathcal{T}^{-*}_{[s,0]}+\sum^{s}_{k=0} \underbrace{\left[\phi_{k,s} \tilde{\phi}_{k,s}^\intercal - \Phi_{k,s}\right]}_{:=\Theta_{k,s}}\big\}.
    \end{equation}
    % Our goal will be to show that $M_s \rightarrow 0$, and that $R_{s+1}$ satisfies the eigenvalue bounds a.s.
    
    \subsubsection{Claim: $\sum\Vert v_{s+1}\Vert^2_{\theta^{(s)}}$ is finite almost surely.} \label{lemma:eigenvalue-lower-claim1}
    Define the filtration $\mathcal{F}_s = \sigma(y_{k,i}, \bar{y}_{k'} : k \leq s, k' \leq s+1, i \leq B)$. 
    Taking $v_{s+1}$ as in \eqref{eq:approx-ngd-update}, 
    % \begin{align}
    %     \mathbb{E}[\Vert v_{s+1}\Vert^2_{\theta^{(s)}} \vert \mathcal{F}_s] & \leq \tau_{s+1}^2 \Vert \mathbf{H}^{-1}_{s+1}\Vert^2_{\text{op}} \times \mathbb{E}[\Vert \widehat{\nabla \mathcal{L}}(\theta^{(s)})\Vert^2_{\theta^{(s)}} \vert \mathcal{F}_s] \\
    %     & \leq\tau_{s+1}^2 \Vert \mathbf{H}^{-1}_{s+1}\Vert^2_{\text{op}} \times \tfrac{1}{B}\mathbb{E}[\Vert g(Y,\theta^{(s)}) \Vert^2_{\theta^{(s)}}\vert \mathcal{F}_s] \\
    %     & \leq O(s^{-(1+\delta)})\times  \tfrac{1}{B}[C_0 + C_1\underbrace{(\mathcal{L}(\theta^{(s)}) - \mathcal{L}(\theta^*))}_{:=W_s}]
    % \end{align}
    \begin{align}
        \mathbb{E}[\Vert v_{s+1}\Vert^2_{\theta^{(s)}} \vert \mathcal{F}_s] & \leq \tau_{s+1}^2 \Vert \mathbf{H}^{-1}_{s+1}\Vert^2_{\text{op}} \times \mathbb{E}[\Vert \widehat{\nabla \mathcal{L}}(\theta^{(s)})\Vert^2_{\theta^{(s)}} \vert \mathcal{F}_s] \\
        & \leq \tau_{s+1}^2 \Vert \mathbf{H}^{-1}_{s+1}\Vert^2_{\text{op}} \times \mathbb{E}[\Vert g(Y,\theta^{(s)}) \Vert^2_{\theta^{(s)}}\vert \mathcal{F}_s] ,
    \end{align}
    where the second line follows by convexity.
    Let $\delta = 2\alpha - 2\beta -1>0$. 
    Then, using \Cref{assumption:eigenvalues} and condition (3) from \Cref{thm:global-convergence}, we have
    \begin{align}
        \mathbb{E}[\Vert v_{s+1}\Vert^2_{\theta^{(s)}} \vert \mathcal{F}_s] & \leq O(s^{-(1+\delta)})\times  [C_0 + C_1\underbrace{(\mathcal{L}(\theta^{(s)}) - \mathcal{L}(\theta^*))}_{:=W_s}] \\
        & = \Vert v_s\Vert^2_{\theta^{(s-1)}}+O(s^{-(1+\delta)})-\Vert v_s\Vert^2_{\theta^{(s-1)}}.
    \end{align}
    Since $\theta^{(s)}\rightarrow \theta^*$ we have $W_s\rightarrow 0$, and the claim follows by the Robbins-Siegmund theorem.
    % where $\delta = 2\alpha - 2\beta -1>0$. The second inequality follows from convexity; the last inequality employs assumption \ref{assumption:g-2nd-moment}, $\tau_s=O(s^{-\alpha})$, and assumption \ref{assumption:g-2nd-moment} for the expectation. Since $\theta^{(k)}\rightarrow \theta^*$ we have $W_s\rightarrow 0$, and the claim follows by e.g., the Robbins-Siegmund theorem.
    
    \subsubsection{Claim: $\Vert M_s\Vert^2_{\theta^{(s-1)}}\rightarrow 0$ almost surely.} \label{lemma:eigenvalue-lower-claim2}
    For linear maps $A,B: T_\theta\mathcal{M}\rightarrow T_\theta\mathcal{M}$, let $\langle A, B\rangle_\theta$ be the Hilbert-Schmidt inner product. Let
    \begin{equation}
        \tilde{\mathcal{F}}_s = \sigma (y_{k,i}, \bar{y}_{k'} : k \leq s, k' \leq s, i \leq B).
    \end{equation}
    Note that $\tilde{\mathcal{F}}_k$ slightly differs from $\mathcal{F}_k$ by excluding $s' = k+1$. Consider the recursion:
    \begin{equation} \label{eqn:thm4-ms1}
        M_{s+1} = \frac{s}{s+1} \mathcal{T}^{-1}_{s,s-1} \circ M_s \circ \mathcal{T}^{-*}_{s,s-1} + \frac{1}{s+1}\Theta_{s,s}.
    \end{equation}
    From the definition of $\Theta_{s,s}$, we have $\mathbb{E}[\Theta_{s,s} \vert \tilde{\mathcal{F}}_s] = \mathbb{E}[\phi_{s,s} \tilde{\phi}_{s,s}^\intercal - I_F(\theta^{(s)}) \vert \tilde{\mathcal{F}_s}] = 0$. Hence
    \begin{equation}
        \E[ \Vert M_{s+1}\Vert^2_{\theta^{(s)}} \vert \tilde{\mathcal{F}}_s] = \frac{s^2}{(s+1)^2} \Vert \mathcal{T}^{-1}_{s,s-1} \circ M_s \circ \mathcal{T}^{-*}_{s,s-1}\Vert^2_{\theta^{(s)}} +\frac{1}{(s+1)^2} \mathbb{E}[ \Vert \Theta_{s,s} \Vert^2_{\theta^{(s)}} \vert \tilde{\mathcal{F}_s}].
    \end{equation}
    For the last term, by convexity and Jensen's inequality
    \begin{equation}
        \mathbb{E}[\Vert \Theta_{s,s} \Vert^2_{\theta^{(s)}} \vert \tilde{\mathcal{F}_s}]  \leq 2 \mathbb{E}[\Vert \phi_{s,s} \tilde{\phi}_{s,s}\Vert^2_{\theta^{(s)}}\vert \tilde{\mathcal{F}}_s] + 2 \Vert I_F(\theta^{(s)})\Vert^2_{\theta^{(s)}} \leq 4 \mathbb{E}[\Vert \phi_{s,s} \tilde{\phi}_{s,s}\Vert^2_{\theta^{(s)}}\vert \tilde{\mathcal{F}}_s].
    \end{equation}
    Finally, since $\phi_{s,s}\tilde{\phi}_{s,s}^\intercal$ is self-adjoint, and applying condition (4) in \Cref{thm:global-convergence}
    \begin{equation}
        4 \mathbb{E}[\Vert \phi_{s,s} \tilde{\phi}_{s,s}\Vert^2_{\theta^{(s)}}\vert \tilde{\mathcal{F}}_s]=4 \mathbb{E}[\Vert \phi_{s,s} \tilde{\phi}_{s,s}\Vert^2_{\text{op}}\vert \tilde{\mathcal{F}}_s] = 4 \mathbb{E}[\Vert \phi_{s,s} \Vert^4_{\theta^{(s)}}\vert \tilde{\mathcal{F}}_s]\leq 4 (C_0' + C_1' W_s^2).
    \end{equation}
    % \begin{align}
    %     \mathbb{E}[\Vert \Theta_{s,s} \Vert^2_{\theta^{(s)}} \vert \tilde{\mathcal{F}_s}] & \leq 2 \mathbb{E}[\Vert \phi_{s,s} \tilde{\phi}_{s,s}\Vert^2_{\theta^{(s)}}\vert \tilde{\mathcal{F}}_s] + 2 \Vert I_F(\theta^{(s)})\Vert^2_{\theta^{(s)}} \\
    %     & \leq 4 \mathbb{E}[\Vert \phi_{s,s} \tilde{\phi}_{s,s}\Vert^2_{\theta^{(s)}}\vert \tilde{\mathcal{F}}_s] = 4 \mathbb{E}[\Vert \phi_{s,s} \tilde{\phi}_{s,s}\Vert^2_{\text{op}}\vert \tilde{\mathcal{F}}_s] = 4 \mathbb{E}[\Vert \phi_{s,s} \Vert^4_{\theta^{(s)}}\vert \tilde{\mathcal{F}}_s] \\ 
    %     & = 4 \mathbb{E}[\Vert \nabla \log{q_{\theta^{(s)}}(\bar{y}_{s+1})} \Vert^4_{\theta^{(s)}}\vert \tilde{\mathcal{F}}_s] \leq 4 (C_0' + C_1' (\mathcal{L}(\theta^{(s)}) - \mathcal{L}(\theta^*))^2),
    % \end{align}
    For the first term in \eqref{eqn:thm4-ms1}, using standard properties of the HS norm, one can show that
    \begin{equation}
        \Vert \mathcal{T}^{-1}_{s,s-1} \circ M_s \circ \mathcal{T}^{-*}_{s,s-1}\Vert_{\theta^{(s)}} \leq \Vert \mathcal{T}^{-1}_{s,s-1}\Vert^2_{\text{op}} \Vert M_s\Vert_{\theta^{(s-1)}}.
    \end{equation}
   Since the iterates are eventually in a compact neighborhood of $\theta^*$, from \Cref{lemma:vector-transport-bounds} we have $\Vert\mathcal{T}^{-1}_{s,s-1}\Vert_{\text{op}} = 1 + O(\Vert v_{s}\Vert^2_{\theta^{(s-1)}})$ almost-surely.
    Hence,
    % \begin{equation}
    %     \Vert\mathcal{T}^{-1}_{s,s-1}\Vert_{\text{op}} = 1 + O(\Vert v_{s}\Vert^2_{\theta^{(s-1)}}), \qquad \text{a.s.}
    % \end{equation}
    % Then bringing everything together
    \begin{equation} \label{eq:ms1-rs-inequality}
        \E[ \Vert M_{s+1}\Vert^2_{\theta^{(s)}} \vert \tilde{\mathcal{F}}_s] \leq \frac{s^2}{(s+1)^2} \cdot \left(1 + O(\Vert v_{s}\Vert^2_{\theta^{(s-1)}})\right) \cdot \Vert M_s\Vert^2_{\theta^{(s-1)}} + \frac{4(C_0'+C_1'W_s^2)}{(s+1)^2}.
    \end{equation}
    To show that $\Vert M_{s+1}\Vert^2_{\theta^{(s)}} \rightarrow 0$ almost surely, we start by defining
    \begin{equation}
        V_{s+1} = \frac{s+1}{\log{(s+1)}^{1+\delta}} \Vert M_{s+1}\Vert^2_{\theta^{(s)}}, \qquad \delta >0.
    \end{equation}
    Then, combining the above with \eqref{eq:ms1-rs-inequality}, we have
    \begin{equation}
        \mathbb{E}[V_{s+1} \vert \tilde{\mathcal{F}}_s] \leq \underbrace{\left(\frac{s \log{s}^{1+\delta}}{(s+1)\log{(s+1)}^{1+\delta}} \right)}_{\leq 1}\left[(1 + O(\Vert v_s\Vert^2_{\theta^{(s-1)}}))V_s + \frac{4(C_0'+C_1'W_s^2)}{s\log{s}^{1+\delta}} \right].
    \end{equation}
    By the Robbins-Siegmund theorem, $V_s = O(1)$ and thus $\Vert M_s\Vert^2_{\theta^{(s-1)}} = O\left( \log{s^{1+\delta}}/s\right)$ a.s.

    \subsubsection{Claim: composite transport behaves like an isometry for large $k,s$.} \label{lemma:eigenvalue-lower-claim3}
    
    Let $\Gamma^{\mathcal{R}}_{s,s-1}$ denote parallel transport along $ \gamma(t) = \mathcal{R}_{\theta^{(s)}}(t \mathcal{R}^{-1}_{\theta^{(s)}}(\theta^{(s-1)}))$ from $t=0$ to $1$. 
    Write
    \begin{equation} \label{eq:expanded-inverse-transport}
        \mathcal{T}^{-1}_{[s,k]} = \Gamma^{\mathcal{R},-1}_{s,s-1}(\Gamma^{\mathcal{R}}_{s,s-1} \mathcal{T}^{-1}_{s,s-1})\ \circ\ ...\ \circ\ \Gamma^{\mathcal{R},-1}_{k+1,k}(\Gamma^{\mathcal{R}}_{k+1,k} \mathcal{T}^{-1}_{k+1,k}) .
    \end{equation}
    From \Cref{lemma:vector-transport-bounds}, there exist $c,N>0$ a.s. so that for $s> k > N$ and $\forall w \in T_{\theta^{(k)}}\mathcal{M}$:
    \begin{equation}
        \underbrace{\prod^{s-1}_{j=k}(1-c\Vert v_{j+1}\Vert^2_{\theta^{(j)}})\Vert}_{:=V_{k,s-1}^-} w\Vert_{\theta^{(k)}} \leq \Vert \mathcal{T}^{-1}_{[s,k]}[w]\Vert_{\theta^{(s)}} \leq \underbrace{\prod^{s-1}_{j=k}(1+c\Vert v_{j+1}\Vert^2_{\theta^{(j)}})}_{:=V_{k,s-1}^+}\Vert w\Vert_{\theta^{(k)}}.
    \end{equation}
    % To control $V_{k,s-1}^-$ and $V_{k,s-1}^+$, we note that, for a fixed sign sequence $a_k\ne-1$,
    To control $V_{k,s-1}^-$ and $V_{k,s-1}^+$, we note that, for a sequence $a_k\ne-1$,
    % \begin{equation}
    %     \sum^\infty_{k=0} a_k < \infty \Longrightarrow \prod^k_{j=0}(1+a_j) \underset{k\rightarrow \infty}{\rightarrow} a \ne 0.
    % \end{equation}
        \begin{equation}
        \sum^\infty_{k=0} \vert a_k\vert  < \infty \Longrightarrow \prod^k_{j=0}(1+a_j) \underset{k\rightarrow \infty}{\rightarrow} a \ne 0.
    \end{equation}
    Consequently,
    \begin{equation}
        \lim_{s\rightarrow \infty}V^-_{k,s} = V^-_{k,\infty}\in \mathbb{R}, \qquad \lim_{s\rightarrow \infty} V^+_{k,s} = V^+_{k,\infty} \in \mathbb{R}, \qquad \text{a.s.},
    \end{equation}
    where $V_{k,\infty}^- , V_{k,\infty}^+ \rightarrow 1$ as $k\rightarrow \infty$. Therefore, for all $w \in T_{\theta^*}\mathcal{M}$
    \begin{equation} \label{eq:composite-transport-almost-isometric}
        \lim_{k\rightarrow \infty}\sup_{s \geq k} \Vert \mathcal{T}^{-1}_{[s,k]} \circ \Gamma_{*,k}[w]\Vert_{\theta^{(s)}} = \lim_{k\rightarrow \infty}\inf_{s \geq k} \Vert \mathcal{T}^{-1}_{[s,k]} \circ \Gamma_{*,k}[w]\Vert_{\theta^{(s)}} = \Vert w\Vert_{\theta^*}.
    \end{equation}
    Clearly, a similar result holds for $\mathcal{T}^{-*}_{[s,k]}$ since it has the same singular values as $\mathcal{T}^{-1}_{[s,k]}$.
    % Note, whether $\Gamma_{s,*} \circ \mathcal{T}^{-1}_{[s,k]} \circ \Gamma_{*,k}$ converges as $s\rightarrow \infty$ is presently uncertain.

    \subsubsection{Eigenvalue bounds for $R_s$.} \label{lemma:eigenvalue-lower-claim4}
     Define
    \begin{equation}
        \tilde{R}_{s+1} = \Gamma_{s,*} \circ R_{s+1} \circ \Gamma_{*,s} = \frac{1}{s+1} \sum^s_{k=0} \Gamma_{s,*} \circ \mathcal{T}^{-1}_{[s,k]} \circ I_F(\theta^{(k)}) \circ \mathcal{T}^{-*}_{[s,k]} \circ \Gamma_{*,s}.
    \end{equation}
    For any $w \in T_{\theta^*} \mathcal{M}$, we have:
    \begin{equation}
        \langle w, \tilde{R}_{s+1}w\rangle_{\theta^*} = \frac{1}{s+1}\sum^s_{k=0}\langle\mathcal{T}^{-*}_{[s,k]}\Gamma_{*,s}w, I_F(\theta^{(k)}) \mathcal{T}^{-*}_{[s,k]}\Gamma_{*,s}w \rangle_{\theta^{(k)}}.
    \end{equation}
    Since $I_F(\theta^{(k)}) \rightarrow I_F(\theta^*)$ where $\lambda_{\min}(I_F(\theta^*)) := \lambda_{\min} >0$, for large $s,k$ we have
    \begin{equation}
        \langle\mathcal{T}^{-*}_{[s,k]}\Gamma_{*,s}w, I_F(\theta^{(k)}) \mathcal{T}^{-*}_{[s,k]}\Gamma_{*,s}w \rangle_{\theta^{(k)}} \geq \lambda_{\min} \Vert \mathcal{T}^{-*}_{[s,k]} \Gamma_{*,s}w\Vert^2_{\theta^{(k)}} > \frac{1}{2}\lambda_{\min} \Vert w\Vert^2_{\theta^*}, \qquad \text{a.s.}
    \end{equation}
    Hence, the eigenvalues of $\mathbf{H}_s$ are bounded below a.s. 
    The upper bound follows similarly.
\end{proof}

\newpage
\subsection{Proof of \Cref{thm:global-convergence}} \label{appx:proof-global-convergence}

The proof is similar to Theorem 5.1 in \citet{godichon2024natural}. However, demonstrating that $\lambda_{max}(\mathbf{H}_{s+1})$ is bounded above almost surely is more involved.

\begin{proof}[Proof of \Cref{thm:global-convergence}]

By the $L_0$-smoothness of $\mathcal{L}$ with respect to the retraction
\begin{equation}
    \mathcal{L}(\theta^{(s+1)}) \leq \mathcal{L}(\theta^{(s)})+\langle\nabla \mathcal{L}(\theta^{(s)}), \mathcal{R}^{-1}_{\theta^{(s)}}(\theta^{(s+1)}) \rangle_{\theta^{(s)}}+\frac{L_0}{2}\Vert \mathcal{R}^{-1}_{\theta^{(s)}}(\theta^{(s+1)})\Vert^2_{\theta^{(s)}}.
\end{equation}
From \eqref{eq:approx-ngd-update} we have $\mathcal{R}^{-1}_{\theta^{(s)}}(\theta^{(s+1)})=  -\tau_{s+1}\mathbf{H}_{s+1}^{-1} \tfrac{1}{B}\sum^{B}_{i=1}g(y_{s+1,i},\theta^{(s)})$. Substituting
    \begin{align}
        \mathcal{L}(\theta^{(s+1)}) & \leq \mathcal{L}(\theta^{(s)}) -\frac{\tau_{s+1}}{B}\sum^{B}_{i=1}\langle\nabla\mathcal{L}(\theta^{(s)}),\mathbf{H}_{s+1}^{-1}g(y_{s+1,i},\theta^{(s)}) \rangle_{\theta^{(s)}} \nonumber \\
        & \hspace{150pt} +\frac{L_0}{2}\left\Vert \frac{\tau_{s+1}}{B}\sum^{B}_{i=1} \mathbf{H}_{s+1}^{-1}g(y_{s+1,i},\theta^{(s)})\right\Vert^2_{\theta^{(s)}} \\
        & \leq \mathcal{L}(\theta^{(s)}) -\frac{\tau_{s+1}}{B}\sum^{B}_{i=1}\langle\nabla\mathcal{L}(\theta^{(s)}),\mathbf{H}_{s+1}^{-1}g(y_{s+1,i},\theta^{(s)}) \rangle_{\theta^{(s)}} \nonumber \\
        & \hspace{150pt} + \frac{L_0 \tau_{s+1}^2}{2B}\Vert \mathbf{H}^{-1}_{s+1}\Vert^2_{\text{op}}\sum^{B}_{i=1}\Vert g(y_{s+1,i},\theta^{(s)}) \Vert^2_{\theta^{(s)}}.
    \end{align}
    Let $W_s = \mathcal{L}(\theta^{(s)}) - \mathcal{L}(\theta^*)$, and define $\mathcal{F}_s$ as in \Cref{lemma:eigenvalue-lower-claim1} of \Cref{lemma:eigenvalue-lower-bound}. We have
    \begin{align}
        \mathbb{E}[W_{s+1}\vert\mathcal{F}_{s}] & \leq W_{s} -\frac{\tau_{s+1}}{B}\sum^{B}_{i=1}\langle\nabla\mathcal{L}(\theta^{(s)}),\mathbf{H}_{s+1}^{-1} \nabla\mathcal{L}(\theta^{(s)}) \rangle_{\theta^{(s)}} \nonumber \\
        & \hspace{50pt} + \frac{L_0 \tau_{s+1}^2}{2B}\Vert \mathbf{H}^{-1}_{s+1}\Vert^2_{\text{op}}\sum^{B}_{i=1} \mathbb{E} [\Vert g(y_{s+1,i},\theta^{(s)}) \Vert^2_{\theta^{(s)}}\vert \mathcal{F}_s].
    \end{align}
In the first line, we used $\mathbb{E}_{Y\sim p_\theta}[g(Y,\theta)] = \nabla \mathcal{L}(\theta)$. From condition (3) in the theorem,
    \begin{align}
        \mathbb{E}[W_{s+1}\vert\mathcal{F}_{s}] & \leq W_{s} -\tau_{s+1}\langle\nabla\mathcal{L}(\theta^{(s)}),\mathbf{H}_{s+1}^{-1} \nabla \mathcal{L}(\theta^{(s)}) \rangle_{\theta^{(s)}} \nonumber \\
        & \hspace{100pt} + \frac{L_0 \tau_{s+1}^2}{2}\Vert \mathbf{H}^{-1}_{s+1}\Vert^2_{\text{op}}(C_0 + C_1W_s).
    \end{align}
    Taking $\mathbf{H}^{-1}_{s+1}$ outside of the inner product and rearranging terms gives
    \begin{align}
        \mathbb{E}[W_{s+1}\vert\mathcal{F}_{s}] & \leq (1 +\frac{1}{2} C_1 L_0 \tau_{s+1}^2 \Vert \mathbf{H}^{-1}_{s+1}\Vert^2_{\text{op}})W_{s} -\tau_{s+1}\lambda_{\text{min}}(\mathbf{H}_{s+1}^{-1})\Vert \nabla\mathcal{L}(\theta^{(s)})\Vert^2_{\theta^{(s)}} \nonumber \\
        & \hspace{230pt} + \frac{1}{2}C_0L_0 \tau_{s+1}^2\Vert \mathbf{H}^{-1}_{s+1}\Vert^2_{\text{op}}.
    \end{align}
    Since $\mathbf{H}_{s+1}^{-1}$ is self-adjoint $\Vert \mathbf{H}^{-1}_s\Vert_{\text{op}} = \lambda_{\max}(\mathbf{H}^{-1}_s) = O(s^\beta)$ a.s., where the last equality is by \Cref{assumption:eigenvalues}. Then, since $\tau_s \propto (c_\alpha'+s)^{-\alpha}$ where $\beta < \alpha-1/2$, we have $2\alpha -2\beta >1$ and
    \begin{equation}
        \sum_s \tau_{s+1}^2 \Vert \mathbf{H}_{s+1}^{-1}\Vert^2_{\text{op}} = O\big( \sum_s s^{2\beta -2\alpha}\big)< + \infty, \qquad \text{a.s.}
    \end{equation}
    By the Robbins-Siegmund theorem $W_s$ converges a.s. to a finite RV $W_\infty$. Hence, we have the first part of \Cref{thm:global-convergence}. For the second assertion, from the Robbins-Siegmund theorem
    \begin{equation} \label{eq:bounded-gradient-sum}
        \sum_s \tau_{s+1} \lambda_{\min}( \mathbf{H}^{-1}_{s+1}) \Vert \nabla \mathcal{L}(\theta^{(s)})\Vert^2_{\theta^{(s)}} < \infty, \qquad \text{a.s.}
    \end{equation}
    Thus, if we can show $\lambda_{\max}(\mathbf{H}_{s+1})$ is bounded above a.s., then $\sum_s \tau_{s+1} \Vert \nabla \mathcal{L}(\theta^{(s)})\Vert^2_{\theta^{(s)}} < \infty$, which implies that $\min^s_{k=0}\Vert \nabla \mathcal{L}(\theta^{(k)})\Vert^2_{\theta^{(k)}} = o(s^{-(1-\alpha)})$; see Lemma 2 in \citet{liu2022almost}.

\subsubsection{Claim: $\lambda_{\max}(\mathbf{H}_{k+1})$ is bounded above almost surely.}
We re-use reasoning from \Cref{lemma:eigenvalue-lower-claim1,lemma:eigenvalue-lower-claim2,lemma:eigenvalue-lower-claim3}. For these claims it suffices to assume the conditions\footnote{In particular, \Cref{thm:global-convergence} assumes that the iterates are restricted to a compact domain explicitly, so the applications of \Cref{lemma:vector-transport-bounds} remain valid. This is the only reason for assuming compactness.} of \Cref{thm:global-convergence}, and the previous result on the convergence of $W_s$.

Recall $\mathbf{H}_{s+1} = R_{s+1}+ M_{s+1}$, where $\Vert M_{s+1}\Vert_{\theta^{(s)}} \rightarrow 0$ a.s., and
\begin{equation}
    R_{s+1} = \frac{1}{s+1}\sum^s_{k=0} \Phi_{k,s}, \qquad \Phi_{k,s} = \mathcal{T}^{-1}_{[s,k]} \circ I_F(\theta^{(k)}) \circ \mathcal{T}^{-*}_{[s,k]} .
\end{equation}
Using standard properties of the HS norm, and \Cref{lemma:vector-transport-bounds}
\begin{equation}
    \Vert\Phi_{k,s} \Vert^2_{\theta^{(s)}} \leq \Vert \mathcal{T}^{-1}_{[s,k]}\Vert^4_{\text{op}} \Vert I_F(\theta^{(k)})\Vert^2_{\theta^{(k)}} \leq \prod^{s-1}_{j=k}(1 + O (\Vert v_{j+1}\Vert^2_{\theta^{(j)}}))\times \Vert I_F(\theta^{(k)})\Vert^2_{\theta^{(k)}} .
\end{equation}
From Jensen's inequality and condition (4) in the theorem
\begin{equation}
    \Vert I_F(\theta^{(s)})\Vert^2_{\theta^{(s)}}\leq  \mathbb{E}[\Vert \nabla \log{q_{\theta^{(s)}}(\bar{y}_{s+1})} \Vert^4_{\theta^{(s)}}\vert \tilde{\mathcal{F}}_s] \leq C_0' + C_1' W_s^2 .
\end{equation}
The summability of the $\Vert v_{s+1} \Vert^2_{\theta^{(s)}}$ implies that the product term is almost surely bounded above by a finite random variable. Since this holds for every term comprising the summation in $R_{s+1}$, we conclude that $\Vert R_{s+1}\Vert^2_{\theta^{(s)}}$ is also bounded above almost surely. Thus
\begin{equation}
    \limsup_{s \rightarrow \infty} \Vert \mathbf{H}_{s+1}\Vert^2_{\theta^{(s)}} < \infty, \qquad \text{a.s.}
\end{equation}
The same holds for $\lambda_{\max}( \mathbf{H}_{s+1})$ since the HS norm upper-bounds the operator norm.
\end{proof}

\subsubsection{Remark: convergence to the set of minimizers.}

Suppose that $\mathcal{L}$ is geodesically\footnote{The argument can be reproduced with minor changes if $\mathcal{L}$ is retraction-convex w.r.t. some retraction.} convex, then we have
\begin{equation}
    W_s = \mathcal{L}(\theta^{(s)}) - \mathcal{L}(\theta^*) \leq -\langle \nabla \mathcal{L}(\theta^{(s)}), \exp^{-1}_{\theta^{(s)}}(\theta^*)\rangle_{\theta^{(s)}} \leq \Vert \nabla \mathcal{L}(\theta^{(s)})\Vert_{\theta^{(s)}} d(\theta^{(s)},\theta^*).
\end{equation}
Provided $\mathcal{L}$ has bounded level sets, then the first conclusion of the theorem implies that $d(\theta^{(s)}, \theta^*)$ is uniformly bounded almost surely. Consequently \eqref{eq:bounded-gradient-sum} implies that
\begin{equation}
    \sum_s\tau_s W_s^2 \leq \sum_s \tau_s \Vert \nabla \mathcal{L}(\theta^{(s)})\Vert^2 d(\theta^{(s)},\theta^*)^2<\infty.
\end{equation}
Since $W_s$ converges a.s. to a finite RV and $\sum_s \tau_s = \infty$, $W_s \rightarrow 0$. Therefore, $\theta^{(s)}$ converges a.s. to the set of minimizers. If $\mathcal{L}$ is \textit{strictly} convex along geodesics, the minimizer is unique.
% If $\mathcal{L}$ does not have bounded level sets, then compactness of $\mathcal{X}$ also ensures that $W_s \rightarrow 0$. 

% Otherwise, we say $\mathcal{L}$ satisfies a Polyak-Lojasiewicz (PL) inequality with parameter $\mu$ if
% \begin{equation}
%     \tfrac{1}{2}\Vert \nabla \mathcal{L}(\theta^{(s)})\Vert^2_{\theta^{(s)}} \geq \mu W_s,
% \end{equation}
% which also implies $W_s \rightarrow 0$ a.s. This holds if e.g., $\mathcal{L}$ is $\mu$-strongly convex along geodesics.

\newpage
\subsection{Proof of Theorem \ref{thm:convergence-rate}} \label{appx:convergence-rate-proof}

The proof is similar to Theorem 5.2 in \citet{godichon2024natural}. The key idea is to frame our analysis in the tangent space of the limit point $T_{\theta^*}\mathcal{M}$, and work with
\begin{equation}
    \Delta_{k} := \mathcal{R}^{-1}_{\theta^*}(\theta^{(k)}) \in T_{\theta^*}\mathcal{M}.
\end{equation}
The new challenges include managing the error in $\mathbf{H}_{s+1}^{-1}$ arising from vector transportation, and not assuming that $\mathbf{H}_{s+1}^{-1}$ converges to $I_F(\theta^*)$, since we require Theorem \ref{thm:convergence-rate} to prove this.

\subsubsection{Defining the recursion for $\Delta_{k+1}$}

Define the function
\begin{equation}
    F_{\theta} = \mathcal{R}^{-1}_{\theta^*} \circ \mathcal{R}_\theta : T_\theta \mathcal{M} \rightarrow T_{\theta^*} \mathcal{M}, \qquad F_k := F_{\theta^{(k)}}.
\end{equation}
The linearized iterates can be expressed as
\begin{equation}
    \Delta_{k+1} = \mathcal{R}^{-1}_{\theta^*} (\theta^{(k+1)}) = \mathcal{R}^{-1}_{\theta^*} \circ \mathcal{R}_{\theta^{(k)}}(v_{k+1}) = F_k(v_{k+1}).
\end{equation}
From Lemma 4 in \citet{tripuraneni2018averaging}, we can expand $F_k$ via a Taylor expansion
\begin{equation} 
    F_{k}(v_{k+1}) = F_k(0) + DF_k(0)[v_{k+1}] + D^2F_k(t_* v_{k+1})[v_{k+1}, v_{k+1}], \qquad t_*\in[0,1].
\end{equation}
We further have $F_k(0)=\Delta_k$, and it can be shown that
\begin{equation}
    DF_k(0) = [D\mathcal{R}_{\theta^*}(\mathcal{R}^{-1}_{\theta^*}(\theta^{(k)}))]^{-1} = \mathcal{T}^{-1}_{\theta^*, \theta^{(k)}} := \mathcal{T}^{-1}_{*,k}.
\end{equation}
For the Hessian term, since $\mathcal{R}$ is smooth and $(\theta_k, v_{k+1}) \rightarrow (\theta^*, 0_{\theta^*})$ almost surely
\begin{equation}
     \Vert D^2F_k(t_* v_{k+1})[v_{k+1},v_{k+1}]\Vert_{\theta^*} = O(\Vert v_{k+1}\Vert^2_{\theta^{(k)}}), \qquad \text{a.s.}
\end{equation}
% \textcolor{red}{For the Hessian term, since $DF_{\theta^*}(0)=\text{Id}_{\theta^*}$, we should be able to argue that for $\theta\in \mathcal{M}$ within some small radius of $\theta^*$, we have $\Vert D^2F_{\theta}(u)\Vert_{\text{op}} =O(d(\theta^*,\theta))$ for all sufficiently small $u \in T_\theta \mathcal{M}$.} Then, since $(\theta_k, v_{k+1}) \rightarrow (\theta^*, 0_{\theta^*})$ almost surely 
% \begin{equation}
%     \Vert DF_k^2(t_* v_{k+1})[v_{k+1},v_{k+1}]\Vert_{\theta^{(k)}} = O(\Vert\Delta_k\Vert_{\theta^*}\Vert v_{k+1}\Vert^2_{\theta^{(k)}}), \qquad \text{a.s.}
% \end{equation}
Therefore, we have the following recursive expression
\begin{equation}
    \Delta_{k+1} = \Delta_k + \mathcal{T}^{-1}_{*,k}[v_{k+1}] + O(\Vert v_{k+1}\Vert^2_{\theta^{(k)}}).
\end{equation}
% \begin{equation}
%     \Delta_{k+1} = \Delta_k + \mathcal{T}^{-1}_{*,k}[v_{k+1}] + O(\Vert\Delta_k\Vert_{\theta^*}\Vert v_{k+1}\Vert^2_{\theta^{(k)}})
% \end{equation}
Consider an expanded version of this recursion, where $\Gamma_{x,y}$ denotes geodesic PT,
\begin{align}
    \Delta_{k+1} & = \Delta_k - \tau_{k+1} \mathcal{T}^{-1}_{*,k}\big[ \mathbf{H}^{-1}_{k+1} \widehat{\nabla \mathcal{L}}(\theta^{(k)})\big] + O( \Vert v_{k+1}\Vert^2_{\theta^{(k)}}) \\
    & = \Delta_k - \tau_{k+1} \mathcal{T}^{-1}_{*,k} \Gamma_{*,k} \ \circ \underbrace{\Gamma^{-1}_{*,k} \mathbf{H}^{-1}_{k+1} \Gamma_{*,k}}_{:=\tilde{\mathbf{H}}^{-1}_{k+1}} \circ \underbrace{\Gamma^{-1}_{*,k} \widehat{\nabla \mathcal{L}}(\theta^{(k)})}_{:=\hat{g}_k} + O(\Vert v_{k+1}\Vert^2_{\theta^{(k)}}).
\end{align}
% Consider an expanded version of this recursion, where $\Gamma_{x,y}$ denotes geodesic parallel transport
% \begin{align}
%     \Delta_{k+1} & = \Delta_k - \tau_{k+1} \mathcal{T}^{-1}_{*,k}\big[ \mathbf{H}^{-1}_{k+1} \widehat{\nabla \mathcal{L}}(\theta^{(k)})\big] + O(\Vert\Delta_k\Vert_{\theta^*} \Vert v_{k+1}\Vert^2_{\theta^{(k)}}) \\
%     & = \Delta_k - \tau_{k+1} \mathcal{T}^{-1}_{*,k} \Gamma_{*,k} \ \circ \underbrace{\Gamma^{-1}_{*,k} \mathbf{H}^{-1}_{k+1} \Gamma_{*,k}}_{:=\tilde{\mathbf{H}}^{-1}_{k+1}} \circ \underbrace{\Gamma^{-1}_{*,k} \widehat{\nabla \mathcal{L}}(\theta^{(k)})}_{:=\hat{g}_k} + O(\Vert\Delta_k\Vert_{\theta^*}\Vert v_{k+1}\Vert^2_{\theta^{(k)}})
% \end{align}
From \Cref{lemma:transport-pt} we have $\mathcal{T}^{-1}_{*,k} \Gamma_{*,k} = \text{Id}_{\theta^*} + O(\Vert \Delta_k\Vert^2_{\theta^{*}})$ almost surely, hence
\begin{equation}
    \Delta_{k+1} = \Delta_k - \tau_{k+1} \tilde{\mathbf{H}}_{k+1}^{-1} \hat{g}_{k} + \underbrace{O(\Vert v_{k+1}\Vert^2_{\theta^{(k)}} + \Vert \Delta_k \Vert^2_{\theta^{*}}\Vert v_{k+1}\Vert_{\theta^{(k)}})}_{:=\zeta_{k+1}}.
\end{equation}
% \begin{equation}
%     \Delta_{k+1} = \Delta_k - \tau_{k+1} \tilde{\mathbf{H}}_{k+1}^{-1} \hat{g}_{k} + \underbrace{O(\Vert\Delta_k\Vert_{\theta^*}\Vert v_{k+1}\Vert^2_{\theta^{(k)}} + \Vert \Delta_k \Vert^2_{\theta^{*}}\Vert v_{k+1}\Vert_{\theta^{(k)}})}_{:=\zeta_{k+1}}
% \end{equation}
Some further manipulations yield
\begin{align}
    \Delta_{k+1} & = \Delta_k - \tau_{k+1} \tilde{\mathbf{H}}_{k+1}^{-1} \circ \underbrace{\Gamma^{-1}_{*,k} \nabla \mathcal{L}(\theta^{(k)})}_{:=g_k} + \tau_{k+1} \tilde{\mathbf{H}}_{k+1}^{-1}\underbrace{(g_k - \hat{g}_{k})}_{:=\xi_{k+1}} + \zeta_{k+1} \\
    & = \underbrace{(\text{Id} - \tau_{k+1} \tilde{\mathbf{H}}_{k+1}^{-1} \nabla^2 \mathcal{L}(\theta^*))}_{:=J_{k+1}}\Delta_k - \tau_{k+1} \tilde{\mathbf{H}}_{k+1}^{-1}\underbrace{(g_k - \nabla^2 \mathcal{L}(\theta^*)\Delta_k)}_{:=\delta_{k}} \nonumber \\
    & \hspace{230pt} + \tau_{k+1} \tilde{\mathbf{H}}_{k+1}^{-1}\xi_{k+1} + \zeta_{k+1}.
\end{align}
To summarize, we have shown that
\begin{equation} \label{eq:delta-recursion}
    \Delta_{k+1} = J_{k+1} \Delta_k - \tau_{k+1} \tilde{\mathbf{H}}_{k+1}^{-1} \delta_k + \tau_{k+1} \tilde{\mathbf{H}}_{k+1}^{-1}\xi_{k+1} + \zeta_{k+1},
\end{equation}
where
\begin{equation}
    g_k = \Gamma^{-1}_{*,k} \nabla \mathcal{L}(\theta^{(k)}), \qquad \hat{g}_k = \Gamma^{-1}_{*,k} \widehat{\nabla \mathcal{L}}(\theta^{(k)}), \qquad \tilde{\mathbf{H}}^{-1}_{k+1} = \Gamma^{-1}_{*,k} \mathbf{H}^{-1}_{k+1} \Gamma_{*,k},
\end{equation}
\begin{equation}
    \xi_{k+1} = g_k - \hat{g}_k, \qquad \delta_k = g_k - \nabla^2 \mathcal{L}(\theta^*) \Delta_k, \qquad J_{k+1} = (\text{Id} - \tau_{k+1} \tilde{\mathbf{H}}^{-1}_{k+1} \nabla^2 \mathcal{L}(\theta^*)),
\end{equation}
\begin{equation}
    \zeta_{k+1}=O(\Vert v_{k+1}\Vert^2_{\theta^{(k)}} + \Vert \Delta_k \Vert^2_{\theta^{*}}\Vert v_{k+1}\Vert_{\theta^{(k)}}).
\end{equation}
% \begin{equation}
%     \zeta_{k+1}=O(\Vert\Delta_k\Vert_{\theta^*}\Vert v_{k+1}\Vert^2_{\theta^{(k)}} + \Vert \Delta_k \Vert^2_{\theta^{*}}\Vert v_{k+1}\Vert_{\theta^{(k)}}).
% \end{equation}
Expanding \eqref{eq:delta-recursion} recursively yields
\begin{align}
    \Delta_{s+1} = & \underbrace{\beta_{s+1,0} \Delta_0}_{:=M^1_{s+1}} + \underbrace{\sum^s_{k=0} \beta_{s+1,k+1} \tau_{k+1} \tilde{\mathbf{H}}^{-1}_{k+1}\xi_{k+1}}_{:=M_{s+1}^2} \nonumber \\
    & + \underbrace{\sum^s_{k=0} \beta_{s+1,k+1} \big(\tau_{k+1} \tilde{\mathbf{H}}^{-1}_{k+1}\delta_k + O( \Vert v_{k+1}\Vert_{\theta^{(k)}}^2+\Vert \Delta_k\Vert^2_{\theta^{*}} \Vert v_{k+1}\Vert_{\theta^{(k)}})\big)}_{:=M^3_{s+1}},
\end{align}
% \begin{align}
%     \Delta_{s+1} =&  \underbrace{\beta_{s+1,0} \Delta_0}_{:=M^1_{s+1}} + \underbrace{\sum^s_{k=0} \beta_{s+1,k+1} \tau_{k+1} \tilde{\mathbf{H}}^{-1}_{k+1}\xi_{k+1}}_{:=M_{s+1}^2}\\
%     & + \underbrace{\sum^s_{k=0} \beta_{s+1,k+1} \big(\tau_{k+1} \tilde{\mathbf{H}}^{-1}_{k+1}\delta_k + O(\Vert \Delta_k\Vert_{\theta^{*}} \Vert v_{k+1}\Vert_{\theta^{(k)}}^2+\Vert \Delta_k\Vert^2_{\theta^{*}} \Vert v_{k+1}\Vert_{\theta^{(k)}})\big)}_{:=M^3_{s+1}},
% \end{align}
where $\beta_{s,k} := \prod^s_{i=k+1} J_i$ and $\beta_{s,s} := \text{Id}$.

\subsubsection{Behavior of $\Vert J_{k+1}\Vert_{\text{op}}$} \label{sec:Jk-behavior}

The operators $\tilde{\mathbf{H}}^{-1}_{k+1}$ and $\nabla^2\mathcal{L}(\theta^*)$ are self-adjoint, therefore
\begin{align}
    J_{k+1}^*J_{k+1} & = \text{Id} - \tau_{k+1} \tilde{\mathbf{H}}^{-1}_{k+1} \nabla^2 \mathcal{L}(\theta^*) - \tau_{k+1}\nabla^2 \mathcal{L}(\theta^*) \tilde{\mathbf{H}}^{-1}_{k+1} \nonumber \\ 
    & \hspace{180pt} + \tau_{k+1}^2 \nabla^2 \mathcal{L}(\theta^*)\tilde{\mathbf{H}}^{-2}_{k+1}  \nabla^2 \mathcal{L}(\theta^*).
\end{align}
By assumption $\lambda(\nabla^2\mathcal{L}(\theta^*)) \in (\lambda_{\min}, \lambda_{\max})$ where $\lambda_{\min}>0$. From \Cref{lemma:eigenvalue-lower-bound} we can also assume $\lambda(\mathbf{H}_{k+1}^{-1})\in (\lambda_{\min}, \lambda_{\max})$ almost surely for all $k$ sufficiently large. Consequently
\begin{equation}
    \Vert J_{k+1}\Vert^2_{\text{op}} = \lambda_{\max}(J_{k+1}^* J_{k+1}) \leq 1 - 2 \lambda_{\min}^2 \tau_{k+1} + \tau_{k+1}^2 \lambda_{\max}, \qquad \text{a.s.}
\end{equation}
Therefore, $\exists c,K>0$ such that $\Vert J_{k+1}\Vert_{\text{op}} \leq 1 - c \tau_{k+1}$ for all $k \geq K$ almost surely.

\subsubsection{Convergence rate of $M_s^1$:}
For $c,K$ as above, we have

\begin{equation}
    \Vert M_s^1 \Vert_{\theta^*} \leq \big\Vert \prod^{K-1}_{k=1}J_k \Delta_0\big\Vert_{\theta^*} \times \prod^{s}_{k=K}(1-c\tau_{k+1}), \qquad \text{a.s.}
\end{equation}
The first term is bounded almost surely. For the second term:
\begin{equation}
    \prod^{s}_{k=K}(1-c\tau_{k+1}) = O\big( \exp{\big[-c \sum^{s}_{k=K}\tau_{k+1}\big]}\big) = O(\exp[-cs^{1-\alpha}]).
\end{equation}
Therefore $\Vert M_s^1\Vert_{\theta^*}= O(\exp[-cs^{1-\alpha}])$ almost surely.

\subsubsection{Convergence rate of $M_s^2$:}
Define $c,K$ as before, and split $M_s^2$ into two components
\begin{equation}
    M_{s+1}^2 = \sum^{K-1}_{k=0}\beta_{s+1,k+1}\tau_{k+1}\tilde{\mathbf{H}}^{-1}_{k+1}\xi_{k+1}+\sum^{s-1}_{k=K}\beta_{s+1,k+1}\tau_{k+1} \tilde{\mathbf{H}}^{-1}_{k+1}\xi_{k+1}.
\end{equation}
Let the first term be $M_{s+1}^{2a}$ and second $M_{s+1}^{2b}$. Using similar reasoning as for $\Vert M_{s+1}^1\Vert_{\theta^*}$,
\begin{equation}
    \Vert M^{2a}_{s+1}\Vert_{\theta^*} \leq \prod^s_{k=K} \left\Vert J_k\right\Vert_{\text{op}} \times \underbrace{\Vert M_{K-1}^2\Vert_{\theta^*}}_{=O(1)} = O(\exp[-cs^{1-\alpha}]), \qquad \text{a.s.}
\end{equation}
For $M_{s+1}^{2b}$, we would like to employ Theorem 6.1 in \citet{cenac2020efficient}, to conclude
\begin{equation} \label{eq:m2bs-convergence-rate}
    \Vert M_s^{2b}\Vert_{\theta^*} = O(\sqrt{\log{s}/s^\alpha}), \qquad \text{a.s.}
\end{equation}
However, our $\beta_{n,k}$ take a different form to theirs; they define these as
\begin{equation}
    \beta_{n,k} = \prod^n_{j=k+1}(\text{Id}-\tau_{j+1} \Gamma), \qquad \beta_{n,n}=\text{Id}.
\end{equation}
where $\Gamma$ is a positive, bounded, self-adjoint linear operator. Then for sufficiently large $k$
\begin{equation}
    \Vert \text{Id} - \tau_{k+1}\Gamma\Vert_{\text{op}} \leq (1 - \tau_{k+1} \lambda_{\min}(\Gamma)).
\end{equation}
We remark that the proof of Theorem 6.1 in \citet{cenac2020efficient} uses the structure of $\beta_{n,k}$ only through the bound $\Vert \beta_{n,k}\Vert_{\mathrm{op}} \leq \prod^n_{j=k+1}(1-\lambda_{\min}\tau_{j+1})$. From \Cref{sec:Jk-behavior}, we can see that the same bound is guaranteed for $\Vert \beta_{n,k}\Vert_{\mathrm{op}}$ here almost surely for all $k$ sufficiently large. Consequently, the proof carries over without substantive changes, and we can conclude \eqref{eq:m2bs-convergence-rate}.

\subsubsection{Convergence rate of $M_s^3$:}
For sufficiently large $s$, we have
\begin{align}
    \mathbb{E}[\Vert M_{s+1}^3\Vert_{\theta^*} \vert \mathcal{F}_s] & \leq (1 - c\tau_{s+1}) \Vert M_s^3\Vert_{\theta^*} \\ 
    & + O(\tau_{s+1} \Vert \delta_s\Vert_{\theta^*} + \Vert v_{s+1}\Vert_{\theta^{(s)}}^2+ \Vert \Delta_s\Vert^2_{\theta^*} \Vert v_{s+1}\Vert_{\theta^{(s)}}).
\end{align}
% \begin{align}
%     \mathbb{E}[\Vert M_{s+1}^3\Vert_{\theta^*} \vert \mathcal{F}_s] & \leq (1 - c\tau_{s+1}) \Vert M_s^3\Vert_{\theta^*} \\ 
%     & + O(\tau_{s+1} \Vert \delta_s\Vert +\Vert \Delta_s\Vert_{\theta^*} \Vert v_{s+1}\Vert_{\theta^{(s)}}^2+ \Vert \Delta_s\Vert^2_{\theta^*} \Vert v_{s+1}\Vert_{\theta^{(s)}})
% \end{align}
By Holder's inequality, then using condition (3) in \Cref{thm:global-convergence} and \Cref{lemma:eigenvalue-lower-bound}
\begin{equation}
    \mathbb{E}[\Vert v_{s+1}\Vert_{\theta^{(s)}}\vert \mathcal{F}_s] \leq \mathbb{E}[\Vert v_{s+1}\Vert^2_{\theta^{(s)}}\vert \mathcal{F}_s]^{1/2} =O( \tau_{s+1}(C_0+C_1 W_s)^{1/2}).
\end{equation}
Therefore, since $\Vert \Delta_s\Vert_{\theta^*} = o(1)$ and $\tau_{s}^2 \propto (c_\alpha'+s)^{-2\alpha}$, we have
\begin{equation}
    \mathbb{E}[\Vert M_{s+1}^3\Vert_{\theta^*} \vert \mathcal{F}_s] \leq (1 - c\tau_{s+1}) \Vert M_s^3\Vert_{\theta^*} + O(s^{-2\alpha})+\tau_{s+1}O(\Vert \delta_s \Vert_{\theta^*} + \Vert \Delta_s\Vert_{\theta^*}^2).
\end{equation}
% \begin{equation}
%     \mathbb{E}[\Vert M_{s+1}^3\Vert \vert \mathcal{F}_s] \leq (1 - c\tau_{s+1}) \Vert M_s^3\Vert + o(s^{-2\alpha})+\tau_{s+1}O(\Vert \delta_s \Vert + \Vert \Delta_s\Vert^2)
% \end{equation}
Provided that $\Vert \delta_s\Vert_{\theta^*} = O(\Vert \Delta_s\Vert^2_{\theta^*})$, then the big-O term equals
\begin{equation}
    O(\Vert \delta_s\Vert_{\theta^*}+\Vert \Delta_s\Vert^2_{\theta^*}) = O\big( \frac{\log{s}}{s^\alpha} + \Vert M_s^3 \Vert^2_{\theta^*}\big).
\end{equation}
Since $\Vert \Delta_s \Vert_{\theta^*}, \Vert M_s^i\Vert_{\theta^*} \rightarrow 0$ a.s. for $i=1,2$, we must also have $\Vert M_s^3\Vert_{\theta^*} \rightarrow 0$. Hence

\begin{equation}
     \mathbb{E}[\Vert M_{s+1}^3\Vert_{\theta^*} \mid \mathcal{F}_s] \leq (1 - (c - O(\Vert M_s^3\Vert_{\theta^*}))\tau_{s+1}) \Vert M_s^3 \Vert_{\theta^*} + O\left(\frac{\log{s}}{s^{2\alpha}}\right).
\end{equation}
Therefore, eventually we have for some $\tilde{c}>0$
\begin{equation}
     \mathbb{E}[\Vert M_{s+1}^3\Vert_{\theta^*} \vert \mathcal{F}_s] \leq (1 - \tilde{c}s^{-\alpha}) \Vert M_s^3 \Vert_{\theta^*} + O\left(\frac{\log{s}}{s^{2\alpha}}\right).
\end{equation}
Taking $V_{s+1} =\Vert M_{s+1}^3\Vert_{\theta^*}(s+1)^{2\alpha-1}/\log{(s+1)}^{2+\delta}$, it follows that
\begin{equation}
    \mathbb{E}[V_{s+1} \mid \mathcal{F}_s] \leq \left(1-\frac{\tilde{c}}{s^\alpha}\right) \left(1 + \frac{1}{s}\right)^{2\alpha-1} V_s+ O\left( \frac{1}{s\log{s}^{1+\delta}}\right).
\end{equation}
One can show e.g. via a Taylor series that the coefficient of $V_s$ is $\leq 1$ eventually, hence
\begin{equation}
    \Vert M_{s+1}^3\Vert_{\theta^*} = O\left(\frac{\log{s}^{2+\delta}}{s^{2\alpha-1}}\right), \qquad \text{a.s.}
\end{equation}

\subsubsection{Claim: $\Vert \delta_s \Vert_{\theta^*} = O(\Vert \Delta_s\Vert^2_{\theta^*})$}
Expanding the definition of $\delta_s$
\begin{equation}
    \delta_s = \underbrace{\Gamma^{-1}_{*,s} \nabla \mathcal{L}(\theta^{(s)}) - \nabla^2 \mathcal{L}(\theta^*) \exp^{-1}_{\theta^*}(\theta^{(s)})}_{=(a)} + \underbrace{\nabla^2 \mathcal{L}(\theta^*) [\exp^{-1}_{\theta^*}(\theta^{(s)})-\Delta_s]}_{=(b)}.
\end{equation}
For (a), we can use a manifold version of Taylor's theorem to conclude this is $O(\Vert \Delta_s\Vert^2_{\theta^*})$; see e.g. Lemma 7.4.7 in \citet{absil2009optimization}. For (b), by a Taylor expansion we have
\begin{equation}
    f_x(v):=\exp_x^{-1} \circ \ \mathcal{R}_x(v) = v + O(\Vert v\Vert^2_x).
\end{equation}
Letting $x = \theta^{*}, v=\Delta_s=\mathcal{R}^{-1}_{\theta^*}(\theta^{(s)})$, it follows that
\begin{equation}
    f_x(v)-v = \exp^{-1}_{\theta^*}(\theta^{(s)})-\Delta_s = O(\Vert \Delta_s\Vert^2_{\theta^*}).
\end{equation}
Therefore, we have $\Vert \delta_s\Vert_{\theta^*} = O(\Vert \Delta_s\Vert^2_{\theta^*})$, which concludes the proof of the theorem.

\newpage
\subsection{Proof of \Cref{thm:fisher-consistency}} \label{appx:fisher-consistency-proof}

\begin{proof}
Recall we denote geodesic parallel transport from $\theta^{(s)}\rightarrow \theta^*$ by $\Gamma_{s,*} := \Gamma_{\theta^{(s)}, \theta^*}$.
From the proof of \Cref{lemma:eigenvalue-lower-bound}, we had that, for all $\delta >0$, the following holds a.s.:
\begin{equation} \label{eq:pt-fisher-expanded}
    \Gamma_{s,*} \circ \mathbf{H}_{s+1} \circ \Gamma_{*,s} = \frac{1}{s+1}\sum^s_{k=0} \Gamma_{s,*} \circ \Phi_{k,s} \circ \Gamma_{*,s} + O\left(\sqrt{\frac{\log{s^{1+\delta}}}{s}}\right) ,
\end{equation}
 where $\Phi_{k,s} = \mathcal{T}^{-1}_{[s,k]} \circ I_F(\theta^{(k)}) \circ \mathcal{T}^{-*}_{[s,k]}$. 
We can express each summand as
\begin{equation}
    \Gamma_{s,*} \circ \Phi_{k,s} \circ \Gamma_{*,s} = \left(\Gamma_{s,*} \mathcal{T}^{-1}_{[s,k]}\Gamma_{*,k}\right) \circ \left(\Gamma_{k,*}I_F(\theta^{(k)}) \Gamma_{*,k}\right) \circ \left(\Gamma_{k,*}\mathcal{T}^{-*}_{[s,k]} \Gamma_{*,s}\right).
\end{equation}
From the local Lipschitz continuity of $I_F$ at $\theta^*$ and \Cref{thm:convergence-rate}, we have
\begin{equation} \label{eq:pt-fisher-rate}
    \Vert \Gamma_{s,*} I_F(\theta^{(s)})\Gamma_{*,s} - I_F(\theta^*)\Vert_{\text{op}} = O\left(\sqrt{\frac{\log{s}}{s^\alpha}}\right).
\end{equation}
The goal now is to prove the following convergence rate; the theorem follows by expanding the summation in \eqref{eq:pt-fisher-expanded} and examining the convergence rates of the individual error terms.
\begin{equation}
    \sup_{s\geq k} \Vert \Gamma_{s,*} \mathcal{T}^{-1}_{[s,k]}\Gamma_{*,k} - \text{Id}_{\theta^*}\Vert_{\text{op}} =O\left(\frac{\log{k}^{1+\delta}}{k^{3\alpha/2-1}}\right), \qquad \delta > 0.
\end{equation}
Expanding the composite transport operation
\begin{equation} \label{eq:composite-transport}
    \Gamma_{s,*} \mathcal{T}^{-1}_{[s,k]} \Gamma_{*,k} = (\Gamma_{s,*} \mathcal{T}^{-1}_{s,s-1} \Gamma_{*,s-1}) \ \circ \ ... \ \circ \ (\Gamma_{k+1,*} \mathcal{T}^{-1}_{k+1,k} \Gamma_{*,k}).
\end{equation}
We bound the distance to $\text{Id}_{\theta^*}$ for individual terms using \Cref{lemma:triangle-transport}
\begin{equation}
    \Gamma_{s+1,*} \mathcal{T}^{-1}_{s+1,s} \Gamma_{*,s} = \text{Id}_{*} + O(d(\theta^{(s)},\theta^{(s+1)})d(\theta^{(s)},\theta^*) + d(\theta^{(s)},\theta^{(s+1)})^2).
\end{equation}
Recall from \Cref{lemma:eigenvalue-lower-claim1} in the proof of \Cref{lemma:eigenvalue-lower-bound} that
\begin{equation}
    \mathbb{E}[\Vert v_{s+1} \Vert^2_{\theta^{(s)}} \vert \mathcal{F}_s] \leq \tau_{s+1}^2 \Vert \mathbf{H}_{s+1}^{-1} \Vert^2_{\text{op}} (C_0 + C_1 W_s) := (1+s)^{-2\alpha} \tilde{W}_s.
\end{equation}
From \Cref{lemma:eigenvalue-lower-bound} and the proof of \Cref{thm:global-convergence}, $\tilde{W}_s$ is bounded almost surely. Write
\begin{equation}
    V_{s+1} = \frac{(s+1)^{2\alpha-1}}{\log{(s+1)^{1+\delta}}} \Vert v_{s+1}\Vert^2_{\theta^{(s)}}, \qquad \delta > 0.
\end{equation}
Combining this with the previous inequality yields that
\begin{equation}
    \mathbb{E}[V_{s+1} \vert \mathcal{F}_s] \leq V_s - V_s + \frac{\tilde{W}_s}{(s+1)\log{(s+1)^{1+\delta}}}.
\end{equation}
The final term has a finite sum a.s., therefore by the Robbins-Siegmund theorem,
\begin{equation} \label{eq:vs-scaled-sum}
    \sum^\infty_{s=2} V_s = \sum^\infty_{s=2} \frac{s^{2\alpha-1}}{\log{s}^{1+\delta}} \Vert v_s \Vert^2_{\theta^{(s-1)}} < \infty, \qquad \text{a.s.}
\end{equation}
Let $c_s:= s^{\gamma}/(\log{s})^{1+\delta}$ for some $\delta >0$ and $\gamma < 2\alpha-1$; by Cauchy-Schwarz
\begin{equation} \label{eq:sum-cauchy-schwarz}
    \sum^\infty_{s=k} d(\theta^{(s)}, \theta^{(s+1)}) d(\theta^{(s)},\theta^*)\leq  \left(\sum^\infty_{s=k}c_s d(\theta^{(s)}, \theta^{(s+1)})^2\right)^{1/2} \left(\sum^\infty_{s=k}c_s^{-1}d(\theta^{(s)}, \theta^*)^2\right)^{1/2}.
\end{equation}
From \Cref{lemma:retraction-ineq} we have $d(\theta^{(s)}, \theta^{(s+1)}) = O(\Vert v_{s+1}\Vert_{\theta^{(s)}})$, and hence almost surely
\begin{equation}
    \sum^\infty_{s=k}c_sd(\theta^{(s)},\theta^{(s+1)})^2 \leq \frac{1}{k^{2\alpha-1-\gamma}}\sum^\infty_{s=k} \frac{s^{2\alpha-1}}{\log{s}^{1+\delta}} d(\theta^{(s)},\theta^{(s+1)})^2 = O(k^{-(2\alpha-1-\gamma)}).
\end{equation}
For the other term, if $\gamma >1-\alpha$ (note, $1-\alpha <2\alpha -1\Longleftrightarrow \alpha >2/3$), by a standard result
\begin{equation}
    \sum^\infty_{s=k} c_s^{-1} d(\theta^{(s)},\theta^*)^2 =O\left( \sum^\infty_{s=k} \frac{\log{s}^{2+\delta}}{s^{\alpha+\gamma}}\right) = O\left( \frac{\log{k}^{2+\delta}}{k^{\alpha+\gamma-1}}\right), \qquad \text{a.s.}
\end{equation}
Consequently, since $\sum^\infty_{s=k}d(\theta^{(s)},\theta^{(s+1)})^2$ has a comparatively negligible rate:
\begin{align}
    \sum^\infty_{s=k} \underbrace{[d(\theta^{(s)},\theta^{(s+1)})d(\theta^{(s)},\theta^*)+d(\theta^{(s)},\theta^{(s+1)})^2]}_{:=a_s} & = O\left(\sqrt{\frac{\log{k}^{2+\delta}}{k^{\alpha+\gamma-1}}}\times \sqrt{\frac{1}{k^{2\alpha-1-\gamma}}} \right) \\ & = O\left(\frac{\log{k}^{1+\delta/2}}{k^{3\alpha/2-1}} \right).
\end{align}
The tail product of $(1+a_s)$ converges to $1$ at the same rate; since $e^x - 1 = O(x)$ for small $x$,
\begin{equation}
    \prod^\infty_{s=k}(1+a_s)-1 = \exp{\left(\sum^\infty_{s=k}\log{(1+a_s)} \right)}-1 \leq \exp\left({\sum^\infty_{s=k}a_s}\right)-1 = O\left(\sum^\infty_{s=k}a_s\right).
\end{equation}
For a sequence of linear operators $A_n$ and sub-multiplicative norm $\Vert \cdot \Vert$, by induction
\begin{equation}
    \left\Vert \prod^n_{i=m} (\text{Id} + A_i) - \text{Id} \right\Vert \leq \prod^n_{i=m} (1 + \Vert A_i \Vert) -1.
\end{equation}
Therefore, from \eqref{eq:composite-transport}, we have for all $k\leq s$
\begin{equation}
    \Vert \Gamma_{s,*} \mathcal{T}^{-1}_{[s,k]} \Gamma_{*,k} - \text{Id}_{\theta^*}\Vert_{\text{op}} \leq \prod^{s-1}_{n=k}( 1 + \underbrace{\Vert\Gamma_{n+1,*}\mathcal{T}^{-1}_{n+1,n} \Gamma_{*,n} - \text{Id} \Vert_{\text{op}}}_{:=B_n})-1.
\end{equation}
From \Cref{lemma:triangle-transport} we have $B_n = O(a_n)$. Therefore, for any $\delta > 0$:
\begin{equation}
    \sup_{s\geq k} \Vert \Gamma_{s,*} \mathcal{T}^{-1}_{[s,k]}\Gamma_{*,k} - \text{Id}_{\theta^*}\Vert_{\text{op}}\leq \prod^\infty_{s=k}(1+B_s)-1 =O\left(\frac{\log{k}^{1+\delta}}{k^{3\alpha/2-1}}\right),
\end{equation}
which concludes the proof.
\end{proof}

\newpage
\subsection{KL and Fisher Information on a Manifold} \label{appendix:kl-fisher-manifold}
Let $q_\theta(y)$ be a density on $\mathbb{R}^d$ parametrized by $\theta \in \mathcal{M}$, where $\mathcal{M}$ is a Riemannian manifold.
\[F_\theta[u,v]=\E_{Y\sim q_\theta}\Big[\langle\nabla\log q_\theta(Y),u\rangle_\theta\cdot\langle\nabla\log q_\theta(Y),v\rangle_\theta\Big],\qquad u,v\in T_\theta\mathcal M.\]
The local expansion of the KL in terms of $F_\theta$ on $\mathcal{M}$ can be expressed as follows.

\begin{lemma}
Let $\mathcal{R}$ be a second-order retraction. Under standard regularity conditions,
\[\text{KL}\big(q_\theta \mid q_{\mathcal{R}_\theta(v)}\big)=\frac12 F_\theta[v,v]+o(\|v\|_\theta^2), \qquad (\theta,v) \in T\mathcal{M}.\]
\end{lemma}
\begin{proof}
    Since $\mathcal{R}$ is a second-order retraction, we have the following Taylor expansion\footnote{See e.g. Proposition 5.44 in \citet{boumal2023introduction}.}
    \begin{equation*}
        \log{q_{\mathcal{R}_\theta(v)}(y)} = \log{q_{\theta}(y)}+\langle \nabla \log q_\theta (y), v\rangle_\theta + \frac{1}{2}\big\langle \mathrm{Hess}(\log q_\theta (y))[v], v \big\rangle_\theta + o(\Vert v \Vert_\theta^2),
    \end{equation*}
    in terms of the Riemannian Hessian at $\theta$. Re-arranging and taking expectations
    \begin{equation*}
        \mathbb{E}_{Y\sim q_\theta}\left[\log \frac{q_{\mathcal{R}_\theta(v)}(Y)}{q_{\theta}(Y)}\right] = \mathbb{E}_{Y\sim q_\theta}[ \langle \nabla \log q_\theta(Y), v\rangle_\theta] + \tfrac{1}{2}\mathbb{E}_{Y\sim q_\theta}\big[ \big\langle \mathrm{Hess}(\log q_\theta (Y))[v], v \big\rangle_\theta\big] + o(\Vert v \Vert^2_\theta).
    \end{equation*}
    The score term vanishes, provided we can interchange differentiation and integration. For the Hessian term, from Theorem  1 in \citet{smith2005covariance} we have
    \begin{equation}
        \mathbb{E}_{Y\sim q_\theta}\big[ \big\langle \mathrm{Hess}(\log q_\theta (Y))[v], v \big\rangle_\theta\big] = -F_\theta[v,v].
    \end{equation}
    Thus
    \begin{equation}
        \mathrm{KL}(q_\theta \mid q_{\mathcal{R}_\theta(v)}) = -\mathbb{E}_{Y\sim q_\theta}\left[\log \frac{q_{\mathcal{R}_\theta(v)}(Y)}{q_{\theta}(Y)}\right] = \frac{1}{2}F_\theta[v,v]+o(\Vert v\Vert^2_\theta).
    \end{equation}
\end{proof}

\newpage
\subsection{ELBO Score-function and Reparameterization Gradients}\label{appendix:score and reparameterization gradient}

% Application of the proposed natural gradient method in VB requires computing the Riemannian gradient of the lower bound on $(\mathcal{M},\langle
% \cdot,\cdot\rangle_\theta)$
% \[\text{LB}(\theta)=\E_{Y\sim q_\theta}\big[\log\frac{\bar\pi(Y)}{q_\theta(Y)}\big]=\E_{Y\sim q_\theta}[h_\theta(Y)],\;\;\;h_\theta(y)=\log\frac{\bar\pi(y)}{q_\theta(y)}\]
% where $\pi(y)\propto\bar\pi(y)$ is the target density.

% In the case $\mathcal{M}$ is embedded in an Euclidean space, the gradient $\nabla_\theta\text{LB}(\theta)$ can be obtained by
% projecting its Euclidean version on $T_\theta\mathcal{M}$, where the Euclidean gradient can be in the form of a score-function formulation or reparameterization formulation.
% If the score-function gradient is used, control variates are often needed to reduce the variance of the gradient estimator.

% For a general $\mathcal{M}$ not necessarily embedded in an Euclidean space, the lemma below presents expressions for score-function and reparameterized Riemannian gradients.

Consider again a family of densities $q_\theta(y)$ on $\mathbb{R}^d$ parameterized by $\theta \in \mathcal{M}$, where $\mathcal{M}$ is a Riemannian manifold. Recall, the evidence lower bound (ELBO) is given by
\begin{equation}
    \text{LB}(\theta)=\E_{Y\sim q_\theta}\big[\log\frac{\bar\pi(Y)}{q_\theta(Y)}\big]=\E_{Y\sim q_\theta}[h_\theta(Y)],\qquad h_\theta(y)=\log\frac{\bar\pi(y)}{q_\theta(y)}
\end{equation}
where $\pi(y) \propto \bar{\pi}(y)$ denote a target density. 

Depending on the approximating family, $\nabla_\theta \mathrm{LB}(\theta)$ can often be expressed in several forms; we review the score function and reparameterization gradients here for completeness.

\begin{lemma} \label{lemma:score-gradient}
The score-function gradient of $\mathrm{LB}(\theta)$ can be expressed as
\begin{equation}
    \nabla_\theta \mathrm{LB}(\theta) = \mathbb{E}_{Y \sim q_\theta}\left[ \nabla_\theta \log q_\theta(Y) \times h_\theta (Y)\right]
\end{equation}
\end{lemma}
\begin{proof}
    Let $v \in T_\theta \mathcal{M}$ and $c:[0,1] \rightarrow \mathcal{M}$ be a smooth curve with $c(0)=\theta, c'(0)=v$.
    \begin{equation}
        D\mathrm{LB}(\theta)[v] =\int \frac{\partial}{\partial t} q_{c(t)}(y)\Big\vert_{t=0} \cdot h_\theta(y) dy+ \int q_\theta(y) \cdot \frac{\partial}{\partial t} h_{c(t)}(y)\Big\vert_{t=0}dy
    \end{equation}
    The second integrand evaluates to $-\langle\nabla_\theta q_\theta(y), v\rangle_\theta$, hence this term is zero. Thus
    \begin{equation}
        D\mathrm{LB}(\theta)[v]=\int \langle \nabla_\theta q_\theta (y), v\rangle_\theta \cdot h_\theta (y)dy  = \langle \underbrace{\mathbb{E}_{Y \sim q_\theta}\left[\nabla_\theta \log q_\theta (y) h_\theta(Y) \right]}_{:=\nabla \mathrm{LB}(\theta)}, v\rangle_\theta
    \end{equation}
\end{proof}

The Monte-Carlo estimator of the score function gradient can exhibit considerable variance. Techniques such as control variates are often required to obtain a useful update direction \citep{ranganath2014black}. A popular alternative is to employ the reparameterization trick \citep{kingma2013auto}, which often results in a more stable Monte-Carlo estimator.

\begin{lemma}
    Let $\zeta : \mathcal{M} \times \mathbb{R}^m \rightarrow \mathbb{R}^d$ such that $Y = \zeta(\theta, \epsilon)\sim q_\theta(\cdot)$, where $\epsilon \sim p(\cdot)$ is some distribution independent of $\theta$. The gradient of the lower bound can be expressed as
    \begin{equation}
        \nabla_\theta \mathrm{LB}(\theta) = \mathbb{E}_{p(\epsilon)}\big\{ D\zeta_\epsilon(\theta)^*[\nabla_yh_\theta(\zeta_\epsilon(\theta))]\big\},
    \end{equation}
    where $D\zeta_\epsilon(\theta)^*: \mathbb{R}^d \rightarrow T_\theta \mathcal{M}$ denotes the adjoint of the differential w.r.t $\theta$ for $\epsilon$ fixed.
\end{lemma}
\begin{proof}
    Let $c(t)$ be as in \Cref{lemma:score-gradient}, then applying the chain rule 
    \begin{align}
        D\mathrm{LB}(\theta)[v]& = \frac{\partial}{\partial t}\mathbb{E}_{p(\epsilon)}\{h_{c(t)}[\zeta(c(t), \epsilon)]\}\big|_{t=0} \\
        & = \mathbb{E}_{p(\epsilon)}\left\{ \frac{\partial}{\partial t} h_{c(t)}[\zeta(\theta, \epsilon)]\big\vert_{t=0}+ \frac{\partial}{\partial t} h_\theta[\zeta(c(t),\epsilon)]\big\vert_{t=0}\right\}
    \end{align}
    Defining $y^*(\epsilon) = \zeta(\theta, \epsilon)$ to disambiguate the gradient operator, the first term equals
    \begin{equation}
        \frac{\partial}{\partial t} h_{c(t)}[\zeta(\theta, \epsilon)]\big\vert_{t=0} = \langle\nabla_\theta h_\theta[y^*(\epsilon)], v\rangle_\theta = -\langle \nabla_\theta \log q_\theta(y^*(\epsilon)), v\rangle_\theta
    \end{equation}
    Taking the expectation of the above in $p(\epsilon)$ yields zero. For the second term, we have
    \begin{equation}
        \frac{\partial}{\partial t} h_\theta[\zeta(c(t),\epsilon)]\big\vert_{t=0} = \langle \nabla_y h_\theta[\zeta(\theta,\epsilon)], D\zeta_\epsilon(\theta)[v]\rangle  = \langle D\zeta_\epsilon(\theta)^*[\nabla_y h_\theta[\zeta(\theta,\epsilon)]], v\rangle_\theta
    \end{equation}
    Taking expectations again with respect to $p(\epsilon)$ yields the result.
\end{proof}

\newpage
\subsection{Low-dimensional Representation of $H_s^{-1}$} \label{appendix:vectorized-fisher}

The purpose of this section is to describe how one can store, update, and transport a vectorized ``sliding window'' version of $H_s^{-1}$, comprised of the most recent $K$ score vectors. 

Let $\mathcal{M}$ be a Riemannian manifold, and $\theta \in \mathcal{M}$, $u_s, v_s \in T_\theta \mathcal{M}$ for $0 \leq s \leq K$. Define
\begin{equation}
    H_s := \underbrace{\epsilon I}_{:=H_0} + \sum^s_{k=1} u_k v_k^\top  G_\theta, \qquad 1 \leq s \leq K.
\end{equation}
We consider a slightly more general setting than described in \Cref{rem:limited-memory-approximation}. Rather than defining $H_s$ via a summation of terms $\phi_k \phi_k^\top G_\theta$, we have used $u_k v_k^\top G_\theta$. The purpose of this change is to accommodate the use of general (i.e, non-isometric) transports in \Cref{sec:transporation-of-hs1}.

\subsubsection{Vectorized Representation of $H_s^{-1}$}
Following the inversion formula \eqref{eq:sherman-morrison}, for $0 \leq s \leq K-1$ we have
\begin{align}
    H_{s+1}^{-1} & = H_s^{-1}-(1+v_{s+1}^\top G_\theta H_s^{-1} u_{s+1})^{-1} H_s^{-1} u_{s+1} v_{s+1}^\top G_\theta H^{-1}_s \\
    & = H_s^{-1} - (1+\langle v_{s+1}, H_s^{-1}u_{s+1}\rangle_\theta )^{-1} (H^{-1}_su_{s+1})(H^{-*}_sv_{s+1})^\top G_\theta.
\end{align}
For $0 \leq s \leq K-1$ define
\begin{equation}
    \mu_s := H_s^{-1}u_{s+1}, \qquad \nu_{s} :=H_s^{-*}v_{s+1}, \qquad c_s := (1 + \langle v_{s+1}, \mu_s\rangle_\theta)^{-1}.
\end{equation}
Expanding the previous recursion, we therefore have
\begin{equation} \label{eq:inv-H-sum}
    H_{s+1}^{-1} := \tfrac{1}{\epsilon}I - \sum^s_{k=0}c_k \mu_k \nu_k^\top G_\theta, \qquad 0 \leq s \leq K-1.
\end{equation}
\vspace{-0.5cm}
For the adjoint of $H_s$, we just swap the roles of $\mu$ and $\nu$.
\begin{align}
    H_{s+1}^{-*} & = (H_s+u_{s+1}v_{s+1}^\top G_\theta )^{-*} = (H_s^*+ v_{s+1}u_{s+1}^\top G_\theta)^{-1} \\
    & = H_s^{-*} - (1 + \langle v_{s+1}, H_s^{-1} u_{s+1}\rangle_\theta)^{-1} (H_s^{-*}v_{s+1})(H_s^{-1} u_{s+1})^\top G_\theta \\
    & = H_s^{-*} - c_s \nu_s \mu_s^\top G_\theta = \tfrac{1}{\epsilon} I - \sum^s_{k=0}c_k \nu_k \mu_k^\top G_\theta.
\end{align}

\subsubsection{Updating Vectorized Representation of $H_K^{-1}$}
In the following, we index score vectors from newest to oldest, as this simplifies the presentation. 
We want an efficient mechanism to compute $(H_{K-1}^0)^{-1}$ from $H_K^{-1}$, where
\begin{align}
    H_{K-1}^0&:= H_0+u_0v_0^\top G_\theta+u_1v_1^\top G_\theta +\cdots u_{K-1}v_{K-1}^\top G_\theta.
\end{align}
% from $H_K^{-1}$,
% \begin{align}
%     H_K= H_0+u_1v_1^\top G_\theta+\cdots u_{K-1}v_{K-1}^\top G_\theta+u_{K}v_{K}^\top G_\theta.
% \end{align}
That is, we want to drop the oldest vectors $u_K$, $v_K$ and incorporate the new vectors $u_0$, $v_0$.
To drop the oldest vectors, one simply removes the last term in the summation \eqref{eq:inv-H-sum}. Incorporating the new term is slightly trickier, as we need to add it at the head of this summation, which then requires updating the subsequent $c_k, \mu_k, \nu_k$. Define
\begin{equation}
    H^{0}_{-1} := \epsilon I, \qquad H_s^0 := H_s + u_0 v_0^\top G_\theta \quad \mathrm{for} \quad 0 \leq s \leq K-1.
\end{equation}
We wish to find the $\tilde{c}_s, \tilde\mu_s$, and $\tilde \nu_s$ defining $(H_{K-1}^0)^{-1}$ as in \eqref{eq:inv-H-sum}
\begin{equation}
    (H_{K-1}^0)^{-1} = \tfrac{1}{\epsilon}I - \tilde c_{-1} \tilde\mu_{-1} \tilde \nu_{-1}^\top G_\theta - \sum^{K-2}_{s=0} \tilde c_s \tilde \mu_s \tilde \nu_s^\top G_\theta .
\end{equation}
That is, where for $-1 \leq s \leq K-2$ we have
\begin{equation}
    \tilde \mu_s = (H^0_s)^{-1} u_{s+1}, \qquad \tilde{\nu}_s = (H^0_s)^{-*} v_{s+1}, \qquad \tilde{c}_s = (1 + \langle v_{s+1}, \tilde\mu_s\rangle_\theta)^{-1}.
\end{equation}
To that end, let $z_0 := u_0 /\epsilon$ and $z_0^*:=v_0/\epsilon$, and for $1 \leq s \leq K-1$ define
\begin{align}
    & z_s := H_s^{-1} u_0 = (H^{-1}_{s-1} - c_{s-1} \mu_{s-1} \nu_{s-1}^\top G_\theta)u_0 = z_{s-1} - c_{s-1} \mu_{s-1} \langle \nu_{s-1}, u_0 \rangle_\theta \\
    & z_s^* := H_s^{-*} v_0 = (H^{-*}_{s-1} - c_{s-1} \nu_{s-1} \mu_{s-1}^\top G_\theta)v_0 = z_{s-1}^* - c_{s-1} \nu_{s-1} \langle \mu_{s-1}, v_0 \rangle_\theta.
\end{align}
Then, for the $\tilde\mu_s$ with $0 \leq s \leq K-2$ do
\begin{align}
    \tilde\mu_s &= (H^0_s)^{-1} u_{s+1} = (H_s + u_0 v_0^\top G_\theta)^{-1}{u_{s+1}} 
    \\ & = H_s^{-1}u_{s+1} - (1 + \langle v_0, H_s^{-1}u_0\rangle_\theta)^{-1} H_s^{-1} u_0 v_0^\top G_\theta H_s^{-1} u_{s+1}
    \\ & = \mu_s - (1 + \langle v_0, z_s\rangle_\theta)^{-1} \langle v_0, \mu_s \rangle_\theta \cdot z_s.
\end{align}
By similar reasoning, for the $\tilde\nu_s$ we have
\begin{equation}
    \tilde \nu _s = (H_s^0)^{-*} v_{s+1} = \nu_s - (1 + \langle u_0, z_s^*\rangle_\theta)^{-1} \langle u_0, \nu_s\rangle_\theta \cdot z_s^*.
\end{equation}
Note that $\langle u_0, z_s^*\rangle_\theta = \langle v_0, z_s\rangle_\theta$. Finally, to obtain the $\tilde{c}_s$ for $0 \leq s \leq K-2$
\begin{align}
    \tilde c^{-1}_s -1 & = \langle v_{s+1}, \tilde \mu_s\rangle_\theta =  v_{s+1}^\top G_\theta (H_s^0)^{-1} u_{s+1}
    \\ & = v_{s+1}^\top G_\theta H_s^{-1} u_{s+1} - (1 + \langle v_0, H_s^{-1} u_0 \rangle_\theta)^{-1} v_{s+1}^\top G_\theta H_s^{-1} u_0 v_0^\top G_\theta H_s^{-1} u_{s+1}
    \\ & = c^{-1}_s - 1 - (1+\langle v_0, z_s\rangle_\theta)^{-1} \langle u_0, \nu_s\rangle_\theta \cdot \langle v_0, \mu_s\rangle_\theta.
\end{align}
The full procedure for recomputing $\tilde{c}_k, \tilde{\mu}_k, \tilde{\nu}_k$ thus requires $O(K)$ evaluations of the metric.

\subsubsection{Transportation of $H_s^{-1}$} \label{sec:transporation-of-hs1}
Let $\tilde \theta \in \mathcal{M}$, and $\mathcal{T} : T_\theta \mathcal{M} \rightarrow T_{\tilde{\theta}} \mathcal{M}$ an invertible linear map. 
We want to transport $H_s^{-1}$ from $T_\theta \mathcal{M}$ to $T_{\tilde{\theta}} \mathcal{M}$.
For $0 \leq s \leq K-1$ define
\begin{equation}
    \tilde H_{s+1} := \mathcal{T} \circ H_{s+1} \circ \mathcal{T}^{-1} = \tilde H_s + (\mathcal{T}u_{s+1})(\mathcal{T}^{-*}v_{s+1})^\top G_{\tilde\theta}.
\end{equation}
The transported inverse is then
\begin{equation}
    \tilde{H}^{-1}_{s+1} = \tilde{H}_s^{-1} - (1+\langle \mathcal{T}^{-*} v_{s+1}, \tilde H_s^{-1} (\mathcal{T}u_{s+1})\rangle_{\tilde \theta})^{-1} (\tilde H^{-1}_s \mathcal{T}u_{s+1})(\tilde H_s^{-*} \mathcal{T}^{-*}v_{s+1})^\top G_{\tilde\theta}.
\end{equation}
Note that
\begin{equation}
 \langle \mathcal{T}^{-*} v_{s+1}, \tilde H_s^{-1} (\mathcal{T}u_{s+1})\rangle_{\tilde \theta} =  \langle v_{s+1}, H_s^{-1}u_{s+1}\rangle_\theta = \langle v_{s+1},\mu_s\rangle_\theta.   
\end{equation}
Furthermore
\begin{align}
    & \tilde H^{-1}_s \mathcal{T} u_{s+1} = \mathcal{T} (H^{-1}_s u_{s+1})=\mathcal{T} \mu_s, \\
    & \tilde H^{-*}_s \mathcal{T}^{-*} v_{s+1} = \mathcal{T}^{-*} (H_s^{-*}v_{s+1}) = \mathcal{T}^{-*}\nu_s.
\end{align}
Thus
\begin{equation}
    \tilde{H}^{-1}_{s+1} = \tilde{H}^{-1}_s - c_s (\mathcal{T}\mu_s)(\mathcal{T}^{-*}\nu_s)^\top G_{\tilde\theta}.
\end{equation}
Hence $(c_s, \mathcal{T}\mu_s, \mathcal{T}^{-*}\nu_s)$ maintains the structure in \eqref{eq:inv-H-sum} relative to $(c_s,\mu_s, \nu_s)$. Note that if $\mathcal{T}$ is isometric then $\mathcal{T}^{-*}=\mathcal{T}$. Hence, provided that $u=v$ for every outer product term, then we only need to transport and update one set of vectors.

\newpage
\subsection{Gaussian Variational Inference} \label{appx:bw-nat-grad}

We review the relevant details on the Fisher metric and Bures-Wasserstein manifold which are needed to implement \Cref{alg:riemannian-if-ngd}. There are many references on the latter; see e.g. \citet{takatsu2011wasserstein, malago2018wasserstein, bhatia2019}. Then, we provide some additional details on our experimental methodology.

\subsubsection{The Fisher and Bures-Wasserstein Metrics}

The space of non-degenerate Gaussian distributions on $\mathbb{R}^d$ can be viewed as a smooth manifold $\mathcal{M}$, parametrized by the mean and covariance. For a smooth curve $(m(t), \Sigma(t)) \in \mathbb{R}^d \times S^d_{++}$, the velocity $(\dot m(t), \dot \Sigma(t))$ lies in $\mathbb{R}^d \times S^d$, where $S^d$ denotes the space of $d\times d$ symmetric matrices. Thus, tangent spaces are identified as
\begin{equation}
    T_{m,\Sigma}\mathcal{M} \equiv \mathbb{R}^d \times S^d.
\end{equation}

\paragraph{Fisher Metric:} The Fisher information defines\footnote{See for example \citet{skovgaard1984riemannian}.} the following Riemannian metric on $\mathcal{M}$
\begin{equation}
    \langle (u,U), (v,V)\rangle^F_{m,\Sigma} := u^\top \Sigma^{-1} v +\frac{1}{2}\tr\left(\Sigma^{-1} U \Sigma^{-1} V\right), \qquad (u,U), (v,V) \in \mathbb{R}^d \times S^d.
\end{equation}
For a smooth function $f: \mathcal{M} \rightarrow \mathbb{R}$, let $\nabla_\mu f, \nabla_\Sigma f$ denote its Euclidean (partial) gradients. To obtain the gradient with respect to the Fisher metric (i.e. the \textit{natural gradients}), consider
\begin{align}
    \nabla_\mu f^\top v+\tr(\nabla_\Sigma f \cdot V) & = (\Sigma\nabla_\mu f)^\top \Sigma^{-1} v+ \frac{1}{2}\tr(\Sigma^{-1}(2\Sigma \nabla_\Sigma f \Sigma) \Sigma^{-1} V) \\
    & =\langle (\Sigma \nabla_\mu f, 2\Sigma \nabla_\Sigma f \Sigma), (v,V)\rangle^F_{m,\Sigma}.
\end{align}
where $(v,V) \in \mathbb{R}^d \times S^d$. Thus we have
\begin{equation} \label{eq:gauss-nat-grad}
    \nabla_\mu^{\text{nat}}f := \Sigma \nabla_\mu f, \qquad \nabla^{\text{nat}}_\Sigma f := 2\Sigma \nabla_\Sigma f \Sigma.
\end{equation}

\paragraph{Bures-Wasserstein Metric:} The 2-Wasserstein distance on $\mathcal{M}$ has a closed expression
\begin{equation} \label{eq:w2-distance}
    W^2_2(\mathcal{N}(m_1, \Sigma_1), \mathcal{N}(m_2, \Sigma_2)) = \Vert m_1 - m_2\Vert^2_2 + \tr\left(\Sigma_1 + \Sigma_2 -2(\Sigma_1^{1/2}\Sigma_2\Sigma_1^{1/2})^{1/2}\right).
\end{equation}
This distance arises as the unique Riemannian distance associated with the \textit{Bures-Wasserstein} metric. For $(u,U), (v,V) \in \mathbb{R}^d \times S^d$, let $L_\Sigma(U)$ denote the unique symmetric solution of the Lyapunov equation $\Sigma X + X\Sigma = U$. Then
\vspace{-0.3cm}
\begin{equation}
    \langle(u,U), (v,V)\rangle^{\mathrm{BW}}_{m,\Sigma} := u^\top v + \text{tr}(L_\Sigma(U) \Sigma L_\Sigma(V)).
\end{equation}
The corresponding Riemannian manifold is often referred to as the \textit{Bures-Wasserstein space}\footnote{We note that various authors also use the term Bures-Wasserstein space/metric/distance to refer exclusively to a structure on the set of SPD matrices. Indeed, one can observe that the mean and covariance components of the metric decouple, and it is only the latter which is non-Euclidean.} $\mathrm{BW}(\mathbb{R}^d)$. The Riemannian gradient with respect to the BW metric is obtained as follows
\begin{align}
    \nabla_\mu f^\top v+\tr(\nabla_\Sigma f \cdot V) & = \nabla_\mu f^\top v+2\tr(\nabla_\Sigma f \Sigma L_\Sigma (V)) \\ & = \langle (\nabla_\mu f, 2 L_\Sigma^{-1} (\nabla_\Sigma f)), (v,V)\rangle_{\mu,\Sigma}^{\mathrm{BW}}.
\end{align}
Thus we have $\nabla_\mu^\mathrm{BW}f = \nabla_\mu f$, and for the covariance
\begin{equation}
    \nabla^{\mathrm{BW}}_\Sigma f = 2L_\Sigma^{-1}(\nabla_\Sigma f) = 2(\Sigma \nabla_\Sigma f + \nabla_\Sigma f \Sigma).
\end{equation}
The exponential map is given below, and is well-defined for $L_\Sigma(U) \succ -I$
\begin{equation} \label{eq:bw-exp-map}
    \exp_{m,\Sigma}(u,U) :=(m+u, (I+L_\Sigma(U))\Sigma(I+L_\Sigma(U))).
\end{equation}
\paragraph{Alternative Parametrisation:} It can be convenient to parameterize the covariance component of the tangent space in terms of $X = L_\Sigma(U)$ rather than $U$ itself.
In these coordinates, the covariance component of the BW metric reduces to $\langle X_1, X_2\rangle_\Sigma = \tr(X_1 \Sigma X_2)$, the exponential map to $\exp_\Sigma(X) = (I+X)\Sigma (I+X)$, while the BW gradient of $f$ is simply $2\nabla_\Sigma f$. This parametrisation is adopted in some of the optimal transport literature \citep{lambert2022variational, diao2023forward}; we proceed with this convention as it simplifies our presentation.

\subsubsection{Differential of the Exponential Map on $\mathrm{BW}(\mathbb{R}^d)$} \label{appx:bw-diff-exp}
Following \eqref{eq:bw-exp-map}, the exp-map can be written as $\exp_{m,\Sigma}(u,X)=(m+u,\exp_{\Sigma}(X))$. Then
\begin{equation}
    D\exp_{m,\Sigma}((u,X))[(v,Y)] = (v,D\exp_\Sigma(X)[Y]) \in \mathbb{R}^d \times S^d.
\end{equation}
The differential of the covariance component of the exponential map is given below\footnote{This is also provided in \citet{malago2018wasserstein}, we include it here for completeness.}:
\begin{lemma}
    For $\Sigma \in S^{d}_{++}$ and $X,Y \in S^d$, such that $X \succ -I$ we have
    \begin{equation}
        D\exp_{\Sigma}(X)[Y] = L_{\exp_\Sigma(X)}\left((X+I)\Sigma Y+Y\Sigma(X+I)\right).
    \end{equation}
\end{lemma}
\begin{proof}Consider the curve $\Sigma(t) = \exp_{\Sigma} (X+tY)$. Its velocity at $t=0$ is
   \begin{equation}
       \Sigma'(0)= \frac{d}{dt}(I+X+tY)\Sigma(I+X+tY)\vert_{t=0} = (I+X)\Sigma Y+Y\Sigma (I+X).
   \end{equation}
   To get $D\exp_{\Sigma}(X)[Y]$, we just need to convert parametrizations by applying $L_{\exp_\Sigma(X)}$.
\end{proof}
For $A, B \in S^d$ the \textit{matrix geometric mean} \citep{bhatia2009positive} is given by
\begin{equation}
    A \#B = A^{1/2}(A^{-1/2} B A^{-1/2})^{1/2} A^{1/2}.
\end{equation}
The inverse of the exponential map (logarithmic map) is then
\begin{equation}
    \log_{m_1,\Sigma_1}(m_2,\Sigma_2) := (m_2-m_1, \Sigma_1^{-1} \# \Sigma_2-I).
\end{equation}
Letting $\log_{\Sigma_1}(\Sigma_2) = \Sigma_1^{-1} \# \Sigma_2-I$, we therefore have
\begin{equation}
    \mathcal{T}_{\Sigma_1, \Sigma_2}(X) := D\exp_{\Sigma_1}(\log_{\Sigma_1}(\Sigma_2))[X] = L_{\Sigma_2} \left[(\Sigma_1^{-1} \# \Sigma_2)\Sigma_1X + X\Sigma_1(\Sigma_1^{-1} \# \Sigma_2) \right].
\end{equation}

\subsubsection{Algorithm 1 - Representing, Updating, and Transporting $H_{s}^{-1}$} \label{appx:transporting-fisher}

Let $\ell(x)$ denote the log-likelihood of $\mathcal{N}(\mu,\Sigma)$. The two components of the score are then
\begin{equation}
    \nabla_\mu \ell(x) = \Sigma^{-1}(x-\mu), \qquad \nabla_\Sigma \ell(x) = \frac{1}{2}\Sigma^{-1}(x-\mu)(x-\mu)^\top\Sigma^{-1} - \frac{1}{2}\Sigma^{-1}.
\end{equation}
The approximate Fisher information $H_s$ is decomposed into parts which act separately on the mean and covariance parameters; we denote these $H^\mu_s \in \mathbb{R}^{d\times d}$ and $H^\Sigma_s \in \mathbb{R}^{d^2\times d^2}$ respectively.

\paragraph{Score Update (Mean):} The score vector update for this component is
\begin{equation}
    (H^\mu_{s+1})^{-1} = (H^\mu_{s})^{-1} - (1+(\phi_{s+1}^\mu)^\top (H_s^\mu)^{-1} \phi^\mu_{s+1})^{-1} \times (H^\mu_{s})^{-1} \phi^\mu_{s+1} (\phi^\mu_{s+1})^\top (H^\mu_{s})^{-1},
\end{equation}
where $\phi^\mu_{s+1}= \nabla_\mu \ell(\bar y_{s+1})$ with $\bar y_{s+1} \sim \mathcal{N}(\mu^{(s)}, \Sigma^{(s)})$. The update is identical in the Euclidean and $\mathrm{BW}(\mathbb{R}^d)$ versions of the algorithm. The latter does not require a transport operation since the mean component of the tangent space is Euclidean.

\paragraph{Score Update (Covariance):} The covariance component $H_s^\Sigma$ is a $\mathbb{R}^{d^2 \times d^2}$ matrix acting on the space of \textit{vectorized matrices}\footnote{We favor this over a half-vectorized representation as it simplifies the transport operation. This example is primarily demonstrative, hence we favor ease of implementation over modest efficiency gains.}. Here, $\mathrm{vec}$ denotes an operation which stacks columns $\mathrm{vec}(\Sigma) = [\Sigma_{1,1}, ..., \Sigma_{d,1}, \Sigma_{1,2}, ...,\Sigma_{d,2}, ...]^\intercal$ for $\Sigma \in \mathbb{R}^{d\times d}$. The vec-space score update is then
% \begin{equation}
%     \mathrm{vec}(\Sigma) = [\Sigma_{1,1}, ..., \Sigma_{d,1}, \Sigma_{1,2}, ...,\Sigma_{d,2}, ...]^\intercal, \qquad \Sigma \in \mathbb{R}^{d\times d}.
% \end{equation}
% The vec-space score update is
\begin{equation}
    (H^\Sigma)_{s+1}^{-1} = (H^\Sigma)_{s+1/2}^{-1} - c_s \times  (H^\Sigma)_{s+1/2}^{-1}\mathrm{vec}(\phi^\Sigma_{s+1})\mathrm{vec}(\tilde{\phi}^\Sigma_{s+1})^{\top}(H^\Sigma)_{s+1/2}^{-1},
\end{equation}
where $\phi^\Sigma_{s+1} := \nabla_\Sigma \ell(\bar y_{s+1})$ for $\bar y_{s+1} \sim \mathcal{N}(\mu^{(s)},\Sigma^{(s)})$, and

\begin{equation}
    c_s=(1+\mathrm{vec}(\tilde{\phi}^\Sigma_{s+1})^\top (H^\Sigma)_{s+1/2}^{-1} \mathrm{vec}(\phi^\Sigma_{s+1}))^{-1}.
\end{equation}
In the Euclidean setting $\tilde{\phi}^\Sigma_{s+1} = \phi^\Sigma_{s+1}$. For $\mathrm{BW}(\mathbb{R}^d)$, recall that the metric on the covariance tangent space is $\langle X,Y\rangle_\Sigma = \tr(X \Sigma Y)$. In vec-space this takes the form
\begin{equation}
    \langle X, Y\rangle_\Sigma = \mathrm{vec}(X)^\top \underbrace{\tfrac{1}{2}(I \otimes \Sigma + \Sigma \otimes I)}_{=G^\Sigma_{\mathrm{vec}}}\mathrm{vec}(Y), \quad X,Y \in S^d.
\end{equation}
% Note that $\mathrm{vec}(X)^\top (I \otimes \Sigma) \mathrm{vec}(Y) = \tr (X \Sigma Y)$. The symmetric form of $G^\Sigma_{\mathrm{vec}}$ is preferred as it maps $\mathrm{vec}(S^d)$ into $\mathrm{vec}(S^d)$. 
The modified score vector is then
\begin{equation}
    \mathrm{vec}(\tilde{\phi}^\Sigma_{s+1}) :=G^{\Sigma^{(s)}}_{\mathrm{vec}} \mathrm{vec}(\phi^\Sigma_{s+1}) = \frac{1}{2}\mathrm{vec}(\Sigma^{(s)}\phi_{s+1}^\Sigma+\phi_{s+1}^\Sigma \Sigma^{(s)}).
\end{equation}

% \paragraph{Eigenvalue Normalization (Covariance):} Since the score vectors $\phi^\Sigma$ are symmetric, the unscaled matrix $(H^\Sigma_s)^{-1}$ receives no rank-1 updates on the antisymmetric subspace $\mathrm{vec}(S^d)^\perp$. Consequently, the scaled approximation $(\mathbf{H}_s^\Sigma)^{-1} = (s+1)(H^\Sigma_s)^{-1}$ can grow indefinitely on this subspace, which can cause numerical instability. This can be addressed by occasionally projecting onto $\mathrm{vec}(S^d)$\footnote{That is, where $\mathrm{vec}(S^d):=\{ \mathrm{vec}(A) : A \in S^d\}$, and similarly for $\mathrm{vech}(S^d)$.} via the duplication matrix\footnote{For more information on $D_d, D_d^+, K_{d,d}$ etc. see chapter 3 of \citet{magnus2019matrix}} $D_d$ and its Moore-Penrose inverse
% \begin{equation}
%     D_d : \text{vech}(S^d) \rightarrow \text{vec}(S^d), \qquad D_d^+ : \text{vec}(S^d) \rightarrow \text{vech}(S^d),
% \end{equation}
% where $D_d^+D_d=I$ and $D_dD_d^+=\tfrac{1}{2}(I+K_{d,d})$. Specifically, we occasionally do $H_{s}^{-1} \rightarrow D_dD_d^+H_s^{-1} D_d^+D_d = \tfrac{1}{2}(I+K_{d,d})H_s^{-1}$ to normalize its eigenvalues on $\text{vec}(S^d)^\perp$. Note $K_{d,d}$ is very sparse, so this operation can be performed at low cost after every score update.

\paragraph{Transportation:} In vec-space, the differentiated exponential map becomes
\footnotesize
\begin{align}
    \tilde{\mathcal{T}}_{\Sigma_1, \Sigma_2}& :=\text{vec} \ \circ \ D\exp_{\Sigma_1} (\log_{\Sigma_1}(\Sigma_2)) \circ \text{vec}^{-1}  \\ & = (P_2\otimes P_2)(\Lambda_2 \otimes I + I\otimes \Lambda_2)^{-1}(P_2\otimes P_2)^\top ((\Sigma_1^{-1} \# \Sigma_2) \Sigma_1 \otimes I + I\otimes (\Sigma_1^{-1} \# \Sigma_2)\Sigma_1).
\end{align}
\normalsize
where $\Sigma_2 = P_2\Lambda_2 P_2^\top$ is the eigendecomposition. The transport operation is then
\begin{equation}
    (H^\Sigma_{s+1/2})^{-1} := \tilde{\mathcal{T}}_{\Sigma^{(s-1)},\Sigma^{(s)}} (H^\Sigma_{s})^{-1} \tilde{\mathcal{T}}_{\Sigma^{(s)},\Sigma^{(s-1)}}.
\end{equation}
Since $\tilde{\mathcal{T}}$ has a Kronecker product structure, the above operation has $O(d^5)$ complexity. This can be substantially accelerated on a GPU, e.g. via \texttt{torch.einsum} operations in PyTorch. Furthermore, one does not need to materialize the full $d^2 \times d^2$ matrix representation of $\tilde{\mathcal{T}}$. In practice, we found the transport operation to be a small constant factor ($2-5\times$) slower than the $O(d^4)$ score vector update across the examples considered.

\subsubsection{Experiment Details - VI (Logistic Regression)} \label{appx:vi-logistic-regression}

Let $\pi =\exp(-V)/Z$ be a probability density where $V: \mathbb{R}^d \rightarrow \mathbb{R}$, and define:
\small
\begin{equation}
    \mathcal{L}(\mu,\Sigma) = \mathrm{KL}(\mathcal{N}(\mu,\Sigma) \mid \pi).
\end{equation}
\normalsize
The partial derivatives of $\mathcal{L}$ are as follows, see e.g. appendices of \citet{rezende2014stochastic},
\small
\begin{align}
    & \nabla_\mu\mathcal{L}(\mu,\Sigma) = \mathbb{E}_{\beta \sim \mathcal{N}(\mu,\Sigma)}[\nabla V(\beta)], \label{eq:logreg-KL-gradient-mu}\\
    & \nabla_\Sigma \mathcal{L}(\mu,\Sigma) = \frac{1}{2}\mathbb{E}_{\beta \sim \mathcal{N}(\mu,\Sigma)}[\nabla^2 V(\beta)]-\frac{1}{2}\Sigma^{-1}. \label{eq:logreg-KL-gradient-sigma}
\end{align}
\normalsize
For the logistic regression model \eqref{eq:logistic-regression}, the log-posterior is
\small
\begin{equation}
    \log p(\beta \mid \{x_i,y_i\}^n_{i=1}) = \underbrace{\sum^n_{i=1} [y_i \beta^\top x_i- \log{(1 + \exp(\beta^\top x_i))}]-\frac{1}{2\sigma^2}\Vert \beta\Vert^2_2}_{:=-V(\beta)}+C.
\end{equation}
\normalsize
Letting $S(x)=1/(1+e^{-x})$, one can show that:
\small
\begin{align}
    & \nabla V(\beta) = \sum^n_{i=1}(S(\beta^\top x_i)-y_i)x_i + \frac{1}{\sigma^2}\beta, \\
    & \nabla^2 V(\beta) = \sum^n_{i=1} S(\beta^\top x_i) (1- S(\beta^\top x_i)) x_i x_i^\top + \frac{1}{\sigma^2}I.
\end{align}
\normalsize

\paragraph{Update Directions:} The update directions for each algorithm are based on the Euclidean stochastic gradients $\widehat{\nabla}_\mu \mathcal{L}$ in \eqref{eq:logreg-KL-gradient-mu}, and $\widehat{\nabla}_\Sigma \mathcal{L}$ in \eqref{eq:logreg-KL-gradient-sigma}. These are estimated with a Monte Carlo sample size of $B=100$ across all methods and datasets. The six algorithms differ in how these Euclidean gradients are transformed into update directions:
\begin{itemize}
    \item \textbf{Euc-GD:} uses $(\widehat{\nabla}_\mu \mathcal{L}, \widehat{\nabla}_\Sigma \mathcal{L})$ directly.
    \item \textbf{Euc-NGD:} applies the transformation \eqref{eq:gauss-nat-grad} giving $(\Sigma \widehat{\nabla}_\mu \mathcal{L}, 2\Sigma \widehat{\nabla}_\Sigma \mathcal{L} \Sigma)$.
    \item \textbf{Euc-NGD Approx:} follows Algorithm \ref{alg:riemannian-if-ngd} in the Euclidean setting.
    % \begin{equation}
    %     ((\mathbf{H}^\mu_s)^{-1}\widehat{\nabla}_\mu \mathcal{L}, \mathrm{vec}^{-1}[(\mathbf{H}^\Sigma_s)^{-1} \mathrm{vec}(\widehat{\nabla}_\Sigma \mathcal{L})])
    % \end{equation}
    \item \textbf{BW-GD:} uses $(\widehat{\nabla}_\mu \mathcal{L}, 2\widehat{\nabla}_\Sigma \mathcal{L})$; the constant $2$ arises from the $X$-parametrisation.
    \item \textbf{BW-NGD:} applies \eqref{eq:gauss-nat-grad}, followed by conversion to $X$-parametrisation\footnote{Note that for $S\in S^d$ we have $L_\Sigma(\Sigma S\Sigma) = L_{\Sigma^{-1}}(S)$; this is easy to show.}
    \begin{equation}
        (\Sigma \widehat{\nabla}_\mu \mathcal{L}, 2L_{\Sigma^{-1}}(\widehat{\nabla}_\Sigma \mathcal{L})).
    \end{equation}
    \item \textbf{BW-NGD Approx:} follows Algorithm \ref{alg:riemannian-if-ngd} on $\mathrm{BW}(\mathbb{R}^d)$.
    % There is an extra $2$ constant applied to the covariance component of the gradient as in BW-GD.
\end{itemize}
\paragraph{Retractions:}The Euclidean (i.e., ``Euc'') algorithms employ an additive step for both parameters; that is, given raw update directions $(\hat{g}_\mu, \hat{g}_\Sigma)$ obtained as above, we do:
\begin{equation}
    (\mu^{(s+1)}, \Sigma^{(s+1)}) = (\mu^{(s)} - \tau_{s+1} \hat{g}_\mu, \mathrm{Clip}_\eta[\Sigma^{(s)} - \tau_{s+1} \hat{g}_\Sigma]).
\end{equation}
where $\tau_{s+1}$ is the step size, and $\mathrm{Clip}_\eta$ projects eigenvalues to $[\eta,\infty)$; we let $\eta=10^{-8}$ in the experiments. For the Bures-Wasserstein algorithms, we perform

\begin{equation}
    (\mu^{(s+1)}, \Sigma^{(s+1)}) = (\mu^{(s)} - \tau_{s+1}\hat{g}_\mu, \mathrm{Clip}_\eta[(I-\tau_{s+1}\hat{g}_\Sigma) \Sigma^{(s)} (I-\tau_{s+1}\hat{g}_\Sigma)]).
\end{equation}

% \paragraph{Tuning Procedure (\Cref{fig:logreg-small})} The step size schedule follows $\tau_s=c_0/(100+s)^{\alpha}$. We tuned the initial step size $\tau_0$ (via $c_0$), $\alpha$, and the initialization parameter $\epsilon$ (for the approximate methods) over the grids $\tau_0 \in \{10^{-1}, 10^{-2}, 10^{-3}, 10^{-4}\}$, $\alpha \in \{0, 0.3, 0.5, 0.7, 0.9\}$, and $\epsilon \in \{100, 250, 500, 1000, 2500\}$. For each algorithm-dataset pair, we sampled 25 random configurations from the grid, and selected the one yielding the lowest NELBO after $2,500$ iterations (start $(\mu,\Sigma)=(0,I)$). Finally, we also evaluated the adjacent configurations.

% \paragraph{Tuning Procedure (\Cref{fig:logreg-large})} Here we used a constant step size $(\alpha=0)$, selecting $\tau_0 \in \{10^{-1}, 10^{-2}, 10^{-3}, 10^{-4}, 10^{-5}\}$, and $\epsilon \in \{1000, 2500, 5000, 10000\}$. Each configuration was evaluated, and the combination with the lowest NELBO after 500 iterations was selected.

\newpage
\subsection{Details of Stiefel Manifold Experiment} \label{appx:Steifel-manifold-experiment}

We review relevant details on the Stiefel manifold, and the implementation of our algorithms.

\subsubsection{Geometric Preliminaries}

The following standard properties of the Stiefel manifold can be found in \citet{absil2009optimization}.

\paragraph{Background:} The Stiefel manifold $\mathrm{St}(p,n)$ is an embedded submanifold of $\mathbb{R}^{n\times p}$
\begin{equation}
    \mathrm{St}(p,n) := \{ X \in \mathbb{R}^{n\times p} \mid X^\top X = I_p\}, \qquad p \leq n.
\end{equation}
The tangent space at $X \in \mathrm{St}(p,n)$ forms a subspace of $\mathbb{R}^{n\times p}$ given by
\begin{equation}
    T_X\mathrm{St}(p,n) = \{ Z \in \mathbb{R}^{n\times p} : Z^\top X + X^\top Z = 0_{p\times p}\}.
\end{equation}
The Riemannian metric on this tangent space is $\langle Y, Z \rangle^S_X = \tr(Y ^\top Z)$; i.e., the ordinary Euclidean metric. The orthogonal projection operator onto $T_X \mathrm{St}(p,n)$ is
\begin{align} \label{eq:stiefel-tangent-projection}
    \mathrm{Proj}_X{M} & = (I-XX^\top)M + X\mathrm{skew}(X^\top M), \quad M \in \mathbb{R}^{n\times p} \\
    & = M - X \mathrm{sym}(X^\top M),
\end{align}
where $\mathrm{sym}(A):=(A+A^\top)/2$ and $\mathrm{skew}(A) := (A-A^\top)/2$. For a differentiable function $f:\mathbb{R}^{n\times p} \rightarrow \mathbb{R}$, with Euclidean gradient $\nabla^\mathcal{E}_X f(X)$ at $X \in \mathrm{St}(p,n)$ its Riemannian gradient is
\begin{equation}
    \nabla_X f(X) := \mathrm{Proj}_X(\nabla^\mathcal{E}_Xf(X)).
\end{equation}

\paragraph{Retraction \& Transport:} For $X \in \mathrm{St}(p,n)$ and $U \in T_X\mathrm{St}(p,n)$, define
\begin{equation}
    W_U := P_X U X^\top - X U^\top P_X, \qquad P_X = I-\frac{1}{2}XX^\top.
\end{equation}
The Cayley transform retraction \citep{zhu2017riemannian} is
\begin{equation}
    \mathcal{R}_X(U) = \left(I-\frac{1}{2}W_U\right)^{-1}\left(I+\frac{1}{2}W_U\right)X.
\end{equation}
The associated isometric vector transport \citep[Lemma 3]{zhu2017riemannian} is
\begin{equation} \label{eq:stiefel-isometric-transport}
    \mathcal{T}_U(V)=\left(I-\frac{1}{2}W_U\right)^{-1}\left(I+\frac{1}{2}W_U\right)V, \qquad V \in T_X\mathrm{St}(p,n),
\end{equation}
which maps $T_X \mathrm{St}(p,n) \rightarrow T_{R_X(U)}\mathrm{St}(p,n)$.

\subsubsection{Algorithm Details}

% \textit{[TBD] An overview of the computation of the stochastic Euclidean NELBO gradient $\nabla^\mathcal{E} \widehat{\mathcal{L}}$, its Riemannian gradient $\nabla \widehat{\mathcal{L}}$, and same for the score function components for $\theta = (W_1, W_2, b_1, b_2$). The retraction in $\theta$ refers to using the Cayley transform retraction for the Stiefel factors, and the additive Euclidean retraction for the intercepts.}

The variational parameter $\theta = (W_1, W_2, b_1, b_2)$ belongs to the product manifold
\begin{equation}
    \mathcal{M} = \mathrm{St}(d, d) \times \mathrm{St}(d,d) \times \mathbb{R}^d \times \mathbb{R}^d.
\end{equation}
Closed-form expressions for the Euclidean score $\nabla^\mathcal{E}_\theta \log q_\theta (y)$ and the negative ELBO gradient $\nabla^\mathcal{E}_\theta \mathcal{L}(\theta)$ are derived in Appendix A.5 of \citet{godichon2024natural}; both are estimated by sampling from the base distribution $\mathcal{N}(0,I)$ of the normalising flow. The Riemannian gradients are obtained by projecting the Stiefel components onto the tangent space via \eqref{eq:stiefel-tangent-projection}. Concretely, for the score (and similarly for the ELBO gradient),
\small
\begin{align}
    \nabla_\theta \log q_\theta (y) & = \mathrm{Proj}_{\theta}(\nabla_\theta^\mathcal{E} \log q_\theta (y))
    \nonumber \\ & :=( \mathrm{Proj}_{W_1}[\nabla^\mathcal{E}_{W_1} \log q_\theta(y)], \mathrm{Proj}_{W_2}[\nabla^\mathcal{E}_{W_2} \log q_\theta(y)], \nabla^\mathcal{E}_{b_1} \log q_\theta(y), \nabla^\mathcal{E}_{b_2} \log q_\theta(y)).
\end{align}
\normalsize
In the following, the retraction $\mathcal{R}_\theta$ on $\mathcal{M}$ applies the Cayley retraction to the Stiefel components $W_1, W_2$, and the standard additive update for $b_1, b_2$.
\normalsize
\paragraph{Riemannian Stochastic Gradient Descent:} This is given by

\begin{equation}
    \theta^{(s+1)} = \mathcal{R}_{\theta^{(s)}}(-\tau_{s+1}\nabla_{\theta}\widehat{\mathcal{L}}(\theta^{(s)})),
\end{equation}
where $\nabla_\theta \widehat{\mathcal{L}}(\theta^{(s)}) = \mathrm{Proj}_{\theta^{(s)}}[\nabla_\theta^\mathcal{E} \widehat{\mathcal{L}}(\theta^{(s)})]$ is the stochastic Riemannian NELBO gradient.

\paragraph{Riemannian Natural Gradient:} This method follows the limited-memory variant of Algorithm \ref{alg:riemannian-if-ngd} described in Remark \ref{rem:limited-memory-approximation} to move and update $\mathbf{H}^{-1}_s$.
%has a block diagonal form
%\begin{equation}
%    \mathbf{H}^{-1}_s := \mathrm{diag}\left[(\mathbf{H}^{W_1}_s)^{-1}, (\mathbf{H}^{W_2}_s)^{-1}, (\mathbf{H}^{b_1}_s)^{-1}, (\mathbf{H}^{b_2}_s)^{-1}\right].
%\end{equation}
The update is
\begin{equation}
    \theta^{(s+1)} = \mathcal{R}_{\theta^{(s)}}(-\tau_{s+1}\mathbf{H}_{s+1}^{-1} \nabla_\theta \widehat{\mathcal{L}}(\theta^{(s)})).
\end{equation}
This matrix is updated using a sliding window of the $K=200$ most recent score vectors; see Appendix \ref{appendix:vectorized-fisher}. The intercept parameters $b_1, b_2$ are Euclidean, so the corresponding score updates are straightforward and require no transportation.
The Stiefel blocks are updated using the Riemannian score components $\nabla_W \log {q}_\theta(y)$, with past vectors transported to the current tangent space via \eqref{eq:stiefel-isometric-transport}. Since this transport is isometric, $\mu =\nu$ in the inverse representation of the $\mathbf{H}_s^{-1}$, which simplifies the sliding window update.

\subsection{Details of Fixed-Rank Manifold Experiment} \label{appx:Fixed-rank-manifold-experiment}

We review details on the fixed-rank manifold, and the implementation of our algorithms.

\subsubsection{Geometric Preliminaries}

The following properties of the fixed-rank manifold are discussed in \citet{vandereycken2013low}; see also \citet{meyer2011linear, mishra2014fixed}.

\paragraph{Background:} The manifold of rank-$r$ matrices is an embedded submanifold of $\mathbb{R}^{m\times n}$,
\begin{equation}
    \mathcal{M}_r := \{ X \in \mathbb{R}^{m\times n}, \mathrm{rank}(X)=r\}, \qquad 1 \leq r \leq \min(m,n),
\end{equation}
with $\mathrm{dim}(\mathcal{M}_r)=r(m+n-r)$. We represent points by their SVD $X = U \Sigma V^\top$, where $U \in \mathrm{St}(r,m)$, $V \in \mathrm{St}(r,n)$, and $\Sigma \in \mathbb{R}^{r\times r}$ is diagonal with non-increasing positive entries.

The tangent space at $X = U\Sigma V^\top \in \mathcal{M}_r$ forms a subspace of $\mathbb{R}^{m\times n}$ given by
\begin{equation}
    T_X \mathcal{M}_r = \{\xi \in \mathbb{R}^{m\times n} : (I_m - UU^\top) \xi (I_n -VV^\top)=0\}.
\end{equation}
For each $\xi \in T_X \mathcal{M}_r$, there exists a triple $(M,U_p,V_p)$ such that
\begin{equation} \label{eq:fixedrank-coordinates}
    \xi = UMV^\top + U_pV^\top + UV^\top_p,
\end{equation}
where $M \in \mathbb{R}^{r\times r}$, $U_p \in \mathbb{R}^{m\times r}$, $V_p \in  \mathbb{R}^{n\times r}$, and furthermore $U^\top U_p = 0$ and $V^\top V_p=0$.

For $\xi\equiv(M_\xi,U_\xi,V_\xi)$ and $\eta\equiv(M_\eta, U_\eta, V_\eta)$ in $T_X \mathcal{M}_r$, the Riemannian metric is given by
\begin{equation} \label{eq:fixedrank-metric}
    \langle \xi, \eta\rangle_X =\tr(\xi^\top\eta)=\tr(M_\xi^\top M_\eta) + \tr(U^\top_\xi U_\eta) + \tr(V_\xi^\top V_\eta).
\end{equation}

The orthogonal projection of $Z \in \mathbb{R}^{m\times n}$ onto the tangent space $T_X \mathcal{M}_r$ is
\begin{equation} \label{eq:fixedrank-proj}
    \mathrm{Proj}_X(Z) = UU^\top Z + (I_m -UU^\top)ZVV^\top.
\end{equation}
In particular, we have $\mathrm{Proj}_X(Z) = (M,U_p, V_p)$ where
\begin{equation} \label{eq:fixedrank-proj-decompose}
    M = U^\top Z V, \qquad U_p = (I_m-UU^\top)ZV, \qquad V_p = (I_n-VV^\top)Z^\top U.
\end{equation}

\paragraph{Retraction \& Transport:} We use the projection-based \eqref{eq:fixedrank-proj} vector transport
\begin{equation}
    \mathcal{T}_{X,\tilde{X}}(\xi)= \mathrm{Proj}_{\tilde{X}}(\xi) = (\tilde{M}, \tilde{U}_p, \tilde{V}_p).
\end{equation}
% Let $X = U\Sigma V^\top$ and $\tilde{X}=\tilde{U} \tilde{\Sigma}\tilde{V}^\top$, the components of $\tilde{\xi} = \mathcal{T}_{X,\tilde{X}}(\xi) := (\tilde{M}, \tilde{U}_p, \tilde{V}_p)$ are
% \begin{equation} \label{eq:fixedrank-proj-decompose}
%     \tilde{M} = \tilde{U}^\top \xi \tilde{V}, \qquad
%     \tilde{U}_p = (I_m-\tilde{U}\tilde{U}^\top) \xi \tilde{V}, \qquad
%     \tilde{V}_p =(I_n-\tilde{V} \tilde{V}^\top) \xi^\top \tilde{U}.
% \end{equation}
For $\xi = (M,U_p, V_p)$, the equations \eqref{eq:fixedrank-proj-decompose} yield a linear mapping $(M, U_p, V_p) \rightarrow (\tilde{M}, \tilde{U}_p, \tilde{V}_p)$ requiring $(O(n+m)r^2)$ operations and $O((n+m)r)$ intermediate storage.

For the retraction, we employ the metric projection \citep{absil2012projection} which returns the closest rank-$r$ matrix to $X+\xi$:
\begin{equation}
    \mathcal{R}_X(\xi)=\mathrm{trunc}_r(X+\xi):=\argmin_{\zeta \in \mathcal{M}_r} \Vert (X+\xi)-\zeta\Vert_F.
\end{equation}

This can be computed by truncating the SVD of $X+\xi$, which is achievable in $O((m+n)r^2)$ operations owing to the low-rank structure of $X$ and $\zeta$; see \citet{vandereycken2013low}.

\subsubsection{Implementation Details}

For a $K$-class classification problem with $d$ features, we take $m=d,n=K-1$. Thus, the parameter vector is $\theta =(B,\alpha) \in \mathcal{M}_r \times \mathbb{R}^{K-1}$.

\paragraph{Score Gradients:} Let $(x,y)$ denote an observation, and $p(x):=(p_1(x),...,p_{K-1}(x))$ the predicted class probabilities. The Euclidean score, and its Riemannian counterpart are
\begin{equation}
    \nabla_B^\mathcal{E} \ell (y \mid x) = x(e_y - p(x))^\top \in \mathbb{R}^{d\times (K-1)}, \qquad \nabla^\mathcal{M}_B\ell (y \mid x)  = \mathrm{Proj}_B (\nabla_B^\mathcal{E} \ell (y \mid x)),
\end{equation}
where $e_y \in \mathbb{R}^{K-1}$ is the $y$-th standard basis vector for $y \in \{1,...,K-1\}$ and $e_K := 0$. The intercept score is $\nabla_\alpha \ell(y \mid x) = e_y - p(x) \in \mathbb{R}^{K-1}$. The stochastic gradient of the NLL is
\begin{equation}
    \nabla_B^\mathcal{E} \hat{\mathcal{L}} = \frac{1}{\vert \mathcal{B} \vert} \sum_{i \in \mathcal{B}} x_i (p(x_i) - e_{y_i})^\top, \qquad \nabla_\alpha \hat{\mathcal{L}} = \frac{1}{\vert \mathcal{B}\vert} \sum_{i \in \mathcal{B}} (p(x_i) - e_{y_i}),
\end{equation}
where $\mathcal{B}$ is a minibatch. The Riemannian stochastic gradient $\nabla^\mathcal{M}_B \hat{\mathcal{L}}$ is obtained by projecting the Euclidean stochastic gradient onto the tangent space via \eqref{eq:fixedrank-proj}.

\paragraph{Fisher Score Update:} Let $\mathcal{B}_s=\{(x^s_i,y^s_i)\}^b_{i=1}$ denote the minibatch of observations used to obtain the objective gradient at the $s$-th iteration. For each observation in $\mathcal{B}_s$, we sample $\tilde{y}_i^s \sim \mathrm{Mult}(1,p(x_i^s))$, and compute the scores $\phi_i^B = \nabla_B \ell(\tilde{y}_i^s \mid x_i^s)$ and $\phi_i^\alpha = \nabla_\alpha \ell(\tilde{y}_i^s \mid x_i^s)$. We maintain a block-diagonal representation of the approximate Fisher operator, with separate components for the $B$ and $\alpha$ parameters
\begin{equation}
    H_s := \mathrm{BlockDiag}(H^B_s, H^\alpha_s).
\end{equation}

Due to the orthogonality of the three components $(M, U_p, V_p)$ in \eqref{eq:fixedrank-metric} with respect to the baseline metric, and the fact that the projection-based transport preserves this decomposition, the approximate Fisher operator $H_s^B$ also admits a block-diagonal structure
\begin{equation}
    H^B_s = \mathrm{BlockDiag}(H^M_s, H^{U_p}_s,H^{V_p}_s).
\end{equation}

For the inverse-free Riemannian natural gradient method (i.e., \Cref{alg:riemannian-if-ngd}), each $\phi_i^B$ is projected onto the tangent space $T_B \mathcal{M}_r$ of the current iterate via \eqref{eq:fixedrank-proj-decompose}, yielding $\xi_i = \mathrm{Proj}_{B}(\phi^B_i)=(M^i,U_p^i,V_p^i)$. The score vector update for $(H_s^B)^{-1}$ is performed separately for each sub-block using $\mathrm{vec}(M^i), \mathrm{vec}(U^i_p)$, and $\mathrm{vec}(V^i_p)$ respectively.
% The $\vert \mathcal{B}_s \vert$ score vectors can be incorporated into the current $H^{-1}$ simultaneously using the Woodbury matrix identity, of which the earlier rank-one update is a special case. Specifically, for each component in $(M,U_p, V_p)$, we stack the corresponding vectorized component of the score vectors, e.g. $\Phi^M = [\mathrm{vec}(\phi_1^M), ..., \mathrm{vec}(\phi_{\vert \mathcal{B}_s\vert}^M)] \in \mathbb{R}^{r^2 \times \vert \mathcal{B}_s \vert}$, and do
% \begin{equation}
%     (H_{\mathrm{new}}^M)^{-1} = (H^M)^{-1} - (H^M)^{-1}\Phi^M(I_{\vert\mathrm{B}_s\vert} + (\Phi^M)^\top (H^M)^{-1} (\Phi^M))^{-1} (\Phi^M)^\top (H^M)^{-1}
% \end{equation}

% For the ``extrinsic'' inverse-free method, the raw Euclidean scores $\phi^B$ are used directly. In both settings, we maintain a block-diagonal representation of the inverse Fisher operator, with separate components for the $B$ and $\alpha$ parameters. 
% The update $(H^\alpha)^{-1} \rightarrow (H^\alpha_\mathrm{new})^{-1}$ for the intercept component proceeds in an identical fashion.

\paragraph{Fisher Transport:} For the inverse-free Riemannian NGD method, when the matrix iterate moves $B^{\mathrm{old}}\rightarrow B^{\mathrm{new}}$, the current inverse Fisher block $(H^B)^{-1}$ is transported via
\begin{equation}
    (H^B_{\mathrm{new}})^{-1} = \mathcal{T}_{B^{\mathrm{old}}, B^{\mathrm{new}}} \circ (H^B)^{-1} \circ \mathcal{T}_{B^{\mathrm{new}}, B^{\mathrm{old}}} = \mathrm{Proj}_{B^{\mathrm{new}}} \circ (H^B)^{-1}.
\end{equation}
The projection term factorizes as a sum of Kronecker products, which preserve the block-diagonal structure of $H^B$. The intercept term $(H^\alpha)^{-1}$ is not transported.

\subsubsection{Algorithms}

We consider three separate algorithms.

\paragraph{Riemannian Stochastic Gradient Descent (RSGD):} This is given by
\begin{equation}
    B^{(s+1)} = \mathcal{R}_{B^{(s)}}\left( -\tau_{s+1} \nabla^\mathcal{M}_B \hat{\mathcal{L}} \right), \qquad \alpha^{(s+1)}=\alpha^{(s)}-\tau_{s+1} \nabla_\alpha \hat{\mathcal{L}}.
\end{equation}

\paragraph{Inverse-Free Riemannian Natural Gradient (IF-RNGD):} This method follows \Cref{alg:riemannian-if-ngd}. The inverse approximation $(H^B_s)^{-1}$ is a linear endomorphism on the tangent space $T_B \mathcal{M}_r$, represented using the coordinate system from \eqref{eq:fixedrank-coordinates}. It is updated with (projected) Riemannian score vectors, and transported between tangent spaces at each iteration.
\begin{equation}
    B^{(s+1)} = \mathcal{R}_{B^{(s)}}\left( -\tau_{s+1} (\mathbf{H}^B_{s+1})^{-1}\nabla^\mathcal{M}_B \hat{\mathcal{L}} \right), \qquad \alpha^{(s+1)}=\alpha^{(s)}-\tau_{s+1} (\mathbf{H}_{s+1}^\alpha)^{-1}\nabla_\alpha \hat{\mathcal{L}}.
\end{equation}

\paragraph{Extrinsic Inverse-Free Natural Gradient (Extrinsic IF-NGD):} The inverse Fisher approximation $(H^B_s)^{-1}$ is a linear endomorphism on the ambient space $\mathbb{R}^{d\times(K-1)}$, and is updated using raw (i.e., Euclidean) score vectors. The update direction is obtained by preconditioning the Euclidean objective gradient, and projecting onto the tangent space.
\begin{align}
    & B^{(s+1)} = \mathcal{R}_{B^{(s)}}\left( -\tau_{s+1} \mathrm{Proj}_{B^{(s)}}[(\mathbf{H}^B_{s+1})^{-1}\nabla^\mathcal{E}_B \hat{\mathcal{L}}] \right), \\
    & \alpha^{(s+1)}=\alpha^{(s)}-\tau_{s+1} (\mathbf{H}_{s+1}^\alpha)^{-1}\nabla_\alpha \hat{\mathcal{L}}.
\end{align}
This approach is similar to that described in \citet[Example 3]{godichon2024natural}.

\paragraph{Hyperparameters:} The step size schedule is $\tau_s = c_0 / (c_1 + s)^{\alpha}$ where $c_0=1, c_1=100$, and $\alpha=0.75$. For each method, the stochastic gradient of the objective is computed using a random minibatch of $128$ observations. The intercept parameter is initialized at zero, while $B_{\mathrm{init}}$ is diagonal with entries $(\bm{1}_r, \bm{0}_{\mathrm{rest}})$. The approximate inverse Fisher matrices $(H^B_0)^{-1},(H^\alpha_0)^{-1}$ are initialized at $I/\epsilon$, where the damping parameter is set to $\epsilon =1$. Variation of these hyperparameters largely leads to the same qualitative behaviour.

\end{document}